\documentclass[11pt]{article}

\usepackage[margin=1in]{geometry}
\usepackage{amsmath,amssymb,amsthm}
\usepackage{mathtools}
\usepackage{mathrsfs}
\usepackage{enumitem}
\usepackage{tikz}
\usepackage{tikz-cd}
\usepackage{float}
\usepackage{graphicx}
\usetikzlibrary{arrows.meta}
\IfFileExists{pythonhighlight.sty}{\usepackage{pythonhighlight}}{}
\usepackage{hyperref}

\newtheorem{theorem}{Theorem}[section]
\newtheorem{proposition}[theorem]{Proposition}
\newtheorem{definition}[theorem]{Definition}
\newtheorem{example}[theorem]{Example}
\newtheorem{remark}[theorem]{Remark}
\newtheorem{corollary}[theorem]{Corollary}
\newtheorem{lemma}[theorem]{Lemma}

\newcommand{\Hom}{\mathrm{Hom}}

\newcommand{\R}{\mathbb{R}}
\newcommand{\Z}{\mathbb{Z}}
\newcommand{\C}{\mathbb{C}}
\newcommand{\OO}{\mathrm{O}}
\newcommand{\SO}{\mathrm{SO}}
\newcommand{\SE}{\mathrm{SE}}
\newcommand{\E}{\mathrm{E}}
\newcommand{\Int}{\operatorname{Int}}
\newcommand{\Irr}{\operatorname{Irr}}
\newcommand{\Rep}{\operatorname{Rep}}

\newcommand{\id}{\mathrm{id}}
\newcommand{\sgn}{\mathrm{sgn}}
\newcommand{\tc}{\mathrm{tc}}
\newcommand{\Bis}{\operatorname{Bis}}
\newcommand{\Bisloc}{\operatorname{Bis}_{\mathrm{loc}}}

\title{Theory for groupoid equivariant neural networks:\\
an approach for steerable CNNs on bounded domains}
\author{A. Ibort$^{1,2,4}$\href{https://orcid.org/0000-0002-0580-5858}{\includegraphics[scale=0.7]{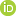}}, M. Jim\'enez-V\'azquez$^{1,5}$\href{https://orcid.org/0009-0009-5445-9320}{\includegraphics[scale=0.7]{ORCID.png}}, J.M. P\'erez-Pardo$^{1,6}$\href{https://orcid.org/0000-0003-2424-8943}{\includegraphics[scale=0.7]{ORCID.png}}}

\date{}

\begin{document}
\maketitle

\noindent
{\footnotesize $^{1}$  Universidad Carlos III de Madrid, ROR: \href{https://ror.org/03ths8210}{03ths8210}, Departamento de Matem\'aticas, Avenida de la Universidad 30 (edificio Sabatini), 28911 Legan\'es (Madrid), Espa\~na.}  \\
{{\footnotesize $^{2}$ Instituto de Ciencias Matem\'{a}ticas ICMAT (CSIC-UAM-UC3M-UCM), ROR: \href{https://ror.org/05e9bn444}{05e9bn444}, Campus de Cantoblanco UAM, Calle Nicol\'as Cabrera 13-15, 28049 Madrid, Espa\~na.} \\}

\bigskip
\noindent
{\scriptsize 
$^{4}$\texttt{albertoi[at]math.uc3m.es}, $^5$\texttt{mjvazque[at]pa.uc3m.es}, $^6$\texttt{jmppardo[at]math.uc3m.es}}

\begin{abstract}
Equivariant convolutional neural networks are usually built from a group
acting globally on the space of signals.  This hypothesis is inappropriate
for many bounded or stratified domains: an ambient rigid motion may be
admissible only on part of the domain, and the boundary introduces geometric
types that are invisible to a transitive group action.  We develop a theory of
groupoid-equivariant neural networks in which the symmetry datum consists of
a groupoid, a selected pseudogroup of local bisections, a measure, and input
and output representation bundles.  For integral channels on the object
space, we prove a bisection-equivariant kernel theorem: equivariance is
equivalent to a transport constraint on the two-point kernel, and its
solutions are classified by one joint-stabilizer intertwiner on each orbit of
pairs.  The usual steerable convolution constraint on homogeneous spaces is
recovered as a special case.

As a case study we apply the theory to bounded planar domains.  Requiring Euclidean arrows to
preserve tangent cones produces a groupoid whose orbits distinguish bulk,
smooth-edge, and corner points.  On a rectangle their isotropy groups are
$O(2)$, $\mathbb Z_2$, and $\mathbb Z_2$ respectively, and the kernel theorem yields a
complete boundary-aware layer with diagonal and cross-stratum couplings.  On
a pixel grid, the edge and corner isotropy subgroups are nonconjugate copies of
$\mathbb Z_2$ inside $D_4$, leading to different branching rules and closed
parameter-count formulas for ten geometric block families.  We also show that
partial equivariance is filtered under composition: finite-propagation layers
remain equivariant on explicitly eroded bisection domains, whereas global
symmetries are preserved exactly at arbitrary depth.

The resulting architecture is implemented through offline nullspace bases and
sparse gather--transform--scatter operations.  Numerical certificates verify
the kernel constraints and the composition theorem to machine precision, and
a synthetic identification experiment confirms completeness and reduced
sample complexity.  A Poisson--Dirichlet kernel study is used separately to assess
boundary-aware inductive bias; the exact inverse is shown to preserve the
global symmetries of the rectangle but not general proper local bisections.  The numerical results show that the proposed architectures provide significant advantages when symmetries cannot be globally implemented by group actions and provide an accuracy improvement of at least one order of magnitude  with respect to the models tested.
\end{abstract}

\begingroup
\small
\tableofcontents
\endgroup
\clearpage


\section{Introduction}
\label{sec:introduction}

Symmetry is one of the principal mechanisms by which prior geometric
knowledge is incorporated into machine learning (see, for instance, \cite{bronstein} and references therein).  A neural map is
\emph{equivariant} when transforming its input and then applying the map is
equivalent to first applying the map and then transforming its output.  In
convolutional architectures this requirement restricts the admissible kernels,
produces systematic weight sharing, and often improves sample efficiency and
interpretability.  Group-equivariant CNNs and steerable CNNs formalize this
principle for transformations generated by a group acting on the signal space
\cite{cohen-welling-groups,cohen-welling-steerable,cohen-welling-spherical,WeilerCesa2019}.  Their
representation-theoretic foundations are particularly transparent on
homogeneous spaces, where equivariant linear maps are characterized by
convolutions with constrained kernels
\cite{KondorTrivedi2018,CohenGeigerWeiler2019}.

The global-action hypothesis is nevertheless too restrictive in many
geometric problems.  Consider a signal on a bounded domain
$D\subset\mathbb R^2$.  The ambient Euclidean group
$E(2)=\mathbb R^2\rtimes O(2)$ acts transitively on the plane, but a generic
rigid motion does not map $D$ to itself.  It may still map one region of $D$
to another, and hence remains meaningful as a \emph{partial} symmetry.  At the
same time, the boundary distinguishes interior points, smooth-edge points, and
corners.  Treating every pixel as though it belonged to a homogeneous copy of
the plane suppresses this information and leaves boundary handling to padding,
masking, or other auxiliary conventions.  Padding is not neutral: standard
convolutions can exploit the absolute-position information created by the
boundary of the padded array \cite{KayhanGemert2020,IslamJiaBruce2020}.

Groupoids provide a natural language for the invertible partial symmetries
that survive this restriction.  An arrow relates a source point to a target
point only when the corresponding transformation is admissible; a local
bisection assembles such arrows into a coherent partial transformation of a
region.  The groupoid alone, however, does not determine a notion of
network equivariance.  One must also specify which local bisections are to be
regarded as admissible symmetries, how the object space is measured, and how
the feature fibres transform.  The basic symmetry datum (see Sect. \ref{ssec:symmetry-datum}, Def. \ref{def:symmetry-datum}) of this paper is therefore
\begin{equation}
 \mathfrak S_{\mathrm{in,out}}
 =
 \Bigl(
   \Gamma\rightrightarrows\Omega,
   \mathcal B,
   \nu;
   (E^{\mathrm{in}},R^{\mathrm{in}}),
   (E^{\mathrm{out}},R^{\mathrm{out}})
 \Bigr),
 \label{eq:intro-symmetry-datum}
\end{equation}
where $\mathcal B$ is a selected pseudogroup of local bisections and the two
bundles $E^{\mathrm{in}}$, $E^{\mathrm{out}}$, carry representations $R^{\mathrm{in}}$, $R^{\mathrm{out}}$ of the groupoid that implements the symmetry, $\Gamma\rightrightarrows\Omega$.  This formulation contains global
group equivariance as the special case in which the relevant bisections are
global.

The central analytical question is then the following.  Given the datum
\eqref{eq:intro-symmetry-datum}, which nonlocal linear maps between fields of
features are equivariant?  A single groupoid arrow acts only between two
fibres and cannot by itself constrain an operator that mixes values at many
points.  Local bisections are the correct objects because they move regions
coherently.  We prove that an integral channel with two-point kernel $K(y,x)$
is equivariant under $\mathcal B$ precisely when
\begin{equation}
 K\bigl(\tau_b(y),\tau_b(x)\bigr)
 =
 R^{\mathrm{out}}\bigl(b(y)\bigr)
 K(y,x)
 R^{\mathrm{in}}\bigl(b(x)\bigr)^{-1}
 \label{eq:intro-kernel-constraint}
\end{equation}
for every admissible local bisection $b$ and almost every pair in its domain (see Sect. \ref{ssec:bisection-kernel-theorem}, Thm. \ref{thm:bisection-kernel}).
The free kernel data reduce, orbit by orbit in $\Omega\times\Omega$, to
intertwiners of the corresponding joint stabilizers.  In the homogeneous
Euclidean case, translations reduce $K(y,x)$ to a function of $y-x$, and
\eqref{eq:intro-kernel-constraint} becomes the familiar steerable-kernel
constraint.  Our object-space kernel theorem should be distinguished from
Haar-system convolution on the arrow space of a Lie groupoid: the two are
related in important examples, but they are not the same construction.

The geometric model developed here is the tangent-cone restricted Euclidean
groupoid of a bounded planar domain.  The ordinary restriction of the action
groupoid $(E(2)\ltimes\mathbb R^2)|_D$ remains transitive and therefore does
not distinguish boundary types at the level of object orbits and isotropy.
We refine it by retaining only arrows whose orthogonal part transports the
tangent cone at the source to the tangent cone at the target.  For a
piecewise-smooth domain this separates the interior, smooth boundary, and
corners of each opening angle.  On a rectangle the three strata have isotropy
$O(2)$, $\mathbb Z_2$, and $\mathbb Z_2$, respectively (see Sect. \ref{ssec:tc-orbits-isotropy}, Thm. \ref{thm:tc-orbit-classification}).  Applying
\eqref{eq:intro-kernel-constraint} to pair orbits produces not only the
stratum-preserving kernels, but also the bulk--edge, bulk--corner, and
edge--corner couplings needed to exchange information across the
stratification (see Sect. \ref{sec:boundary-kernel-classification}).  These cross-stratum blocks are the symmetry-constrained
alternative to treating the boundary by an external padding rule.

The discrete theory contains an additional representation-theoretic feature.
On a rectangular pixel grid the ambient point group is $D_4$.  Edge isotropy
is generated by an axis reflection, whereas corner isotropy is generated by a
diagonal reflection.  These subgroups are isomorphic but nonconjugate in
$D_4$, and the restrictions of the $B_1$ and $B_2$ irreducible types are
therefore exchanged between edges and corners.  The finite kernel constraints
form homogeneous linear systems.  Their nullspaces provide complete bases of
the admissible layers, with closed parameter-count formulas for ten geometric
block families: the nine source--target stratum blocks, with the edge--edge
block separated into same-edge and across-corner interactions (see Sect. \ref{ssec:discrete-parameter-counts}, Thm. \ref{thm:discrete-parameter-counts}).

A second structural issue arises when such layers are stacked.  Global
equivariance is stable under composition, but a proper partial symmetry moves
only a restricted region.  Intermediate points in a multi-layer computation
must also remain inside the domain of the bisection.  We formalize this through
finite propagation and erosion.  A layer is exactly equivariant on every
maximal admissible bisection domain, while a composition is equivariant on an
explicitly eroded domain determined by the propagation radii of its factors (see Sect. \ref{sec:propagation-erosion}).
For $n$ equal-radius layers, the balanced path estimate gives the erosion
radius $\lfloor n/2\rfloor r_0$; global bisections require no erosion.  This
produces a filtration of equivariant operator classes and makes precise the
way in which depth converts a partial symmetry into a controlled boundary
layer.

\medskip
\noindent\textit{Relation with existing approaches.}
The present construction is complementary to several established directions.
Gauge-equivariant CNNs formulate covariance under changes of local frame on a
manifold \cite{CohenGauge2019}; our adapted edge and corner gauges use the same
representation-theoretic principle, but the bisections encode partial rigid
motions of a fixed bounded domain rather than only changes of frame.  Work on
partial or learned equivariance relaxes or estimates the symmetry constraint
from data \cite{RomeroLohit2022}; here the surviving partial symmetry is
specified geometrically and imposed exactly at the layer level.  Recent
categorical and Lie-groupoid approaches formulate equivariant architectures
using naturality or convolution on groupoid arrows
\cite{Maruyama2025,Astwood2026}.  Our focus is more specific: we derive a
boundary stratification from tangent cones, classify the complete
object-space kernel on all source--target stratum pairs, and analyze the
filtered behavior of these kernels under composition.  

Finally, neural
operators learn maps between function spaces and have become a standard tool
for PDE surrogates \cite{FNO2021,Kovachki2023}; the present framework supplies
a representation-theoretic way of encoding bounded-domain geometry, while our
experiments also emphasize that a useful architectural bias need not coincide
with an exact symmetry of the target operator (see Sect. \ref{sec:boundary-learning}).

\medskip
The main results can be summarized as follows.
\begin{enumerate}[label=\textup{(\roman*)}]
\item We define the measured groupoid symmetry datum
      \eqref{eq:intro-symmetry-datum}, separating the groupoid of arrows from
      the selected pseudogroup of coherent partial transformations and from
      the feature representations.
\item We classify pointwise equivariant channels by isotropy intertwiners and
      prove the bisection-equivariant kernel theorem
      (Theorem~\ref{thm:bisection-kernel}).  Its pair-orbit reduction
      (Proposition~\ref{prop:pair-orbits}) identifies one
      joint-stabilizer intertwiner per orbit of pairs and recovers ordinary
      steerable convolution as Corollary~\ref{cor:euclidean-kernel}.
\item We introduce the tangent-cone restricted Euclidean groupoid and prove
      its bulk--edge--corner orbit and isotropy classification
      (Theorem~\ref{thm:tc-orbit-classification}).  For a rectangle we derive
      a complete normal form for all diagonal and cross-stratum kernels
      (Theorem~\ref{thm:rectangle-kernel-classification}).
\item We develop the finite $D_4$ theory, including the distinct edge and
      corner branching rules, finite completeness, and closed parameter
      counts for the ten geometric block families
      (Theorem~\ref{thm:discrete-parameter-counts}).
\item We prove sharp two-layer and multi-layer composition results in terms of
      eroded bisection domains
      (Theorems~\ref{thm:two-layer-erosion} and
      \ref{thm:n-layer-filtered-equivariance}), and extend the result to
      pointwise nonlinearities, affine maps, parallel branches, and residual
      networks.
\item We turn the classification into a practical architecture with
      stratified feature types, equivariant nonlinearities and normalization,
      offline nullspace bases, and sparse gather--transform--scatter
      implementations.  The resulting parameterization is complete for the
      chosen finite symmetry datum.
\item We verify the algebraic constraints and filtered composition theorem to
      machine precision.  A synthetic operator-identification experiment
      confirms completeness and reduced identification complexity.  A
      Poisson--Dirichlet kernel example is reported as a test of
      boundary-aware inductive bias after showing that the fixed-domain inverse
      is globally, but not generally locally, equivariant.
\end{enumerate}

\medskip
\noindent\textit{Organization.}
Part~\ref{part:symmetry-kernels} develops the abstract symmetry datum,
transport operators, pointwise channels, and the bisection-equivariant kernel
theorem.  Part~\ref{part:euclidean-groupoids} studies the Euclidean action
groupoid, its tangent-cone refinement, the boundary-aware pair-orbit
classification, and the finite $D_4$ parameter counts.
Part~\ref{part:filtered-equivariance} introduces finite propagation and proves
the erosion theorems for compositions and deep networks.
Part~\ref{part:architecture} specifies the corresponding neural architecture,
including nonlinearities, normalization, residual connections, gauges, and
the sparse implementation.  Part~\ref{part:experiments} reports exact
certificates, synthetic recovery, and the boundary-value experiment.  The final
section summarizes the results and discusses the principal mathematical and
computational extensions.

\part{Groupoid symmetry data and equivariant kernels}
\label{part:symmetry-kernels}

\section{Groupoids, bisections, and symmetry data}
\label{sec:groupoid-symmetry-data}
\label{sec:groupoids-basics} 

The purpose of this section is to isolate the geometric and
representation-theoretic data that determine the notion of equivariance used
throughout the article.  A groupoid records which points of the object space can
be related by admissible arrows.  A selected family of local bisections records
which arrows can be assembled into coherent partial transformations of
regions.  A measure determines the spaces of fields on which the network
acts, and representations of the groupoid determine how the internal feature
fibres transform.  Equivariance is therefore attached not to an abstract
groupoid alone, but to the complete collection of these data.

This formulation contains ordinary group equivariance as a special case.  It
also covers situations in which an ambient group acts on a larger space but
not on the domain supporting the signals.  In that case the surviving rigid
motions are naturally partial transformations; the corresponding arrows and
local bisections remain available even though no global group action on the
domain exists.

\subsection{Groups, global actions, and partial transformations}
\label{sec:groups-symmetries}

We begin with the classical global picture.

\begin{definition}[Group action]
Let $G$ be a group and let $X$ be a set.  A left action of $G$ on $X$ is a map
\[
  G\times X\longrightarrow X,
  \qquad (g,x)\longmapsto g\cdot x,
\]
such that
\[
  e\cdot x=x,
  \qquad
  (gh)\cdot x=g\cdot(h\cdot x)
\]
for all $g,h\in G$ and $x\in X$.  When $G$ and $X$ carry topological or smooth
structures, the action is required to have the corresponding regularity.
\end{definition}

A group action provides one transformation of the whole space for every
element of $G$.  This globality is precisely what fails after restricting the
space of signals.  Let $D\subseteq X$.  Even if $G$ acts on $X$, a generic
$g\in G$ need not satisfy $gD \subseteq D$.  It nevertheless defines a partial
transformation
\begin{equation}
  g\colon D\cap g^{-1}D\longrightarrow D\cap gD,
  \qquad x\longmapsto g\cdot x.
  \label{eq:partial-action-on-domain}
\end{equation}
The domains in \eqref{eq:partial-action-on-domain} depend on $g$; consequently,
the surviving transformations do not in general form a group acting on $D$.
They are closed instead under restriction, partially defined composition, and
inverse.  Groupoids record these pointwise admissible transformations, while
local bisections assemble them into coherent motions of subsets of $D$.

\begin{remark}[Scope of the groupoid language]
Groupoids are not the only formalism for partial symmetry: partial group
actions, inverse semigroups, pseudogroups, and more general categories are
closely related alternatives.  The advantage relevant here is that groupoids
combine invertible partial transformations with orbit, isotropy, and
representation theory.  Our claim is therefore not that every conceivable
notion of symmetry must be a groupoid, but that groupoids are a natural and
effective language for the invertible local symmetries considered in this
article.
\end{remark}

\subsection{Groupoids, orbits, isotropy, and restrictions}
\label{sec:groupoids}

\begin{definition}[Groupoid]
A groupoid consists of two sets $\Gamma$ and $\Omega$, source and target maps
\[
  s,t\colon\Gamma\longrightarrow\Omega,
\]
and a partially defined multiplication (or composition) $\circ$.  It is customarily displayed as $ \Gamma\rightrightarrows\Omega$.
Elements of $\Omega$ are called objects.  An element $\alpha\in\Gamma$ with
$s(\alpha)=a$ and $t(\alpha)=b$ is written $\alpha\colon a\to b$ and is called
an arrow. Two arrows $\beta$, $\alpha$ are composable if $s(\beta) = t(\alpha)$, in which case their multiplication $\beta \circ \alpha$ is an arrow $\beta \circ \alpha \colon  s(\alpha) \to t(\beta)$.  Multiplication is associative: $\gamma \circ (\beta \circ \alpha) = (\gamma \circ \beta) \circ \alpha$, provided $s(\gamma) = t(\beta)$ and $s(\beta) = t(\alpha)$;
each object $a\in \Omega$ has a unit arrow $1_a\colon a\to a$: 
$\alpha \circ 1_a = 1_b \circ \alpha = \alpha$ for every $\alpha \colon a \to b$; and every arrow
$\alpha\colon a\to b$ has an inverse $\alpha^{-1}\colon b\to a$: $\alpha^{-1}\circ \alpha = 1_a$, $\alpha\circ \alpha^{-1} = 1_b$.
\end{definition}

For $a,b\in\Omega$ we use the notation
\[
  \Gamma_a=s^{-1}(a),
  \qquad
  \Gamma^b=t^{-1}(b),
  \qquad
  \Gamma_a^b=\Gamma_a\cap\Gamma^b.
\]
that represent, respectively, the subset of arrows with source $a$, the subset of arrows with target $b$ and the subset of arrows with source $a$ and target $b$.
The isotropy group at $a$ is defined as:
\begin{equation}
  \Gamma(a):=\Gamma_a^a,
  \label{eq:isotropy-group}
\end{equation}
and the orbit through $a$ is
\begin{equation}
  \mathcal O(a)
  :=\{b\in\Omega:\Gamma_a^b\neq\varnothing\}.
  \label{eq:groupoid-orbit}
\end{equation}
The orbits partition $\Omega$; their set is denoted by $\pi_0(\Gamma)$.  The
groupoid is transitive when it has a single orbit.

\begin{definition}[Restriction]
Let $D\subseteq\Omega$.  The restriction, or full subgroupoid, of $\Gamma$ to
$D$ is
\begin{equation}
  \Gamma|_D
  :=\{\alpha\in\Gamma:s(\alpha),t(\alpha)\in D\}
  \rightrightarrows D.
  \label{eq:groupoid-restriction}
\end{equation}
\end{definition}

Restriction is the elementary operation that will allow us to retain the
arrows of an ambient symmetry that remain meaningful on a bounded domain.
Notice that restricting a transitive groupoid need not preserve the original
global transformations, although the restricted groupoid may still have many
arrows and may itself remain transitive.

\begin{example}[Action groupoid]
\label{ex:action-groupoid}
Let $G$ act on $X$.  The associated action groupoid is
\[
  G\ltimes X\rightrightarrows X.
\]
Its arrows are pairs $(g,x)\in G\times X$, with
\[
  s(g,x)=x,
  \qquad
  t(g,x)=g\cdot x.
\]
Composition and inversion are
\[
  (h,g\cdot x)\circ(g,x)=(hg,x),
  \qquad
  (g,x)^{-1}=(g^{-1},g\cdot x).
\]
Thus every global group action determines a groupoid.  The isotropy group of
the action groupoid at $x$ is the stabilizer
\[
  G_x=\{g\in G:g\cdot x=x\}.
\]
\end{example}

\begin{example}[A restricted translation symmetry]
\label{ex:restricted-translation}
Let $\Z$ act on $\R$ by translations and let $D=[a,b]$.  The group $\Z$ does
not act on $D$, unless the interval is replaced by a translation-invariant
set.  The restricted action groupoid is nevertheless well defined:
\begin{equation}
  (\Z\ltimes\R)|_D
  =\{(n,x)\in\Z\times D:x+n\in D\}
  \rightrightarrows D.
  \label{eq:restricted-translation-groupoid}
\end{equation}
Its arrows are exactly the integer translations whose source and target lie
in $D$.  The same construction applies to a lattice action on $\R^d$ and a
bounded image domain.  This elementary example contains the basic mechanism
used later for Euclidean motions on domains with boundary.
\end{example}

\begin{remark}[Regularity]
In the applications below, $\Omega$ will be a manifold, a manifold with
corners, a stratified space, or a finite pixel grid.  Correspondingly,
$\Gamma$ will be a Lie groupoid, a groupoid compatible with a stratification,
or a finite groupoid.  For a Lie groupoid, the arrow and object spaces are
smooth manifolds, the source and target maps are submersions, and all
structure maps are smooth; see, for example, \cite{Ma05}.  Whenever the object
space has strata of different dimensions, all regularity statements are to be
understood stratum by stratum unless explicitly stated otherwise.
\end{remark}

\subsection{Local bisections and admissible pseudogroups}
\label{ssec:bisections-basics}

A single arrow relates two fibres of the source or target maps.  For instance, if $\alpha \colon a \to b \in \Gamma$, then left composition with $\alpha$ defines a map $L_\alpha \colon \Gamma^a \to \Gamma^b$, $L_\alpha(\beta) = \alpha \circ \beta$.  Equivariance of a nonlocal operator,
however, requires a coherent transformation of a whole region.  This is the
role that local bisections are going to play.

\begin{definition}[Local bisection]
\label{def:bisection}
Let $\Gamma\rightrightarrows\Omega$ be a topological groupoid.  A local
bisection is a continuous map
\[
  b\colon U_b\longrightarrow\Gamma
\]
defined on an open set $U_b\subseteq\Omega$ such that
\[
  s\circ b=\id_{U_b}
\]
and
\[
  \tau_b:=t\circ b\colon U_b\longrightarrow U_b':=\tau_b(U_b)
\]
is a homeomorphism onto an open subset $U_b'\subseteq\Omega$.  We will also call $U_b$ the domain of the local bisection $b \colon U_b \to \Gamma$, and we will denote it as $U_b = \mathrm{dom\,} (b)$.  In the smooth
setting $b$ and $\tau_b$ are required to be smooth, with $\tau_b$ a
diffeomorphism onto its image.  In the purely algebraic or finite setting,
``open'' and ``homeomorphism'' are replaced by ``subset'' and ``bijection''.
A bisection is global if $U_b=U_b'=\Omega$.
\end{definition}

\begin{remark}
Alternatively, a local bisection can be defined as a subset $\Sigma \subset \Gamma$, such that the restriction of the source and target maps to $\Sigma$ are homeomorphisms.   Denoting by $b = (s|_\Sigma)^{-1}$ we recover Def. \ref{def:bisection}.
\end{remark}

If $b\colon U\to\Gamma$ and $b'\colon U'\to\Gamma$ are local bisections, their
composition is defined wherever $\tau_b(x)\in U'$ for $x \in U$, by
\begin{equation}
  (b'\circ b)(x)
  =b'\bigl(\tau_b(x)\bigr)\circ b(x) \, ,
  \label{eq:bisection-composition}
\end{equation}
and then, $b'\circ b$ is a local bisection with domain $\tau_b^{-1}(U') = \tau_b^{-1}(\mathrm{dom\,}(b')) \subseteq U = \mathrm{dom\,}(b)$. 
The inverse bisection $b^{-1}$ is  defined as:
\begin{equation}
  b^{-1}(y)
  =b\bigl(\tau_b^{-1}(y)\bigr)^{-1} \, .
  \label{eq:bisection-inverse}
\end{equation}
with domain $\mathrm{dom\,}(b^{-1}) = \tau_b (U) = \tau_b (\mathrm{dom}(b))$.   Note that both $b^{-1} \circ b$ and $b\circ b^{-1}$ are well defined local bisections with domains $U = \mathrm{dom\,}(b)$ and $U' = \mathrm{dom\, }(b^{-1})$ respectively, and $b^{-1} \circ b =: 1_U \colon U \to \Gamma$, $1_U(x) = 1_x$ for all $x\in U$; $b \circ b^{-1} =: 1_{U'} \colon U' \to \Gamma$, $1_{U'}(y) = 1_y$, for all $y\in U'$. The local bisections $1_U$ are called identity local bisections.

Local bisections are stable under restriction of their domains.  
Local bisections form an inverse semigroup denoted as $\Bisloc(\Gamma)$; the composition of local bisections is associative and $b^{-1}\circ b \circ b^{-1} = b^{-1}$, $b\circ b^{-1} \circ b = b$ \cite{Pa99}.    Their associated local
transformations $\tau_b$ form a pseudogroup\footnote{Pseudogroup here is understood in its simpler form: A pseudogroup $\mathcal{G}$ is a collection of partial homeomorphisms between open subsets of a topological space such that $\mathcal{G}$  is closed under composition and inverses, where we compose $\phi, \psi \in \mathcal{G}$ only if $\mathrm{im\,} \psi = \mathrm{dom\,} \phi$.} after closure under compatible
unions.  On the other hand global bisections form a group $\Bis(\Gamma)$.  In what follows we will not make use of the theory of inverse semigroups and/or pseudogroups as we will concentrate on concrete families of local bisections determined by the structure of the supporting groupoid.  

Constant and rigid bisections:
For the action groupoid $G\ltimes X$, every $g\in G$ defines a constant global
bisection
\begin{equation}\label{ex:constant-bisections}
  b_g(x)=(g,x),
  \qquad
  \tau_{b_g}(x)=g\cdot x.
\end{equation}
Thus the original group $G$ embeds into $\Bis(G\ltimes X)$ in a canonical way.  If the action
groupoid is restricted to $D\subseteq X$, the same formula, Eq. (\ref{ex:constant-bisections}), defines a local
bisection on every open set
\[
  U\subseteq D\cap g^{-1}D.
\]
We denote it by
\begin{equation}
  b_{g,U}(x)=(g,x)
  \label{eq:rigid-bisection}
\end{equation}
and call it a rigid local bisection.  Its associated partial transformation is
$g|_U\colon U\to gU$.  These bisections formalize the ambient rigid motions
that survive only locally on the restricted domain.

\medskip
The complete family $\Bisloc(\Gamma)$ of local bisections of a groupoid is often larger than the geometric
symmetries one wishes to impose on a model.  We therefore select an admissible
subfamily.

\begin{definition}[Admissible and ample families of bisections]
\label{def:admissible-bisections}
An admissible family of local bisections is a subset
$\mathcal B\subseteq\Bisloc(\Gamma)$ that contains the identity local bisections and
is closed under restriction, inversion, and every defined composition.  It is
called ample if every arrow is attained by a member of the family: for every
$\alpha\in\Gamma$ there exists $b\in\mathcal B$ such that
\[
  s(\alpha)\in U_b,
  \qquad
  b\bigl(s(\alpha)\bigr)=\alpha.
\]
\end{definition}

The choice of $\mathcal B$ is part of the model.  Two different admissible
families on the same groupoid can impose different equivariance constraints.
For example, the full family of local bisections permits all coherent local
motions encoded by the groupoid, whereas the rigid family of
Example~\ref{ex:constant-bisections} retains only those assembled from a single
ambient transformation on each local domain.  The latter is the family used
for the Euclidean constructions in Part~\ref{part:euclidean-groupoids}.

\subsection{Linear representations and reduction to isotropy}
\label{sec:reps}

The arrows of a groupoid ``move'' points of the object space.  A representation
specifies how they transport the internal feature spaces attached to those
points.

\begin{definition}[Linear representation of a groupoid]
\label{def:groupoid-representation}
Let $p\colon E\to\Omega$ be a field of finite-dimensional complex vector
spaces, with fibre $E_a=p^{-1}(a)$.  A linear representation of
$\Gamma\rightrightarrows\Omega$ on $E$ is an assignment
\[
  \alpha\colon a\to b
  \quad\longmapsto\quad
  R(\alpha)\colon E_a\longrightarrow E_b
\]
of linear isomorphisms satisfying
\begin{equation}
  R(\beta\circ\alpha)=R(\beta)R(\alpha),
  \qquad
  R(1_a)=I_{E_a}.
  \label{eq:groupoid-representation-law}
\end{equation}
If the fibres are Hermitian spaces and every $R(\alpha)$ is unitary, the
representation is called unitary.  Measurability, continuity, or smoothness is
required according to the category in which $E$ and $\Gamma$ are considered.
\end{definition}

A representation may equivalently be viewed as a functor from the groupoid
$\Gamma$ to the category of vector spaces and linear isomorphisms.  An
intertwiner between representations $(E,R)$ and $(E',R')$ is a family of
linear maps $T_a\colon E_a\to E'_a$ such that
\begin{equation}
  T_bR(\alpha)=R'(\alpha)T_a
  \label{eq:representation-intertwiner}
\end{equation}
for every arrow $\alpha\colon a\to b$.  Alternatively, in categorical language, an intertwiner is a natural transformation between the functors $R$ and $R'$. The representations are equivalent if
the maps $T_a$ can be chosen invertible.

The restriction of $R$ to the isotropy group $\Gamma(a)$ is a  linear representation
\[
  \mu_a\colon\Gamma(a)\longrightarrow\mathrm{GL}(E_a),
\]
called the isotropy representation at $a$.  If $a$ and $b$ belong to the same
orbit and $\gamma\colon a\to b$, conjugation by $\gamma$ identifies
$\Gamma(a)$ with $\Gamma(b)$, while $R(\gamma)$ identifies the corresponding
isotropy representations, thus isotropy representations at objects $a,b$ in the same orbit are equivalent.

\begin{proposition}[Reduction of a transitive groupoid to isotropy]
\label{thm:imprimitivity}
Let $\Gamma\rightrightarrows\Omega$ be a transitive algebraic groupoid and fix
$a_0\in\Omega$.  Restriction to the isotropy group defines an equivalence of
categories
\begin{equation}
  \Rep(\Gamma)\simeq\Rep\bigl(\Gamma(a_0)\bigr).
  \label{eq:rep-equivalence-isotropy}
\end{equation}
More explicitly, if
$\mu\colon\Gamma(a_0)\to\mathrm{GL}(V)$ is a linear representation, the corresponding
linear representation $R_\mu$ of $\Gamma$ is carried by the associated bundle
\begin{equation}
  E_\mu
  :=\Gamma_{a_0}\times_{\Gamma(a_0)}V
  \longrightarrow\Omega,
  \qquad
  [\gamma,v]\longmapsto t(\gamma),
  \label{eq:associated-groupoid-bundle}
\end{equation}
where $\Gamma_{a_0}=s^{-1}(a_0)$ and
\[
  (\gamma\circ h,v)\sim(\gamma,\mu(h)v),
  \qquad h\in\Gamma(a_0).
\]
An arrow $\alpha\colon a\to b$ acts by
\begin{equation}
  R_\mu(\alpha)[\gamma,v]=[\alpha\circ\gamma,v],
  \qquad t(\gamma)=a.
  \label{eq:induced-groupoid-action}
\end{equation}
For a nontransitive groupoid, the same construction is applied independently
on every orbit.
\end{proposition}

\begin{proof}
Restriction gives a functor from representations of $\Gamma$ to
representations of $\Gamma(a_0)$.  Conversely,
\eqref{eq:associated-groupoid-bundle}--\eqref{eq:induced-groupoid-action}
induce a representation from any isotropy representation.  The two
constructions are inverse up to natural equivalence.  In the topological and
smooth settings the same statement holds in the corresponding categories,
provided the groupoid, the associated bundle, and the representation satisfy
the required regularity; see \cite{Ma05} and
\cite{ibort-marmo-mas-dorca-schiavone}.
\end{proof}

\begin{remark}[Complete reducibility]
The equivalence \eqref{eq:rep-equivalence-isotropy} does not by itself imply an
isotypic decomposition.  Whenever Schur orthogonality and multiplicity
formulas are used later, the fibres will be finite-dimensional and the
relevant isotropy groups will be compact (or finite), so that unitary
representations are completely reducible.  We work over $\C$ in the abstract
theory.  Real implementations will be obtained by choosing real forms of the
representations; this avoids the additional real, complex, and quaternionic
commutant cases in the abstract form of Schur's lemma.
\end{remark}

\subsection{Measures and the symmetry datum of a channel}
\label{ssec:symmetry-datum}

To pass from fibrewise representations to spaces of signals, the object space
must be equipped with a measure.  Let $\nu$ be a $\sigma$-finite Borel measure
on $\Omega$.  A local bisection $b\colon U_b\to\Gamma$ is
$\nu$-preserving if
\begin{equation}
  (\tau_b)_*(\nu|_{U_b})=\nu|_{U_b'}.
  \label{eq:measure-preserving-bisection}
\end{equation}
In the continuous Euclidean examples, $\nu$ will be Lebesgue measure in the
bulk, arclength on smooth boundary strata, and counting measure on corner
strata.  In the discrete theory it will be counting measure on every stratum.

\begin{definition}[Measured bisection symmetry datum]
\label{def:symmetry-datum}
A measured bisection symmetry datum for a channel is a tuple
\begin{equation}
  \mathfrak S_{\mathrm{in,out}}
  =\Bigl(
     \Gamma\rightrightarrows\Omega,
     \mathcal B,
     \nu;
     (E^{\mathrm{in}},R^{\mathrm{in}}),
     (E^{\mathrm{out}},R^{\mathrm{out}})
    \Bigr),
  \label{eq:symmetry-datum}
\end{equation}
where
\begin{enumerate}[label=\textup{(\roman*)}]
\item $\Gamma\rightrightarrows\Omega$ is a groupoid in the chosen regularity
category;
\item $\mathcal B\subseteq\Bisloc(\Gamma)$ is an admissible family of local
bisections;
\item $\nu$ is a $\sigma$-finite measure preserved by the bisections in
$\mathcal B$;
\item $(E^{\mathrm{in}},R^{\mathrm{in}})$ and
$(E^{\mathrm{out}},R^{\mathrm{out}})$ are measurable finite-dimensional
unitary representations of $\Gamma$.
\end{enumerate}
\end{definition}

Each entry of \eqref{eq:symmetry-datum} has a distinct role.  The groupoid
determines the background ``symmetry'' of the problem, that is, the admissible pointwise arrows and their isotropy groups.  The
family $\mathcal B$ determines which arrows are required to assemble into
coherent partial transformations and hence which equivariance identities are
imposed.  The measure $\nu$ determines the Hilbert spaces of sections and the
change-of-variables rule.  Finally, the two representations specify how input
and output feature fibres transform.  Omitting any one of these ingredients
leaves the equivariance problem underdetermined.

\begin{remark}[Quasi-invariant measures]
The measure-preserving assumption is sufficient for all applications in this
article and keeps the transport formulas transparent.  The construction extends
to quasi-invariant measures by inserting the appropriate Radon--Nikodym
factor, or equivalently by working with half-densities.  We do not pursue that
extension here in order to keep the formulas as simple as possible.
\end{remark}

\subsection{Basic examples and standing assumptions}
\label{ssec:symmetry-datum-examples}

\begin{example}[Ordinary group equivariance]
\label{ex:global-group-symmetry-datum}
Suppose that a group $G$ acts on a measured space $(X,\nu)$ by
measure-preserving transformations.  Set
\[
  \Gamma=G\ltimes X,
  \qquad
  \mathcal B_G=\{b_g:g\in G\},
\]
where $b_g$ is the constant global bisection associated to $g$.  A symmetry datum based on
$(\Gamma,\mathcal B_G,\nu)$ reproduces the usual global notion of
$G$-equivariance.  Enlarging $\mathcal B_G$ to include restrictions of the
$b_g$ does not change a global intertwining relation, but it prepares the
formalism for domains on which only the restricted bisections survive.
\end{example}

\begin{example}[Restriction of an ambient symmetry]
\label{ex:restricted-symmetry-datum}
Let $G$ act on $X$ and let $D\subseteq X$ be the signal domain.  Define
\begin{equation}
  \Gamma_D:=(G\ltimes X)|_D
  =\{(g,x):x\in D,\ g\cdot x\in D\}
  \rightrightarrows D
  \label{eq:restricted-action-groupoid-general}
\end{equation}
and let $\mathcal B_D^{\mathrm{rig}}$ be the admissible family generated by the
rigid bisections $b_{g,U}$ of \eqref{eq:rigid-bisection}.  Even in the situations where the global
symmetry group
\[
  G_D=\{g\in G:gD=D\}
\]
becomes very small or even trivial, the datum
\[
  (\Gamma_D,\mathcal B_D^{\mathrm{rig}},\nu,
  R^{\mathrm{in}},R^{\mathrm{out}})
\]
can retain a rich collection of partial ambient symmetries. For instance consider that $D$ is a rectangle in the plane.  The full Euclidean group reduces to those transformations preserving $D$.  If $D$ is a generic domain, there are no non-trivial Euclidean motions preserving it.
This is the
prototype for the Euclidean construction developed in
Part~\ref{part:euclidean-groupoids}.
\end{example}

\begin{example}[Boundary-sensitive refinement]
\label{ex:boundary-sensitive-preview}
The restriction $\Gamma_D$ remembers which ambient motions have source and
target in $D$, but it need not distinguish interior points from boundary or
corner points at the level of object orbits and isotropy.  If $D$ has boundary,
one may replace $\Gamma_D$ by a subgroupoid whose arrows also preserve local
geometric data, such as tangent cones.  The resulting groupoid has separate
bulk, edge, and corner orbits.  Its rigid bisections, stratified measure, and
orbitwise representations define the boundary-aware symmetry datum used in
Parts~\ref{part:euclidean-groupoids} and~\ref{part:architecture}.
\end{example}

For clarity, the general results of Part~\ref{part:symmetry-kernels} will be
proved under the following standing assumptions, which are satisfied by the
continuous Euclidean and finite-grid models studied later:

\begin{enumerate}[label=\textup{(A\arabic*)},leftmargin=3em]
\item \label{ass:object-space}
$\Omega$ is a second-countable locally compact Hausdorff space, a compatible
stratified variant, or a finite set, and $\nu$ is a $\sigma$-finite Borel
measure.
\item \label{ass:groupoid-regularity}
$\Gamma\rightrightarrows\Omega$ is a topological, Lie, stratified, or finite
groupoid compatible with the structure of $\Omega$.
\item \label{ass:bisections}
$\mathcal B$ is an admissible family of $\nu$-preserving local bisections.  It
will be assumed ample whenever a classification in terms of all arrows or
isotropy data is invoked.
\item \label{ass:representations}
The input and output fields are finite-dimensional measurable Hermitian
bundles carrying unitary representations of $\Gamma$.
\item \label{ass:compact-isotropy}
Whenever irreducible decompositions or character formulas are used, the
relevant isotropy and joint-stabilizer groups are compact or finite.
\end{enumerate}

The next section constructs the partial transport operators associated with a
symmetry datum, defines pointwise and nonlocal equivariant channels, and proves
the bisection-equivariant kernel theorem.  The homogeneous group case will
then be recovered as a special case before the theory is applied to bounded
Euclidean domains.

\section{Equivariant channels and the bisection-equivariant kernel theorem}
\label{sec:bisection-kernel-theorem}
\label{sec:equivariant-maps} 

The symmetry datum of Definition~\ref{def:symmetry-datum} determines two
levels of equivariance.  A pointwise channel acts independently on each fibre
and is therefore constrained directly by the arrows of the groupoid.  A
nonlocal channel mixes fibres over different objects; to constrain such a map
one must transport a whole region coherently, which is the role of the chosen
family $\mathcal B$ of local bisections.  The main result of this section shows
that equivariance of an integral channel is equivalent to a transport law for
its two-point kernel.  The solutions are then classified orbit by orbit on
$\Omega\times\Omega$, with one joint-stabilizer intertwiner attached to each
orbit of pairs.

Throughout the section we work under assumptions
\ref{ass:object-space}--\ref{ass:representations}.  Compactness or finiteness
of isotropy and joint-stabilizer groups is invoked only when irreducible
decompositions or character formulas are used. 
\subsection{Representation bundles and spaces of sections}
\label{ssec:fields}

Let $(E,R)$ be a finite-dimensional measurable Hermitian representation of
$\Gamma\rightrightarrows\Omega$.  Thus every arrow
$\alpha\colon a\to b$ determines a unitary map
\[
  R(\alpha)\colon E_a\longrightarrow E_b,
\]
compatible with units and composition.  The corresponding Hilbert space of
square-integrable feature fields is the direct integral
\begin{equation}
  L^2(\Omega,E;\nu)
  :=\int_\Omega^{\oplus} E_a\,d\nu(a),
  \label{eq:L2-sections}
\end{equation}
consisting of measurable sections $\psi$ satisfying
\[
  \|\psi\|_{L^2}^2
  =\int_\Omega \|\psi(a)\|_{E_a}^2\,d\nu(a)<\infty.
\]
For a measurable set $U\subseteq\Omega$, let
\begin{equation}
  P_U^E\colon L^2(\Omega,E;\nu)\longrightarrow L^2(\Omega,E;\nu),
  \qquad
  (P_U^E\psi)(a)=\mathbf 1_U(a)\psi(a),
  \label{eq:section-projection}
\end{equation}
be the orthogonal projection onto sections supported in $U$.  When the bundle
is clear from the context we simply write $P_U$.

A linear channel is a bounded operator
\[
  \Phi\colon L^2(\Omega,E^{\mathrm{in}};\nu)
  \longrightarrow
  L^2(\Omega,E^{\mathrm{out}};\nu).
\]
The distinction between pointwise and nonlocal channels concerns how $\Phi$
interacts with the direct-integral decomposition~\eqref{eq:L2-sections}.

\subsection{Pointwise equivariant channels}
\label{ssec:pointwise}

\begin{definition}[Pointwise channel]
\label{def:pointwise-channel}
A pointwise channel is an essentially bounded measurable bundle map
\[
  \varphi=\{\varphi_a\}_{a\in\Omega},
  \qquad
  \varphi_a\colon E_a^{\mathrm{in}}\longrightarrow E_a^{\mathrm{out}},
\]
covering the identity of $\Omega$.  It acts on sections by
\[
  (\varphi\psi)(a)=\varphi_a\psi(a).
\]
It is \emph{$\Gamma$-equivariant} if, for every arrow
$\alpha\colon a\to b$,
\begin{equation}
  \varphi_b\,R^{\mathrm{in}}(\alpha)
  =R^{\mathrm{out}}(\alpha)\,\varphi_a.
  \label{eq:pointwise-equivariance}
\end{equation}
\end{definition}

Condition~\eqref{eq:pointwise-equivariance} says that $\varphi$ is a natural
transformation between the two representations.  Equivalently, the diagram
\[
\begin{tikzcd}[column sep=large,row sep=large]
E_a^{\mathrm{in}} \arrow[r,"\varphi_a"]
 \arrow[d,"R^{\mathrm{in}}(\alpha)"']
& E_a^{\mathrm{out}} \arrow[d,"R^{\mathrm{out}}(\alpha)"] \\
E_b^{\mathrm{in}} \arrow[r,"\varphi_b"']
& E_b^{\mathrm{out}}
\end{tikzcd}
\]
commutes for every arrow.  In convolutional terminology, pointwise channels
are the fibrewise linear maps usually implemented as $1\times1$ convolutions.

The next proposition gives the orbitwise classification in its most useful
form.

\begin{proposition}[Reduction of pointwise channels to isotropy]
\label{prop:pointwise-isotropy}
Let $\mathcal O\subseteq\Omega$ be an orbit and choose $a_0\in\mathcal O$.
Restriction to the fibre over $a_0$ defines a linear isomorphism
\begin{equation}
  \left\{
  \begin{array}{c}
  \Gamma|_{\mathcal O}\text{-equivariant}\\
  \text{pointwise channels}
  \end{array}
  \right\}
  \simeq
  \Hom_{\Gamma(a_0)}
  \bigl(E_{a_0}^{\mathrm{in}},E_{a_0}^{\mathrm{out}}\bigr).
  \label{eq:pointwise-isotropy-reduction}
\end{equation}
More explicitly, if
$C\in\Hom_{\Gamma(a_0)}(E_{a_0}^{\mathrm{in}},E_{a_0}^{\mathrm{out}})$
and $\alpha\colon a_0\to a$, then
\begin{equation}
  \varphi_a
  =R^{\mathrm{out}}(\alpha)
   C
   R^{\mathrm{in}}(\alpha)^{-1}
  \label{eq:pointwise-transport}
\end{equation}
defines a $\Gamma|_{\mathcal O}$-equivariant pointwise channel, and the
right-hand side is independent of the chosen arrow $\alpha$.
\end{proposition}

\begin{proof}
If $\varphi$ is equivariant, then~\eqref{eq:pointwise-equivariance} restricted
to the isotropy group $\Gamma(a_0)$ shows that $\varphi_{a_0}$ is an
isotropy intertwiner.  Conversely, let $C$ be such an intertwiner and define
$\varphi_a$ by~\eqref{eq:pointwise-transport}.  If
$\widetilde\alpha\colon a_0\to a$ is another arrow, then
$h=\alpha^{-1}\circ\widetilde\alpha\in\Gamma(a_0)$ and
\[
 R^{\mathrm{out}}(\widetilde\alpha)C
 R^{\mathrm{in}}(\widetilde\alpha)^{-1}
 =R^{\mathrm{out}}(\alpha)
  R^{\mathrm{out}}(h)C R^{\mathrm{in}}(h)^{-1}
  R^{\mathrm{in}}(\alpha)^{-1},
\]
which equals~\eqref{eq:pointwise-transport} because $C$ intertwines the
isotropy representations.  Equivariance under an arbitrary arrow in the
orbit follows by composition.
\end{proof}

If the chosen family $\mathcal B$ is ample, arrowwise equivariance can be
checked entirely through bisections: every arrow occurs as $b(a)$ for some
$b\in\mathcal B$.  Thus Proposition~\ref{prop:pointwise-isotropy} is also the
pointwise specialization of the bisection formalism developed below.

Suppose now that $\Gamma(a_0)$ is compact or finite.  Over $\C$, the two
isotropy representations admit isotypic decompositions
\begin{equation}
  E_{a_0}^{\mathrm{in}}
  \simeq
  \bigoplus_{\lambda\in\widehat{\Gamma(a_0)}}
  M_{\lambda}^{\mathrm{in}}\otimes W_\lambda,
  \qquad
  E_{a_0}^{\mathrm{out}}
  \simeq
  \bigoplus_{\lambda\in\widehat{\Gamma(a_0)}}
  M_{\lambda}^{\mathrm{out}}\otimes W_\lambda,
  \label{eq:pointwise-isotypic-decomposition}
\end{equation}
where $W_\lambda$ carries the irreducible representation $\lambda$ and
$M_\lambda^{\mathrm{in/out}}$ are multiplicity spaces.

\begin{theorem}[Structure of pointwise equivariant channels]
\label{thm:pointwise-structure}
Under the preceding compactness or finiteness hypothesis, every isotropy
intertwiner has the form
\begin{equation}
  C
  =\bigoplus_{\lambda}
   C_\lambda\otimes I_{W_\lambda},
  \qquad
  C_\lambda\in
  \Hom\bigl(M_\lambda^{\mathrm{in}},
            M_\lambda^{\mathrm{out}}\bigr).
  \label{eq:isotypic-blocks}
\end{equation}
Consequently, on each orbit, a pointwise equivariant channel is determined by
one arbitrary matrix $C_\lambda$ for every irreducible isotropy type occurring
in both the input and output fibres.  In particular,
\begin{equation}
  \dim\Hom_{\Gamma(a_0)}
  \bigl(E_{a_0}^{\mathrm{in}},E_{a_0}^{\mathrm{out}}\bigr)
  =\sum_{\lambda}
   \dim M_\lambda^{\mathrm{in}}
   \dim M_\lambda^{\mathrm{out}}.
  \label{eq:pointwise-dimension}
\end{equation}
\end{theorem}

\begin{proof}
This is Schur's lemma applied to the isotypic decompositions
\eqref{eq:pointwise-isotypic-decomposition}, followed by
Proposition~\ref{prop:pointwise-isotropy}.
\end{proof}

\begin{corollary}[Schur alternative]
\label{prop:schur}
If the two representations are irreducible on a transitive orbit, a
pointwise equivariant channel is zero unless they are equivalent.  If they are
equivalent, the intertwiner space is one-dimensional over $\C$.
\end{corollary}

\begin{remark}[Real representations]
The scalar conclusion in Corollary~\ref{prop:schur} is stated over $\C$.  For
real irreducible representations the commutant may be isomorphic to
$\R$, $\C$, or the quaternions.  All later finite-grid calculations are made
with explicit real matrix representations, so the relevant intertwiner spaces
are computed directly and no scalar-commutant assumption is needed.
\end{remark}

\begin{remark}[Equivariant quantum channels]
\label{rem:quantum}
For a quantum channel one replaces a feature fibre by an operator space such
as $\mathcal B(\mathcal H_a)$ and lets the groupoid act by $*$-automorphisms or
unitary conjugations.  The equivariance condition remains linear and is
therefore governed by the same intertwiner spaces.  Complete positivity is
then positivity of the Choi operator inside the corresponding invariant
subspace, while trace preservation is an affine constraint.  This situation will be pursued elsewhere.
\end{remark}

\subsection{Transport by local bisections}
\label{ssec:bisections}

A single arrow compares two fibres, but a nonlocal channel depends on many
source and target points simultaneously.  A local bisection supplies one arrow
at every point of a region and therefore induces a partial transport operator
on sections.

\begin{definition}[Transport operator]
\label{def:transport}
Let $(E,R)$ be a unitary representation and let
$b\colon U_b\to\Gamma$ be a $\nu$-preserving local bisection, with
$U_b'=\tau_b(U_b)$.  The transport operator associated with $b$ is
\begin{equation}
  \Lambda_R(b)
  \colon L^2(\Omega,E;\nu)\longrightarrow L^2(\Omega,E;\nu),
  \label{eq:transport-operator-map}
\end{equation}
defined by
\begin{equation}
  \bigl(\Lambda_R(b)\psi\bigr)(y)
  =
  \begin{cases}
  R\bigl(b(\tau_b^{-1}y)\bigr)
  \psi(\tau_b^{-1}y), & y\in U_b',\\[3pt]
  0, & y\notin U_b'.
  \end{cases}
  \label{eq:transport}
\end{equation}
We write $\Lambda^{\mathrm{in}}(b)$ and
$\Lambda^{\mathrm{out}}(b)$ for the operators associated with the two
representations in the symmetry datum.
\end{definition}

\begin{proposition}[Partial representation of the bisection pseudogroup]
\label{prop:transport-partial-isometry}
For every $b\in\mathcal B$, $\Lambda_R(b)$ is a partial isometry satisfying
\begin{align}
  \Lambda_R(b)^*\Lambda_R(b)&=P_{U_b}^{E},
  &
  \Lambda_R(b)\Lambda_R(b)^*&=P_{U_b'}^{E},
  \label{eq:transport-projections}\\
  \Lambda_R(b)^*&=\Lambda_R(b^{-1}).
  \label{eq:transport-adjoint}
\end{align}
If $b'$ and $b$ are composable, then
\begin{equation}
  \Lambda_R(b')\Lambda_R(b)
  =\Lambda_R(b'\circ b)
  \label{eq:transport-composition}
\end{equation}
on the natural initial space.  In particular, a global bisection acts
unitarily.
\end{proposition}

\begin{proof}
The measure-preserving property of $\tau_b$ and unitarity of $R(b(x))$ give
\[
 \|\Lambda_R(b)\psi\|_{L^2}^2
 =\int_{U_b}\|\psi(x)\|^2\,d\nu(x)
 =\|P_{U_b}^{E}\psi\|_{L^2}^2.
\]
The formulas for the adjoint and final projection follow by applying the same
calculation to $b^{-1}$.  Equation~\eqref{eq:transport-composition} follows
from the representation law for $R$ and the composition law for local
bisections.
\end{proof}
\begin{remark}[Recovery of the usual induced action on feature fields] For an action groupoid $G\ltimes X$ and a constant global bisection $b_g$,
Definition~\ref{def:transport} recovers the usual induced action of $g$ on
feature fields.  On a restricted domain, the same formula gives only a partial
isometry because the transformation is defined on a proper subset.   
\end{remark}

\subsection{Bisection-equivariant nonlocal channels}
\label{ssec:convolution}

We now formulate equivariance for a bounded operator that mixes different
fibres.  The projections in the definition are essential: a local bisection
can constrain only the part of the input and output that it actually
transports.

\begin{definition}[$\mathcal B$-equivariant channel]
\label{def:B-equivariant}
Let
\[
 \Phi\colon L^2(\Omega,E^{\mathrm{in}};\nu)
 \longrightarrow L^2(\Omega,E^{\mathrm{out}};\nu)
\]
be bounded.  We say that $\Phi$ is $\mathcal B$-equivariant if, for every
$b\in\mathcal B$, with source domain $U=U_b$ and image domain $U'=U_b'$, one
has
\begin{equation}
  P_{U'}^{E^{\mathrm{out}}}\,
  \Phi\,\Lambda^{\mathrm{in}}(b)
  =
  \Lambda^{\mathrm{out}}(b)\,
  P_U^{E^{\mathrm{out}}}\,
  \Phi\,P_U^{E^{\mathrm{in}}}.
  \label{eq:operator-equivariance}
\end{equation}
\end{definition}

The left-hand side transports an input supported in $U$ to $U'$, applies the
channel, and observes the output in $U'$.  The right-hand side first applies
the channel inside the source window $U$ and then transports the resulting
output to $U'$.  If $b$ is global, all projections are identities and
\eqref{eq:operator-equivariance} reduces to the ordinary intertwining relation
\[
  \Phi\Lambda^{\mathrm{in}}(b)
  =\Lambda^{\mathrm{out}}(b)\Phi.
\]
For a proper local bisection, the projected identity is the natural analogue
of global equivariance.

\begin{remark}[Layerwise character]
Unlike ordinary global equivariance, the identity
\eqref{eq:operator-equivariance} on a maximal partial domain is not generally
preserved under unrestricted composition: intermediate points may leave the
domain on which the bisection is defined.  The resulting erosion of admissible
domains is developed systematically in
Part~\ref{part:filtered-equivariance}.  The present section concerns the exact
classification of a single linear layer.
\end{remark}

An integral channel is specified by a measurable section
\[
 K\in
 \Gamma_{\mathrm{meas}}
 \bigl(\Omega\times\Omega,
       \Hom(p_2^*E^{\mathrm{in}},p_1^*E^{\mathrm{out}})\bigr)\, ,
\]
where $p_1(y,x)=y$ and $p_2(y,x)=x$, i.e., $K(y,x)\colon E_x^{\mathrm{in}}\to E_y^{\mathrm{out}}$,
through
\begin{equation}
  (\Phi_K\psi)(y)
  =\int_\Omega K(y,x)\psi(x)\,d\nu(x)\, .
  \label{eq:integral-operator}
\end{equation}
We assume that
$K$ is locally integrable and satisfies a standard boundedness condition, for
example a Schur bound on the relevant domains, so that
\eqref{eq:integral-operator} defines a bounded operator and locally integrable
kernels are unique up to $\nu\times\nu$-null sets.  These hypotheses are
automatic for the finite models and for the compactly supported bounded-domain
kernels used later.

\subsection{The bisection-equivariant kernel theorem}
\label{ssec:bisection-kernel-theorem}

\begin{theorem}[Bisection-equivariant kernel theorem]
\label{thm:bisection-kernel}
\label{thm:groupoid-convolution} 
Let $\mathfrak S_{\mathrm{in,out}} = (\Gamma, \mathcal{B}, \nu ; (E^{\mathrm{in}},R^{\mathrm{in}}), (E^\mathrm{out},R^\mathrm{out}))$ be a measured bisection symmetry datum and
let $\Phi_K$ be an integral channel of the class just described.  Then
$\Phi_K$ is $\mathcal B$-equivariant if and only if, for every
$b\in\mathcal B$,
\begin{equation}\label{eq:kernel_constraint}
  K\bigl(\tau_b(y),\tau_b(x)\bigr)
  =R^{\mathrm{out}}\bigl(b(y)\bigr)
   K(y,x)
   R^{\mathrm{in}}\bigl(b(x)\bigr)^{-1}
\end{equation}
for $\nu\times\nu$-almost every $(y,x)\in U_b\times U_b$.
\end{theorem}

\begin{proof}
Fix $b\in\mathcal B$, write $U=U_b$, $U'=U_b'$, and evaluate
\eqref{eq:operator-equivariance} at $y'=\tau_b(y)$ with $y\in U$.  For an
input section $\psi$, Definition~\ref{def:transport} and the change of
variables $x'=\tau_b(x)$ give
\begin{align*}
 &\bigl(P_{U'}^{E^{\mathrm{out}}}
        \Phi_K\Lambda^{\mathrm{in}}(b)\psi\bigr)(\tau_b(y))\\
 &\quad=
 \int_{U'}K(\tau_b(y),x')
 R^{\mathrm{in}}\bigl(b(\tau_b^{-1}x')\bigr)
 \psi(\tau_b^{-1}x')\,d\nu(x')\\
 &\quad=
 \int_U K\bigl(\tau_b(y),\tau_b(x)\bigr)
 R^{\mathrm{in}}\bigl(b(x)\bigr)
 \psi(x)\,d\nu(x).
\end{align*}
On the other hand,
\begin{align*}
 &\bigl(\Lambda^{\mathrm{out}}(b)
 P_U^{E^{\mathrm{out}}}\Phi_KP_U^{E^{\mathrm{in}}}\psi\bigr)
 (\tau_b(y))\\
 &\quad=
 R^{\mathrm{out}}\bigl(b(y)\bigr)
 \int_U K(y,x)\psi(x)\,d\nu(x).
\end{align*}
If~\eqref{eq:kernel_constraint} holds, the two expressions coincide.  Conversely,
if the operator identity holds for all compactly supported essentially bounded
sections $\psi$, uniqueness of locally integrable kernels implies
\[
 K\bigl(\tau_b(y),\tau_b(x)\bigr)
 R^{\mathrm{in}}\bigl(b(x)\bigr)
 =R^{\mathrm{out}}\bigl(b(y)\bigr)K(y,x)
\]
for almost every $(y,x)\in U\times U$, which is equivalent to
\eqref{eq:kernel_constraint}.
\end{proof}

\begin{remark}[Object-space kernels versus groupoid convolution]
Theorem~\ref{thm:bisection-kernel} concerns kernels on the object-pair space
$\Omega\times\Omega$, transported by local bisections.  This should be
distinguished from the standard convolution algebra of functions on the
arrow space of a locally compact groupoid, where multiplication is defined
using a Haar system on source or target fibres.  The two constructions are
related in action-groupoid and induced-representation settings, but they are
not identical.  The term ``bisection-equivariant kernel theorem'' will be used
throughout to avoid this ambiguity.
\end{remark}

\begin{remark}[Distributional kernels]
\label{rem:schwartz}
The theorem was stated for ordinary locally integrable kernels because that is
the class needed in the continuous bounded-domain and finite-grid models.
Pointwise channels correspond formally to kernels supported on the diagonal,
$K(y,x)=\delta_x(y)\varphi_x$, and restriction or trace maps may require
kernels supported on lower-dimensional correspondences.  On smooth manifolds,
the Schwartz kernel theorem extends the discussion to continuous maps from
compactly supported test sections to distributional sections.  This does not
mean that every bounded operator on an arbitrary $L^2$-space has an ordinary
measurable kernel.
\end{remark}

\begin{remark}[Diagonal kernels]
If $K(y,x)=\delta_x(y)\varphi_x$, condition
\eqref{eq:kernel_constraint} reduces to
\eqref{eq:pointwise-equivariance} for every arrow attained by
$\mathcal B$.  Hence, when $\mathcal B$ is ample, the pointwise theory is the
diagonal specialization of Theorem~\ref{thm:bisection-kernel}.
\end{remark}

\subsection{Pair orbits and joint stabilizers}
\label{ssec:reduction}

The kernel constraint (\ref{eq:kernel_constraint}) has an orbit-theoretic interpretation.  Every local
bisection acts diagonally on the part of $\Omega\times\Omega$ contained in its
domain:
\begin{equation}
  (y,x)\longmapsto
  b\cdot(y,x)
  :=\bigl(\tau_b(y),\tau_b(x)\bigr),
  \qquad (y,x)\in U_b\times U_b.
  \label{eq:pair-pseudogroup-action}
\end{equation}
Because $\mathcal B$ is closed under restriction, inverse, and composition,
these partial transformations generate an equivalence relation on
$\Omega\times\Omega$.  Its equivalence classes will be called
$\mathcal B$-orbits of pairs and denoted by
$\mathcal O_{\mathcal B}(y,x)$.

\begin{definition}[Joint stabilizer]
\label{def:joint-stabilizer}
For a pair $p=(y,x)$, its joint stabilizer relative to $\mathcal B$ is
\begin{equation}
  S_p^{\mathcal B}
  :=\left\{
  \bigl(b(y),b(x)\bigr)
  \in\Gamma(y)\times\Gamma(x):
  \begin{array}{l}
  b\in\mathcal B,\ y,x\in U_b,\\[-2pt]
  \tau_b(y)=y,\ \tau_b(x)=x
  \end{array}
  \right\}.
  \label{eq:joint-stabilizer}
\end{equation}
It is a subgroup of $\Gamma(y)\times\Gamma(x)$ under componentwise
composition.
\end{definition}

The joint stabilizer acts on
$\Hom(E_x^{\mathrm{in}},E_y^{\mathrm{out}})$ by
\begin{equation}
  (\beta,\alpha)\cdot C
  :=R^{\mathrm{out}}(\beta)
    C
    R^{\mathrm{in}}(\alpha)^{-1}.
  \label{eq:joint-stabilizer-action}
\end{equation}
We denote the fixed subspace by
\begin{equation}
  \Hom_{S_p^{\mathcal B}}
  \bigl(E_x^{\mathrm{in}},E_y^{\mathrm{out}}\bigr)
  :=\left\{C:
  R^{\mathrm{out}}(\beta)C
  R^{\mathrm{in}}(\alpha)^{-1}=C
  \text{ for all }(\beta,\alpha)\in S_p^{\mathcal B}\right\}.
  \label{eq:pair-intertwiners}
\end{equation}

\begin{proposition}[Reduction to pair orbits]
\label{prop:pair-orbits}
Let $K$ satisfy the kernel constraint~\eqref{eq:kernel_constraint}.

\begin{enumerate}[label=\textup{(\roman*)}]
\item For every pair $p=(y,x)$,
\begin{equation}
  K(y,x)\in
  \Hom_{S_p^{\mathcal B}}
  \bigl(E_x^{\mathrm{in}},E_y^{\mathrm{out}}\bigr).
  \label{eq:kernel-at-representative}
\end{equation}

\item If $q=b\cdot p$ for some $b\in\mathcal B$, then $K(q)$ is determined by
$K(p)$ through
\begin{equation}
  K\bigl(\tau_b(y),\tau_b(x)\bigr)
  =R^{\mathrm{out}}\bigl(b(y)\bigr)
   K(y,x)
   R^{\mathrm{in}}\bigl(b(x)\bigr)^{-1}.
  \label{eq:pair-orbit-transport}
\end{equation}

\item Conversely, choose one representative $p$ in each pair orbit and an
element $C_p$ of the fixed space~\eqref{eq:pair-intertwiners}.  Transporting
$C_p$ by~\eqref{eq:pair-orbit-transport} gives a well-defined orbitwise
solution of the kernel constraint.  Different choices of a bisection carrying
$p$ to the same pair give the same result precisely because $C_p$ is fixed by
the joint stabilizer.
\end{enumerate}
\end{proposition}

\begin{proof}
If $(\beta,\alpha)=(b(y),b(x))\in S_p^{\mathcal B}$, then
$\tau_b(y)=y$ and $\tau_b(x)=x$.  Substitution into
\eqref{eq:kernel_constraint} yields~\eqref{eq:kernel-at-representative}.  The
second statement is the kernel constraint itself.  For the converse, suppose
$b$ and $\widetilde b$ both carry $p$ to $q$.  The composite
$\widetilde b^{-1}\circ b$ fixes $p$, and the two transported values differ by
the action of an element of $S_p^{\mathcal B}$.  They therefore coincide when
$C_p$ lies in the fixed space.
\end{proof}

Proposition~\ref{prop:pair-orbits} is the nonlocal analogue of
Proposition~\ref{prop:pointwise-isotropy}.  Pointwise channels are classified
by isotropy intertwiners on orbits of objects; integral channels are classified
by joint-stabilizer intertwiners on orbits of pairs.

\begin{remark}[Measurability of orbitwise data]
The proposition is an algebraic classification on each pair orbit.  To obtain
a globally measurable kernel one must choose orbit representatives and free
data measurably, or work with explicit invariants that parameterize the orbit
space.  This is automatic for finite groupoids and will be carried out
constructively for the Euclidean and pixel-grid examples below.
\end{remark}

If $S_p^{\mathcal B}$ is compact or finite, let $d\mu_{S_p}$ be its normalized
Haar measure and let
\[
  \chi_p^{\mathrm{out}}(\beta,\alpha)
  :=\operatorname{tr}R^{\mathrm{out}}(\beta),
  \qquad
  \chi_p^{\mathrm{in}}(\beta,\alpha)
  :=\operatorname{tr}R^{\mathrm{in}}(\alpha).
\]
Schur orthogonality gives
\begin{equation}
  \dim
  \Hom_{S_p^{\mathcal B}}
  \bigl(E_x^{\mathrm{in}},E_y^{\mathrm{out}}\bigr)
  =\int_{S_p^{\mathcal B}}
   \chi_p^{\mathrm{out}}(s)
   \overline{\chi_p^{\mathrm{in}}(s)}\,d\mu_{S_p}(s).
  \label{eq:pair-intertwiner-character-formula}
\end{equation}
Equivalently, after restricting the two fibre representations to the two
projections of $S_p^{\mathcal B}$ and decomposing them into irreducibles, the
dimension is the sum of the products of matching multiplicities.  If the
joint stabilizer is trivial, the full matrix space is allowed:
\[
  \dim\Hom_{S_p^{\mathcal B}}
  \bigl(E_x^{\mathrm{in}},E_y^{\mathrm{out}}\bigr)
  =\dim E_x^{\mathrm{in}}\,\dim E_y^{\mathrm{out}}.
\]

\begin{remark}[Mixed object orbits]
A bisection preserves the groupoid orbit of each individual object, but it may
act simultaneously on a pair whose two entries belong to different object
orbits.  Therefore a disconnected or stratified groupoid does not force a
nonlocal kernel to be block diagonal with respect to object orbits.  Mixed
pair orbits give rise to constrained cross-orbit, and later cross-stratum,
couplings.  This observation is central to the boundary-aware layers of
Part~\ref{part:euclidean-groupoids}.
\end{remark}

\subsection{Recovery of homogeneous and steerable kernels}
\label{ssec:homogeneous-recovery}

We finish Part~\ref{part:symmetry-kernels} by checking that the familiar
steerable convolution constraint is recovered when the symmetry is global and
homogeneous.

\begin{proposition}[Action-groupoid kernel constraint]
\label{prop:action-groupoid-kernel}
Let a group $G$ act on $\Omega$, let $\Gamma=G\ltimes\Omega$, and let
$\mathcal B=\{b_g:g\in G\}$ be the family of constant global bisections.  If
$R^{\mathrm{in/out}}(g,x)$ denote the fibre actions of the arrow $(g,x)$, then
an integral channel is $\mathcal{B}$- equivariant if and only if it is $G$-equivariant, that is if and only if
\begin{equation}
  K(g\cdot y,g\cdot x)
  =R^{\mathrm{out}}(g,y)
   K(y,x)
   R^{\mathrm{in}}(g,x)^{-1}
  \label{eq:action-groupoid-kernel-constraint}
\end{equation}
for every $g\in G$ and almost every $(y,x)$.
\end{proposition}

\begin{proof}
This is Theorem~\ref{thm:bisection-kernel} applied to the constant global
bisections of the action groupoid.
\end{proof}

The standard Euclidean steerable-CNN setting is obtained from a semidirect
product.  Let $H\leq O(d)$ be compact and let
\[
  G=\R^d\rtimes H
\]
act on $\R^d$ by $(t,h)\cdot x=hx+t$.  Let
$\rho^{\mathrm{in/out}}$ be unitary representations of $H$, and use the usual
trivializations of the associated homogeneous bundles so that
\begin{equation}
  R^{\mathrm{in/out}}((t,h),x)
  =\rho^{\mathrm{in/out}}(h).
  \label{eq:semidirect-fibre-action}
\end{equation}

\begin{corollary}[Steerable convolution kernel]
\label{cor:euclidean-kernel}
Under the preceding assumptions, an integral channel is
$\R^d\rtimes H$-equivariant if and only if it has the convolutional form
\begin{equation}
  K(y,x)=k(y-x),
  \qquad
  (\Phi\psi)(y)
  =\int_{\R^d}k(y-x)\psi(x)\,dx,
  \label{eq:convolutional-form}
\end{equation}
where
\begin{equation}
  k(h\xi)
  =\rho^{\mathrm{out}}(h)
   k(\xi)
   \rho^{\mathrm{in}}(h)^{-1},
  \qquad h\in H.
  \label{eq:steerable-constraint}
\end{equation}
\end{corollary}

\begin{proof}
Apply~\eqref{eq:action-groupoid-kernel-constraint} first to translations
$(t,I)$.  Since their fibre action is trivial,
$K(y+t,x+t)=K(y,x)$, hence $K(y,x)=k(y-x)$.  Applying the same constraint to
$(0,h)$ gives~\eqref{eq:steerable-constraint}.  The converse follows by
combining the two identities.
\end{proof}

\begin{remark}[Pair stabilizers in the homogeneous case]
\label{rem:dimension-counts}
The pair orbits of $\R^d\rtimes H$ are the $H$-orbits of the displacement
$\xi=y-x$.  The joint stabilizer of a pair with displacement $\xi$ is
isomorphic to
\[
  H_\xi=\{h\in H:h\xi=\xi\}.
\]
Thus the free value of the kernel on an orbit lies in an $H_\xi$-intertwiner
space.  For $H=O(2)$ and $\xi\neq0$, $H_\xi\simeq\Z_2$, generated by the
reflection across the line spanned by $\xi$; at $\xi=0$ the stabilizer is the
full group $O(2)$.  This is the pair-orbit form of the harmonic kernel
classification used in steerable CNNs.
\end{remark}

\begin{remark}[What changes on a bounded domain]
\label{rem:towards-stratified}
On a bounded domain, the constant bisections of the ambient Euclidean group
survive only locally, and the relevant pair orbits are no longer determined
solely by displacement.  After refining the restricted action groupoid by
tangent-cone data, the object space separates into bulk, edge, and corner
orbits.  Proposition~\ref{prop:pair-orbits} then produces both the familiar
bulk steerable kernels and new mixed pair orbits coupling different strata.
The explicit geometric and discrete classifications are the subject of
Part~\ref{part:euclidean-groupoids}.
\end{remark}

\part{Euclidean and tangent-cone groupoids}
\label{part:euclidean-groupoids}

\section{The Euclidean groupoid and bounded domains}
\label{sec:euclidean_group}

Part~\ref{part:symmetry-kernels} developed the kernel theory for an abstract
symmetry datum
\(
  (\Gamma\rightrightarrows\Omega,\mathcal B,\nu;
  R^{\mathrm{in}},R^{\mathrm{out}})
\).
We now specialize that theory to signals supported on bounded subsets of the
Euclidean plane.  The ambient group is the full Euclidean group
\(\E(2)\), but the bounded domain is not, in general, invariant under its
action.  The action groupoid and its restrictions retain the admissible
point-to-point rigid motions, while a further tangent-cone condition will
separate the interior, smooth boundary, and corner strata.

\subsection{The Euclidean group and its compact isotropy types}
\label{ssec:euclidean-group}

Let \(\R^2\) carry its standard Euclidean metric.  Every Euclidean isometry is
affine, hence has the form
\[
  x\longmapsto Ax+v,
  \qquad v\in\R^2,\quad A\in\OO(2).
\]
Accordingly,
\begin{equation}
  \E(2)=\R^2\rtimes\OO(2),
  \label{eq:E2-semidirect}
\end{equation}
with multiplication and inversion
\begin{equation}
  (v,A)(w,B)=(v+Aw,AB),
  \qquad
  (v,A)^{-1}=(-A^{-1}v,A^{-1}).
  \label{eq:E2-group-law}
\end{equation}
The connected orientation-preserving subgroup is
\(
  \SE(2)=\R^2\rtimes\SO(2)
\).
The orthogonal factor, rather than the full noncompact group \(\E(2)\), will
control the internal feature types of the Euclidean action groupoid.

For later reference we recall the irreducible finite-dimensional
representations of \(\OO(2)\):
\begin{equation}
  \widehat{\OO(2)}
  =\Irr(\OO(2))
  =\{\mathbf 1,\det\}\cup\{\rho_n:n\geq1\}.
  \label{irrepsO(2)}
\end{equation}
Here \(\mathbf 1\) is the trivial representation, \(\det\) is the determinant
character, and \(\rho_n\) is the two-dimensional representation determined by
\begin{equation}
  \rho_n(R_\theta)=
  \begin{pmatrix}
    \cos(n\theta)&-\sin(n\theta)\\
    \sin(n\theta)& \cos(n\theta)
  \end{pmatrix},
  \qquad
  \rho_n(\sigma)=
  \begin{pmatrix}1&0\\0&-1\end{pmatrix}
  \label{eq:O2-irreps}
\end{equation}
for a fixed reflection \(\sigma\in\OO(2)\).  After restriction to any
reflection subgroup \(\langle\sigma\rangle\simeq\Z_2\),
\begin{equation}
  \mathbf 1|_{\Z_2}=\mathbf 1,
  \qquad
  \det|_{\Z_2}=\sgn,
  \qquad
  \rho_n|_{\Z_2}\simeq\mathbf 1\oplus\sgn.
  \label{eq:O2-to-Z2-branching}
\end{equation}
The irreducible unitary representations of \(\Z_2\) are
\begin{equation}
  \widehat{\Z_2}=\{\mathbf 1,\sgn\}.
  \label{eq:irrepsZ2}
\end{equation}

\begin{remark}
The irreducible unitary representations of the group \(\E(2)\) itself are
classified by Mackey's theory and are not exhausted by
\(\widehat{\OO(2)}\).  In the present construction the relevant object is the
transitive action groupoid over the configuration space, whose representation
category is equivalent to that of an isotropy group.  It is this distinction
that makes \(\OO(2)\), and later the smaller boundary isotropy groups, the
feature-type groups.
\end{remark}

\subsection{The Euclidean action groupoid and its feature bundles}
\label{ssec:E2-action-groupoid}

The standard action of \(\E(2)\) on \(\R^2\) gives the action groupoid
\begin{equation}
  \Gamma_{\R^2}:=\E(2)\ltimes\R^2\rightrightarrows\R^2,
  \label{eq:E2-action-groupoid}
\end{equation}
with
\[
  s(g,x)=x,
  \qquad
  t(g,x)=g\cdot x,
  \qquad
  (h,gx)\circ(g,x)=(hg,x).
\]

\begin{proposition}[Euclidean action groupoid]
\label{prop:E2-action-groupoid}
The groupoid \(\Gamma_{\R^2}\) is transitive.  Its isotropy group at every
\(x\in\R^2\) is canonically isomorphic, after choosing the origin at \(x\), to
\(\OO(2)\).  At the origin,
\begin{equation}
  \Gamma_{\R^2}(0)=\{((0,A),0):A\in\OO(2)\}\simeq\OO(2).
  \label{eq:E2isotropy}
\end{equation}
Consequently,
\(
  \Rep(\Gamma_{\R^2})\simeq\Rep(\OO(2))
\).
\end{proposition}

\begin{proof}
Transitivity follows because the translation \((y-x,I)\) sends \(x\) to
\(y\).  An isometry fixes \(x\) precisely when it has the form
\(
  z\mapsto x+A(z-x)
\)
with \(A\in\OO(2)\).  The representation statement is then
Theorem~\ref{thm:imprimitivity}.
\end{proof}

Let \(\rho:\OO(2)\to U(V_\rho)\) be a finite-dimensional unitary
representation.  The associated groupoid representation bundle is
\begin{equation}
  E_\rho=\E(2)\times_{\OO(2)}V_\rho
  \longrightarrow \E(2)/\OO(2)\simeq\R^2.
  \label{eq:E2-associated-bundle}
\end{equation}
The translation section \(x\mapsto(x,I)\) trivializes this bundle.  In the
resulting identification \(E_\rho\simeq\R^2\times V_\rho\), an arrow
\(((v,A),x):x\to Ax+v\) acts by
\begin{equation}
  R_\rho((v,A),x)(x,\xi)
  =(Ax+v,\rho(A)\xi).
  \label{eq:trivialization}
\end{equation}
Thus the translation part moves the base point and the orthogonal part acts on
the internal feature fibre.  Corollary~\ref{cor:euclidean-kernel} is the
corresponding homogeneous steerable-kernel theorem.

\subsection{Restriction to a bounded domain}
\label{ssec:restricted-E2-groupoid}

Let \(D\subset\R^2\) be a nonempty domain, not assumed invariant under
\(\E(2)\).

\begin{definition}[Restricted Euclidean groupoid]
\label{def:restricted-E2-groupoid}
The restriction of the Euclidean action groupoid to \(D\) is
\begin{equation}
  \Gamma_D
  :=(\E(2)\ltimes\R^2)|_D
  =\{(g,x) \mid x\in D,\ g\cdot x\in D\}
  \rightrightarrows D.
  \label{eq:restricted-E2-groupoid}
\end{equation}
\end{definition}

\begin{proposition}[Object-orbit and isotropy structure of \(\Gamma_D\)]
\label{prop:restricted-E2-transitive}
The groupoid \(\Gamma_D\) is transitive, and its isotropy group at every point
is isomorphic to \(\OO(2)\).  Hence
\begin{equation}
  \Rep(\Gamma_D)\simeq\Rep(\OO(2)).
  \label{eq:restricted-rep-equivalence}
\end{equation}
In particular, the object-orbit and isotropy data of \(\Gamma_D\) do not
distinguish interior points from boundary points.
\end{proposition}

\begin{proof}
For \(x,y\in D\), the translation \(y-x\) defines an arrow from \(x\) to
\(y\).  The isotropy computation is unchanged from
Proposition~\ref{prop:E2-action-groupoid}.
\end{proof}

The restricted groupoid should not be confused with the action groupoid of the
global symmetry group of the domain,
\begin{equation}
  \E(2)_D\ltimes D,
  \qquad
  \E(2)_D:=\{g\in\E(2) \mid gD=D\}.
  \label{eq:global-domain-symmetry-group}
\end{equation}
For a generic bounded domain, \(\E(2)_D\) is trivial or finite, whereas
\(\Gamma_D\) contains all ambient rigid motions whose source and target happen
to lie in \(D\).

\begin{remark}[In what sense is \(\Gamma_D\) boundary-blind?]
\label{rem:ordinary-restriction-boundary-blind}
Proposition~\ref{prop:restricted-E2-transitive} is a statement about object
orbits and isotropy groups.  The full arrow set, the domains of local
bisections, and the induced pair-orbit geometry still depend on \(D\).  Thus
\(\Gamma_D\) is not devoid of boundary information; rather, it fails to encode
the boundary as a separate object stratum.  The tangent-cone refinement below
is designed precisely to introduce that stratification.
\end{remark}

\begin{remark}[Rigid bisections need not be ample for \(\Gamma_D\)]
\label{rem:rigid-not-ample-restriction}
The distinction between the groupoid and the selected bisection family is
already visible here.  A rigid bisection has the form
\(
  b_{g,U}(x)=(g,x)
\)
with \(gU\subset D\).  If an arrow of \(\Gamma_D\) sends an interior point to
a boundary point, no open neighborhood of the source can be carried by that
same rigid motion entirely into \(D\).  Hence the rigid family is generally
not ample for \(\Gamma_D\).  Indeed, with the stratified topology such an
interior-to-boundary arrow need not lie on any local bisection at all.  The
tangent-cone condition removes precisely these cross-stratum arrows; for
polygonal domains the remaining arrows are attained by rigid bisections, as
shown in Lemma~\ref{lem:ample-rigid}.
\end{remark}

\section{Tangent-cone groupoids and stratified representations}
\label{sec:tangent-cone-groupoid}

The ordinary restriction \(\Gamma_D\) relates any two points of \(D\).  To
make the local boundary geometry part of the symmetry datum, we now retain
only those arrows whose linear part identifies the tangent cones of the
domain.  For polygonal domains this produces a groupoid with bulk, edge, and
corner orbits and with a rich ample pseudogroup of rigid local bisections.

\subsection{Tangent cones and the refined groupoid}
\label{ssec:tangent-cones}

Assume that \(D\subset\R^2\) is compact and that its boundary is piecewise
\(C^1\), with finitely many corner points.  The following curve definition is
sufficient in this setting and agrees with the usual Bouligand tangent cone \cite{Au09}.

\begin{definition}[Tangent cone]
\label{def:tangent-cone}
For \(x\in D\), the tangent cone of \(D\) at \(x\) is
\begin{equation}
  T_x^cD
  :=\{\dot\gamma(0)\mid 
  \gamma:[0,\varepsilon)\to D\text{ is }C^1,
  \ \gamma(0)=x\}
  \subset T_x\R^2.
  \label{eq:tangent-cone}
\end{equation}
\end{definition}

At an interior point, \(T_x^cD=\R^2\).  At a smooth boundary point it is the
closed inward half-plane, and at a corner it is a closed sector whose opening
angle is the interior angle of \(D\) at that corner.

\begin{definition}[Tangent-cone restricted Euclidean groupoid]
\label{def:tc-groupoid}
The tangent-cone groupoid of \(D\) is the subgroupoid
\begin{equation}
  \Gamma_D^{\tc}
  :=\{((v,A),x)\in\Gamma_D \mid
  A(T_x^cD)=T_{Ax+v}^cD\}
  \rightrightarrows D.
  \label{eq:tc-groupoid}
\end{equation}
\end{definition}

\begin{proposition}
\label{prop:tc-subgroupoid}
The set \(\Gamma_D^{\tc}\) is a subgroupoid of \(\Gamma_D\).
\end{proposition}

\begin{proof}
The unit arrow has linear part \(I\) and preserves every tangent cone.  If
\(A(T_x^cD)=T_y^cD\) and \(B(T_y^cD)=T_z^cD\), then
\((BA)(T_x^cD)=T_z^cD\), so the condition is closed under composition.
Finally, equality rather than inclusion gives
\(A^{-1}(T_y^cD)=T_x^cD\), proving closure under inverses.
\end{proof}

\subsection{Orbit and isotropy stratification}
\label{ssec:tc-orbits-isotropy}

Write \(\partial_{\mathrm{sm}}D\) for the smooth part of the boundary and
\(\mathcal C_\alpha(D)\) for the set of corners of interior angle \(\alpha\).

\begin{theorem}[Tangent-cone orbit classification]
\label{thm:tc-orbit-classification}
For a compact piecewise-\(C^1\) planar domain, the orbits of
\(\Gamma_D^{\tc}\) are determined by the orthogonal congruence type of the
tangent cone:
\begin{enumerate}[label=\textup{(\roman*)}]
\item all interior points form one orbit \(\Omega_{\mathrm b}=\Int(D)\);
\item all smooth boundary points form one orbit
  \(\Omega_{\mathrm e}=\partial_{\mathrm{sm}}D\);
\item two corners lie in the same orbit if and only if they have the same
  interior angle, so the corner orbits are the nonempty sets
  \(\mathcal C_\alpha(D)\).
\end{enumerate}
The isotropy groups are
\begin{equation}
  \Gamma_D^{\tc}(x)\simeq
  \begin{cases}
    \OO(2), & x\in\Omega_{\mathrm b},\\
    \Z_2,   & x\in\Omega_{\mathrm e},\\
    \Z_2,   & x\in\mathcal C_\alpha(D).
  \end{cases}
  \label{eq:tc-isotropy-groups}
\end{equation}
At a smooth boundary point the nontrivial isotropy element is reflection in
the inward normal line; at a corner it is reflection in the bisector of the
sector.
\end{theorem}

\begin{proof}
An arrow identifies the tangent cones by an orthogonal map, so cone type is
constant on every orbit.  Conversely, any two full planes, any two closed
half-planes, and any two sectors with the same opening angle are related by an
orthogonal map.  After choosing such a map \(A\), the translation
\(v=y-Ax\) gives an arrow from \(x\) to \(y\).

For isotropy, an interior cone is the whole plane and is preserved by all of
\(\OO(2)\).  The subgroup preserving a closed half-plane fixes its inward
normal and consists of the identity and reflection in the normal line.  The
subgroup preserving a proper sector consists of the identity and reflection
in its bisector.
\end{proof}

For a rectangle all four corner angles are \(\pi/2\), and there are exactly
three orbits:
\begin{equation}
  D=\Omega_{\mathrm b}\sqcup\Omega_{\mathrm e}\sqcup\Omega_{\mathrm c},
  \qquad
  \Omega_{\mathrm b}=\Int(D),\quad
  \Omega_{\mathrm e}=\partial D\setminus\Omega_{\mathrm c},\quad
  |\Omega_{\mathrm c}|=4.
  \label{eq:rectangle-strata}
\end{equation}

\begin{figure}[ht]
\centering
\begin{tikzpicture}[scale=0.82,>=Latex]
  \draw[thick] (0,0) rectangle (8,4.8);
  \node at (4,2.4) {$\Omega_{\mathrm b}$};
  \node[rotate=90] at (-0.32,2.4) {$\Omega_{\mathrm e}$};
  \node[rotate=90] at (8.32,2.4) {$\Omega_{\mathrm e}$};
  \node at (4,-0.32) {$\Omega_{\mathrm e}$};
  \node at (4,5.12) {$\Omega_{\mathrm e}$};
  \fill (0,0) circle (2pt) node[below left] {$\Omega_{\mathrm c}$};
  \fill (8,0) circle (2pt) node[below right] {$\Omega_{\mathrm c}$};
  \fill (8,4.8) circle (2pt) node[above right] {$\Omega_{\mathrm c}$};
  \fill (0,4.8) circle (2pt) node[above left] {$\Omega_{\mathrm c}$};

  \fill (2.9,2.3) circle (1.6pt);
  \draw[->] (2.9,2.3)--(3.7,2.3);
  \draw[->] (2.9,2.3)--(2.9,3.1);
  \node[below] at (2.9,2.15) {$T_x^cD=\R^2$};

  \fill (5.6,0) circle (1.6pt);
  \draw[->] (5.6,0)--(6.3,0);
  \draw[->] (5.6,0)--(5.6,0.8);
  \draw[->] (5.6,0)--(4.9,0);
  \node[right] at (5.7,0.62) {half-plane};

  \fill (0,4.8) circle (1.6pt);
  \draw[->] (0,4.8)--(0.8,4.8);
  \draw[->] (0,4.8)--(0,4.0);
  \node[right] at (0.55,4.35) {sector};
\end{tikzpicture}   
\caption{The tangent-cone stratification of a rectangle.  The ordinary
restricted groupoid has one object orbit, whereas
\(\Gamma_D^{\tc}\) has bulk, edge, and corner orbits.}
\label{fig:tc-strata}
\end{figure}
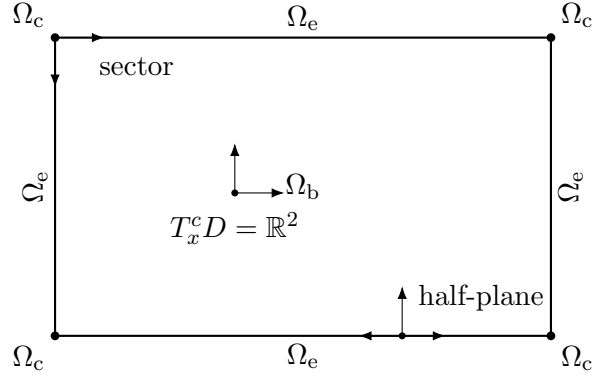

\subsection{Stratified representation bundles and measures}
\label{ssec:stratified-representations}

Because \(\Gamma_D^{\tc}\) is the disjoint union of its transitive orbit
restrictions, its representation category decomposes orbit by orbit.  For a
rectangle,
\begin{equation}
  \Rep(\Gamma_D^{\tc})
  \simeq
  \Rep(\OO(2))\times\Rep(\Z_2)\times\Rep(\Z_2).
  \label{eq:tc-representation-product}
\end{equation}
A global feature type is therefore specified by a triple
\begin{equation}
  (\rho,\varepsilon,\delta),
  \qquad
  \rho\in\Rep(\OO(2)),\quad
  \varepsilon,\delta\in\Rep(\Z_2),
  \label{eq:stratified-feature-type}
\end{equation}
with induced bundles \(E_{\mathrm b,\rho}\), \(E_{\mathrm e,\varepsilon}\),
and \(E_{\mathrm c,\delta}\) over the three strata.  The triple is not, in
general, an irreducible representation of the disconnected groupoid: an
irreducible is supported on a single orbit.  The triple is the natural datum
for a feature field defined over the whole stratified domain.

To retain boundary and corner features as independent components, we use the
stratified measure
\begin{equation}
  \nu=\nu_{\mathrm b}\oplus\nu_{\mathrm e}\oplus\nu_{\mathrm c},
  \label{eq:stratified-measure}
\end{equation}
where \(\nu_{\mathrm b}\) is area measure, \(\nu_{\mathrm e}\) is arclength,
and \(\nu_{\mathrm c}\) is counting measure.  The corresponding section
space is
\begin{equation}
  \mathcal H_{\rho,\varepsilon,\delta}
  =L^2(\Omega_{\mathrm b},E_{\mathrm b,\rho};\nu_{\mathrm b})
  \oplus
  L^2(\Omega_{\mathrm e},E_{\mathrm e,\varepsilon};\nu_{\mathrm e})
  \oplus
  \ell^2(\Omega_{\mathrm c},E_{\mathrm c,\delta}).
  \label{eq:stratified-section-space}
\end{equation}
The use of \(\nu\) is a modeling choice, not the restriction of planar
Lebesgue measure: with area measure alone the boundary and corner components
would have measure zero.

\subsection{Rigid bisections and the polygonal ampleness theorem}
\label{ssec:rigid-bisections-polygonal}

Let \(\mathcal B_D^{\mathrm{rig}}\) be the inverse semigroup generated by the
rigid local bisections
\begin{equation}
  b_{g,U}(x)=(g,x),
  \qquad
  U\subset D\text{ relatively open},
  \qquad
  (g,x)\in\Gamma_D^{\tc}\text{ for all }x\in U.
  \label{eq:tc-rigid-bisection}
\end{equation}
We allow \(U\) to be disconnected; this convention is useful when a bisection
is required to transport two separated points simultaneously.

\begin{lemma}[Rigid ampleness for polygonal domains]
\label{lem:ample-rigid}
If \(D\) is polygonal, \(\mathcal B_D^{\mathrm{rig}}\) is ample for
\(\Gamma_D^{\tc}\): every arrow of the tangent-cone groupoid lies on a rigid
local bisection.
\end{lemma}

\begin{proof}
Let \(((v,A),x):x\to y\) be an arrow.  If \(x\) is interior, choose a ball
\(U\Subset\Int(D)\) around \(x\) small enough that \(AU+v\Subset\Int(D)\).
If \(x\) lies in the relative interior of an edge, the cone condition forces
\(A\) to map the supporting line and inward half-plane at \(x\) to those at
\(y\).  Since the boundary is locally straight at both points, a sufficiently
small relative neighborhood of \(x\) is mapped into a relative neighborhood of
\(y\), with tangent cones preserved at every point.  At a corner the same
argument applies to the two incident rays and the sector between them.
\end{proof}

\begin{remark}[Curved boundaries and higher-order data]
\label{rem:curved-boundary-jets}
For a curved boundary, equality of tangent cones at two points is only a
first-order condition.  A rigid motion matching the tangent and inward normal
at one point need not map a boundary arc to a boundary arc; curvature already
provides a second-order obstruction.  Therefore
\(\mathcal B_D^{\mathrm{rig}}\) need not be ample for
\(\Gamma_D^{\tc}\).  There are two distinct extensions: one may enlarge the
selected bisection family to non-rigid local bisections of the tangent-cone
groupoid, or refine the arrows by requiring equality of higher boundary jets.
The present paper uses polygonal domains, for which
Lemma~\ref{lem:ample-rigid} is exact.
\end{remark}

\section{Boundary pairs of orbits and kernel blocks}
\label{sec:boundary-kernel-classification}
\label{ssec:stratified-filter-banks}

We now apply Proposition~\ref{prop:pair-orbits} to a rectangle.  This is the
step at which the tangent-cone groupoid produces layer types that do not occur
in homogeneous steerable CNNs.  The object orbits are the three strata, but a
nonlocal kernel is defined on ordered pairs of points.  Pair orbits may
therefore have their target and source in different strata, and the associated
cross-stratum blocks are constrained rather than forbidden.

\subsection{Setup and block decomposition}
\label{ssec:rectangle-kernel-setup}

Let
\(
  D=[0,L_1]\times[0,L_2]
\)
with the strata \(\Omega_{\mathrm b},\Omega_{\mathrm e},\Omega_{\mathrm c}\)
of \eqref{eq:rectangle-strata}.  Input and output feature types are
\begin{equation}
  (\rho_{\mathrm{in}},\varepsilon_{\mathrm{in}},\delta_{\mathrm{in}}),
  \qquad
  (\rho_{\mathrm{out}},\varepsilon_{\mathrm{out}},\delta_{\mathrm{out}}).
  \label{eq:in-out-stratified-types}
\end{equation}
We use adapted gauges: the ambient Cartesian frame in the bulk; a tangent--inward
normal frame \((t_y,n_y)\) on an edge; and a frame aligned with the two rays,
or equivalently the bisector, at a corner.  A different coherent gauge changes
the matrix representatives by fibrewise conjugation but not the orbit spaces
or parameter dimensions.

Fix a support radius
\begin{equation}
  0<r_0<\tfrac12\min\{L_1,L_2\},
  \qquad
  K(y,x)=0\quad\text{if }|y-x|>r_0 \, .
  \label{eq:small-filter-regime}
\end{equation}
The previous condition (\ref{eq:small-filter-regime}) will be called the \emph{small-filter regime}.
This excludes interactions between opposite edges and between distinct
corners; adjacent-edge interactions near a common corner remain present.  The
domains of rigid bisections are allowed to be disconnected, as in
Section~\ref{ssec:rigid-bisections-polygonal}.

Relative to the direct sums in \eqref{eq:stratified-section-space}, a kernel is
a \(3\times3\) block matrix
\begin{equation}
  K(y,x)=
  \begin{pmatrix}
    K^{\mathrm{bb}}&K^{\mathrm{be}}&K^{\mathrm{bc}}\\
    K^{\mathrm{eb}}&K^{\mathrm{ee}}&K^{\mathrm{ec}}\\
    K^{\mathrm{cb}}&K^{\mathrm{ce}}&K^{\mathrm{cc}}
  \end{pmatrix}(y,x),
  \label{eq:3x3-kernel-blocks}
\end{equation}
where the first superscript is the target stratum and the second is the source
stratum.  The corresponding channel is
\begin{equation}
  (\Phi\psi)^s(y)
  =\sum_{s'\in\{\mathrm b,\mathrm e,\mathrm c\}}
  \int_{\Omega_{s'}}K^{ss'}(y,x)\psi^{s'}(x)\,d\nu_{s'}(x),
  \qquad y\in\Omega_s.
  \label{eq:stratified-layer}
\end{equation}
Because every tangent-cone arrow preserves the object stratum, the nine blocks
transform independently under the diagonal pair action.  They need not vanish.

For the reference edge frame \((t,n)\), write
\begin{equation}
  \sigma_{\mathrm{ax}}=
  \begin{pmatrix}-1&0\\0&1\end{pmatrix},
  \label{eq:axis-reflection}
\end{equation}
the reflection fixing the inward normal.  For the reference right-angle corner
write
\begin{equation}
  \sigma_{\mathrm{diag}}=
  \begin{pmatrix}0&1\\1&0\end{pmatrix},
  \label{eq:diagonal-reflection}
\end{equation}
the bisector reflection.

\subsection{Diagonal blocks}
\label{ssec:diagonal-kernel-blocks}

\paragraph{Bulk to bulk.}
For two bulk points, rigid Euclidean motions act exactly as in the free plane.
Translations give
\(
  K^{\mathrm{bb}}(y,x)=k^{\mathrm{bb}}(y-x)
\),
and the orthogonal part gives
\begin{equation}
  k^{\mathrm{bb}}(A\xi)
  =\rho_{\mathrm{out}}(A)
   k^{\mathrm{bb}}(\xi)
   \rho_{\mathrm{in}}(A)^{-1},
  \qquad A\in\OO(2).
  \label{eq:bulk-constraint}
\end{equation}
For \(\xi\neq0\), the joint stabilizer is the reflection subgroup fixing the
line \(\R\xi\); at \(\xi=0\) it is the full \(\OO(2)\).  Hence the free data
may be described as one
\(\Hom_{\Z_2}(\rho_{\mathrm{in}},\rho_{\mathrm{out}})\)-valued radial
function for positive radius, together with an
\(\OO(2)\)-intertwiner at the center.  This is precisely the usual steerable
bulk kernel.

\paragraph{Edge to edge on one edge.}
Let \(u\) be the signed arclength of the source relative to the target in the
adapted edge frame.  Translation along an edge removes the absolute position,
while reflection in the inward normal line sends \(u\) to \(-u\).  Thus
\begin{equation}
  K^{\mathrm{ee}}(y,x)=k^{\mathrm{ee}}(u),
  \qquad
  k^{\mathrm{ee}}(-u)
  =\varepsilon_{\mathrm{out}}(\sigma)
   k^{\mathrm{ee}}(u)
   \varepsilon_{\mathrm{in}}(\sigma)^{-1}.
  \label{eq:edge-edge-constraint}
\end{equation}
For \(u\neq0\) the ordered pair has trivial joint stabilizer.  At \(u=0\)
the stabilizer is the full edge \(\Z_2\), so the center value lies in
\(
  \Hom_{\Z_2}(\varepsilon_{\mathrm{in}},\varepsilon_{\mathrm{out}})
\).
In an irreducible parity basis, equal-parity components are even in \(u\) and
opposite-parity components are odd.

\paragraph{Corner to corner.}
Condition \eqref{eq:small-filter-regime} excludes distinct corners.  The only
corner--corner pairs are \((p,p)\), and the four diagonal pairs form one pair
orbit.  Their common kernel value satisfies
\begin{equation}
  k^{\mathrm{cc}}
  \in\Hom_{\Z_2}(\delta_{\mathrm{in}},\delta_{\mathrm{out}}).
  \label{eq:corner-corner}
\end{equation}
If
\(
  \delta_{\mathrm{in}}=m_+^{\mathrm{in}}\mathbf1
  \oplus m_-^{\mathrm{in}}\sgn
\)
and similarly for the output, then
\begin{equation}
  \dim\Hom_{\Z_2}(\delta_{\mathrm{in}},\delta_{\mathrm{out}})
  =m_+^{\mathrm{in}}m_+^{\mathrm{out}}
   +m_-^{\mathrm{in}}m_-^{\mathrm{out}}.
  \label{eq:corner-corner-dimension}
\end{equation}

\subsection{Bulk--edge and bulk--corner blocks in the small-filter regime}
\label{ssec:mixed-bulk-boundary-blocks}

\paragraph{Edge target, bulk source.}
Let \(y\in\Omega_{\mathrm e}\), \(x\in\Omega_{\mathrm b}\), and write
\begin{equation}
  x-y=\xi_t t_y+\xi_n n_y,
  \qquad \xi_n>0.
  \label{eq:edge-bulk-coordinates}
\end{equation}
The pair orbit is represented by \((\xi_t,\xi_n)\) modulo
\(\xi_t\mapsto-\xi_t\), and
\begin{equation}
  k^{\mathrm{eb}}(-\xi_t,\xi_n)
  =\varepsilon_{\mathrm{out}}(\sigma)
   k^{\mathrm{eb}}(\xi_t,\xi_n)
   \rho_{\mathrm{in}}(\sigma_{\mathrm{ax}})^{-1}.
  \label{eq:edge-bulk-constraint}
\end{equation}
Generic pairs \(\xi_t\neq0\) have trivial stabilizer.  On the normal axis
\(\xi_t=0\), the reflection fixes both points and the kernel value is an
intertwiner for the edge reflection subgroup.

\paragraph{Bulk target, edge source.}
The reverse block has the same orbit space and satisfies
\begin{equation}
  k^{\mathrm{be}}(-\xi_t,\xi_n)
  =\rho_{\mathrm{out}}(\sigma_{\mathrm{ax}})
   k^{\mathrm{be}}(\xi_t,\xi_n)
   \varepsilon_{\mathrm{in}}(\sigma)^{-1}.
  \label{eq:bulk-edge-constraint}
\end{equation}
These two blocks are, respectively, an equivariant boundary readout from the
bulk and an equivariant injection of edge data into nearby bulk features.

\paragraph{Corner target, bulk source.}
Let \(p\in\Omega_{\mathrm c}\) and express \(x-p=\xi\) in the reference
quadrant frame at \(p\).  The bisector reflection identifies \(\xi\) with
\(\sigma_{\mathrm{diag}}\xi\), and
\begin{equation}
  k^{\mathrm{cb}}(\sigma_{\mathrm{diag}}\xi)
  =\delta_{\mathrm{out}}(\sigma)
   k^{\mathrm{cb}}(\xi)
   \rho_{\mathrm{in}}(\sigma_{\mathrm{diag}})^{-1}.
  \label{eq:corner-bulk-constraint}
\end{equation}
The stabilizer is trivial away from the bisector and is \(\Z_2\) on the
bisector.

\paragraph{Bulk target, corner source.}
The reverse block satisfies
\begin{equation}
  k^{\mathrm{bc}}(\sigma_{\mathrm{diag}}\xi)
  =\rho_{\mathrm{out}}(\sigma_{\mathrm{diag}})
   k^{\mathrm{bc}}(\xi)
   \delta_{\mathrm{in}}(\sigma)^{-1}.
  \label{eq:bulk-corner-constraint}
\end{equation}

\subsection{Edge--corner and corner-adjacent edge blocks in the small-filter regime}
\label{ssec:corner-adjacent-blocks}

Let \(p\) be a corner and \(e_1,e_2\) its two incident open edges.  Choose
coherent edge gauges, and let \(u>0\) denote arclength away from the corner.
The bisector reflection exchanges \(e_1\) and \(e_2\).  In a gauge symmetric
under that reflection, the corner--edge and edge--corner relations are
\begin{equation}
  k^{\mathrm{ce}}_{e_2}(u)
  =\delta_{\mathrm{out}}(\sigma)k^{\mathrm{ce}}_{e_1}(u),
  \qquad
  k^{\mathrm{ec}}_{e_2}(u)
  =k^{\mathrm{ec}}_{e_1}(u)\delta_{\mathrm{in}}(\sigma)^{-1}.
  \label{eq:corner-edge-constraint}
\end{equation}
In a different coherent gauge, the right-hand sides acquire the corresponding
edge-fibre cocycles.  This does not alter the number of free parameters: one
unconstrained matrix-valued function of \(u\) determines each of the two
blocks.

The edge--edge block has an additional family when target and source lie on
the two different edges adjacent to the same corner.  Such an ordered pair is
represented by \((u_y,u_x)\), with the support restriction
\(
  \sqrt{u_y^2+u_x^2}\leq r_0
\).
The bisector reflection exchanges the two ordered edge configurations.  The
generic joint stabilizer is trivial, so one arbitrary matrix-valued function
of \((u_y,u_x)\) determines the second configuration by transport.  We denote
this contribution by \(K^{\mathrm{ee}}_{\mathrm{corner}}\); together with the
same-edge contribution it forms the single matrix block
\(K^{\mathrm{ee}}\) in \eqref{eq:3x3-kernel-blocks}.

\subsection{Small-filter regime normal form}
\label{ssec:rectangle-normal-form}

\begin{theorem}[Boundary-aware kernel classification on a rectangle]
\label{thm:rectangle-kernel-classification}
Under the assumptions of \eqref{eq:small-filter-regime}, i.e., in the small-filter regime, an integral kernel is
\(\mathcal B_D^{\mathrm{rig}}\)-equivariant if and only if its nine blocks
have the forms and transport relations described in
Sections~\ref{ssec:diagonal-kernel-blocks}--\ref{ssec:corner-adjacent-blocks}.
Equivalently, the free data consist of one joint-stabilizer intertwiner for
each rigid pair orbit listed in Table~\ref{tab:blocks}.  In particular, all six
cross-stratum matrix blocks are permitted by equivariance.
\end{theorem}

\begin{proof}
By Lemma~\ref{lem:ample-rigid}, the selected rigid pseudogroup is ample.  The
bisection-equivariant kernel theorem reduces each block to the diagonal action
of that pseudogroup on its ordered source--target pair space.  Rigid Euclidean
motions identify precisely the representatives described above: distance and
orthogonal direction in the bulk, signed arclength on one edge, normal or
sector coordinates for mixed pairs, and the two configurations exchanged by a
corner bisector.  The fixed loci of the corresponding reflections give the
exceptional stabilizers.  Proposition~\ref{prop:pair-orbits} then proves both
necessity and sufficiency.
\end{proof}

\begin{table}[ht]
\centering
\small
\begin{tabular}{p{0.16\textwidth}p{0.28\textwidth}p{0.22\textwidth}p{0.24\textwidth}}
\hline
Block & Pair-orbit representative & Stabilizer & Free datum \\
\hline
$\mathrm b\leftarrow\mathrm b$
& $r=|y-x|$; angular transport by $\OO(2)$
& $\OO(2)$ at $r=0$; $\Z_2$ for $r>0$
& center intertwiner plus radial $\Z_2$-intertwiner function \\
$\mathrm e\leftarrow\mathrm e$ (same edge)
& signed separation $u$ modulo $u\sim-u$
& $\Z_2$ at $u=0$; trivial otherwise
& one matrix function for $u>0$, mirror determined \\
$\mathrm c\leftarrow\mathrm c$
& one diagonal corner pair orbit
& corner $\Z_2$
& one $\Z_2$-intertwiner \\
$\mathrm e\leftarrow\mathrm b$, $\mathrm b\leftarrow\mathrm e$
& inward half-disk modulo tangential reflection
& $\Z_2$ on normal axis; trivial otherwise
& matrix function on a half fundamental domain \\
$\mathrm c\leftarrow\mathrm b$, $\mathrm b\leftarrow\mathrm c$
& corner sector modulo bisector reflection
& $\Z_2$ on bisector; trivial otherwise
& matrix function on half-sector \\
$\mathrm c\leftarrow\mathrm e$, $\mathrm e\leftarrow\mathrm c$
& $u>0$ on either incident edge, edges exchanged
& trivial
& one matrix function of $u$ \\
$\mathrm e\leftarrow\mathrm e$ (across corner)
& $(u_y,u_x)$, two ordered edge configurations exchanged
& trivial
& one matrix function of $(u_y,u_x)$ \\
\hline
\end{tabular}
\caption{Rigid pair-orbit reduction of a boundary-aware kernel on a rectangle.
The nine matrix blocks are obtained by grouping the two edge--edge orbit
families into the single block $K^{\mathrm{ee}}$.}
\label{tab:blocks}
\end{table}

\begin{remark}[Why cross-stratum blocks do not contradict orbit separation]
\label{rem:coupling}
The groupoid has no arrows from the bulk orbit to the edge or corner orbits, so
pointwise equivariant maps are block diagonal by stratum.  A nonlocal kernel,
however, is defined on ordered pairs.  A local bisection acts simultaneously
on a bulk point and an edge point without changing either point's object orbit.
The mixed pair therefore has a legitimate orbit and joint stabilizer.  This is
why nonlocal equivariance constrains cross-stratum couplings instead of
eliminating them.
\end{remark}

\begin{remark}[Beyond the small-filter regime]
\label{rem:large-filters}
For larger support, additional pair families appear: opposite edges, distinct
corners, and pairs whose transport is obstructed by the finite geometry of the
domain.  The pair-orbit theorem remains valid, but the orbit invariants and
stabilizers must be refined.  The small-filter regime isolates the local
boundary geometry that is most directly relevant to convolutional layers.
\end{remark}

\section{Discrete tangent-cone groupoids and parameter counts}
\label{sec:discrete-tangent-cone}
\label{ssec:discrete}

The pixel-grid model is an exact finite realization of the preceding symmetry
datum.  It should not be confused with an exact preservation of the continuous
\(\OO(2)\) symmetry: the point group is reduced to \(D_4\), and the continuous
pair-orbit variables are sampled on a lattice.  What is exact is the finite
groupoid equivariance constraint itself, which becomes a finite homogeneous
linear system.

\subsection{The finite tangent-cone groupoid}
\label{ssec:finite-tc-groupoid}

Let
\begin{equation}
  \Omega_{W,H}=\{0,\ldots,W-1\}\times\{0,\ldots,H-1\}\subset\Z^2,
  \qquad W,H\geq4,
  \label{eq:pixel-grid}
\end{equation}
and let the ambient wallpaper group be
\begin{equation}
  \mathrm{p4m}=\Z^2\rtimes D_4.
  \label{eq:p4m}
\end{equation}
We use
\begin{equation}
  r=\begin{pmatrix}0&-1\\1&0\end{pmatrix},
  \qquad
  \sigma_{\mathrm{ax}}=\begin{pmatrix}-1&0\\0&1\end{pmatrix},
  \qquad
  \sigma_{\mathrm{diag}}=\begin{pmatrix}0&1\\1&0\end{pmatrix}
  \label{eq:D4-generators}
\end{equation}
as a quarter-turn, an axis reflection, and a diagonal reflection.
At each pixel we use the same full-plane, half-plane, or quadrant cone as in
the continuum rectangle.  Define
\begin{equation}
  \Gamma_{W,H}^{\tc}
  =\{((t,A),x):
  x,(Ax+t)\in\Omega_{W,H},\ 
  A(T_x^c\Omega_{W,H})=T_{Ax+t}^c\Omega_{W,H}\}
  \rightrightarrows\Omega_{W,H}.
  \label{eq:discrete-tc-groupoid}
\end{equation}

\begin{proposition}[Discrete orbit and isotropy structure]
\label{prop:discrete-orbits}
The groupoid \(\Gamma_{W,H}^{\tc}\) has exactly three orbits: interior pixels,
boundary non-corner pixels, and the four corners.  Their isotropy groups are
\begin{equation}
  D_4,
  \qquad
  \Z_2^{\mathrm{ax}}=\langle\sigma_{\mathrm{ax}}\rangle,
  \qquad
  \Z_2^{\mathrm{diag}}=\langle\sigma_{\mathrm{diag}}\rangle,
  \label{eq:discrete-isotropy}
\end{equation}
respectively.  The two reflection subgroups are non-conjugate in \(D_4\).
The four edges form one groupoid orbit even when \(W\neq H\).
\end{proposition}

\begin{proof}
Translations act transitively on the interior.  Along the boundary,
translations parallel to an edge connect pixels on the same edge, and a
quarter-turn followed by a translation connects a horizontal edge pixel to a
vertical edge pixel while carrying its half-plane cone to the corresponding
half-plane.  Quarter-turns and translations likewise connect all four corners.
The stabilizer of the full lattice cone is \(D_4\); the stabilizer of the
reference half-plane is generated by \(\sigma_{\mathrm{ax}}\); and the
stabilizer of the reference quadrant is generated by
\(\sigma_{\mathrm{diag}}\).  Axis and diagonal reflections form the two
distinct reflection conjugacy classes of \(D_4\).
\end{proof}

\subsection{Branching rules and gauge cocycles}
\label{ssec:D4-branching-gauges}

Let \(A_1,A_2,B_1,B_2,E\) denote the standard irreducible real
representations of \(D_4\), with \(A_1\) trivial, \(A_2=\det\), and \(E\) the
standard two-dimensional representation.  With the conventions of
\eqref{eq:D4-generators}, their restrictions are shown in
Table~\ref{tab:branching}.

\begin{table}[ht]
\centering
\small
\begin{tabular}{lccccc}
\hline
 & $A_1$ & $A_2$ & $B_1$ & $B_2$ & $E$ \\
\hline
$\Z_2^{\mathrm{ax}}$
& $\mathbf1$ & $\sgn$ & $\mathbf1$ & $\sgn$
& $\mathbf1\oplus\sgn$ \\
$\Z_2^{\mathrm{diag}}$
& $\mathbf1$ & $\sgn$ & $\sgn$ & $\mathbf1$
& $\mathbf1\oplus\sgn$ \\
\hline
\end{tabular}
\caption{Branching of $D_4$ irreducibles to the edge and corner isotropy
subgroups.  The exchange of $B_1$ and $B_2$ distinguishes edge from corner
couplings on the lattice.}
\label{tab:branching}
\end{table}

A discrete feature type is a triple
\begin{equation}
  (\rho,\varepsilon,\delta),
  \qquad
  \rho\in\Rep(D_4),\quad
  \varepsilon\in\Rep(\Z_2^{\mathrm{ax}}),\quad
  \delta\in\Rep(\Z_2^{\mathrm{diag}}).
  \label{eq:discrete-feature-type}
\end{equation}
To write matrices in fixed reference fibres, choose for every edge and corner
pixel a gauge \(A_x\in D_4\) carrying the reference half-plane or quadrant to
the cone at \(x\).  For an arrow \(((t,A),x):x\to y\), the fibre action on a
boundary stratum is represented by the isotropy cocycle
\begin{equation}
  \eta((t,A),x)=A_y^{-1}AA_x.
  \label{eq:gauge-cocycle}
\end{equation}
For an edge, \(\eta\in\Z_2^{\mathrm{ax}}\); for a corner,
\(\eta\in\Z_2^{\mathrm{diag}}\).  Changing the gauges conjugates all fibre and
kernel matrices coherently and therefore leaves the equivariant layer space and
its dimension unchanged.

\subsection{Finite pair orbits and completeness}
\label{ssec:finite-pair-orbits}

Fix a sup-norm radius \(r_0\) satisfying
\begin{equation}
  2r_0+2\leq\min\{W,H\}.
  \label{eq:discrete-small-filter}
\end{equation}
Let \(\mathcal P_{r_0}\) be the finite set of ordered pixel pairs
\((y,x)\) with \(\|y-x\|_\infty\leq r_0\), together with their source and
target strata.  The rigid pseudogroup acts on \(\mathcal P_{r_0}\), and every
kernel tap belongs to one finite pair orbit.

\begin{proposition}[Finite completeness]
\label{prop:finite-completeness}
For fixed input and output feature types, the vector space of radius-\(r_0\)
\(\mathcal B^{\mathrm{rig}}\)-equivariant linear layers is canonically
isomorphic to
\begin{equation}
  \bigoplus_{[p]\in\mathcal P_{r_0}/\mathcal B^{\mathrm{rig}}}
  \Hom_{S_p}(V_p^{\mathrm{in}},V_p^{\mathrm{out}}),
  \label{eq:finite-orbit-direct-sum}
\end{equation}
where \(p\) is one representative of each ordered-pair orbit and \(S_p\) its
joint stabilizer.  Equivalently, solving the finite kernel transport equations
produces a basis of the entire equivariant layer space, not merely a
subfamily.
\end{proposition}

\begin{proof}
This is Proposition~\ref{prop:pair-orbits} applied to the finite pair set.
There are no measurability issues, and transport from an orbit representative
is a finite collection of matrix identities.
\end{proof}

\subsection{Closed parameter-count formulas}
\label{ssec:discrete-parameter-counts}

For a finite subgroup \(S\), write
\begin{equation}
  h_S(\pi_{\mathrm{in}},\pi_{\mathrm{out}})
  :=\dim\Hom_S(\pi_{\mathrm{in}},\pi_{\mathrm{out}}).
  \label{eq:hS-definition}
\end{equation}
When one argument is a bulk \(D_4\)-representation and the other an edge or
corner representation, the bulk representation is first restricted to the
indicated reflection subgroup.  Let
\(
  d_\rho, d_\varepsilon,d_\delta
\)
denote fibre dimensions, with input and output superscripts when needed, and
set
\(
  q(r_0)=r_0(r_0-1)/2
\).

\begin{theorem}[Parameter counts for the ten geometric block families]
\label{thm:discrete-parameter-counts}
The dimensions of the radius-\(r_0\) kernel spaces are
\begin{subequations}
\label{eq:discrete-counts}
\begin{align}
N_{\mathrm b\leftarrow\mathrm b}
 &=h_{D_4}(\rho_{\mathrm{in}},\rho_{\mathrm{out}})
 +r_0\Bigl[
   h_{\Z_2^{\mathrm{ax}}}(\rho_{\mathrm{in}},\rho_{\mathrm{out}})
  +h_{\Z_2^{\mathrm{diag}}}(\rho_{\mathrm{in}},\rho_{\mathrm{out}})
 \Bigr]
 +q(r_0)d_{\rho_{\mathrm{in}}}d_{\rho_{\mathrm{out}}},
 \label{eq:count-bb}\\
N_{\mathrm e\leftarrow\mathrm e}
 &=h_{\Z_2}(\varepsilon_{\mathrm{in}},\varepsilon_{\mathrm{out}})
 +r_0d_{\varepsilon_{\mathrm{in}}}d_{\varepsilon_{\mathrm{out}}},
 \label{eq:count-ee}\\
N_{\mathrm c\leftarrow\mathrm c}
 &=h_{\Z_2}(\delta_{\mathrm{in}},\delta_{\mathrm{out}}),
 \label{eq:count-cc}\\
N_{\mathrm e\leftarrow\mathrm b}
 &=r_0h_{\Z_2^{\mathrm{ax}}}
   (\rho_{\mathrm{in}},\varepsilon_{\mathrm{out}})
 +r_0^2d_{\rho_{\mathrm{in}}}d_{\varepsilon_{\mathrm{out}}},
 \label{eq:count-eb}\\
N_{\mathrm b\leftarrow\mathrm e}
 &=r_0h_{\Z_2^{\mathrm{ax}}}
   (\varepsilon_{\mathrm{in}},\rho_{\mathrm{out}})
 +r_0^2d_{\varepsilon_{\mathrm{in}}}d_{\rho_{\mathrm{out}}},
 \label{eq:count-be}\\
N_{\mathrm c\leftarrow\mathrm b}
 &=r_0h_{\Z_2^{\mathrm{diag}}}
   (\rho_{\mathrm{in}},\delta_{\mathrm{out}})
 +q(r_0)d_{\rho_{\mathrm{in}}}d_{\delta_{\mathrm{out}}},
 \label{eq:count-cb}\\
N_{\mathrm b\leftarrow\mathrm c}
 &=r_0h_{\Z_2^{\mathrm{diag}}}
   (\delta_{\mathrm{in}},\rho_{\mathrm{out}})
 +q(r_0)d_{\delta_{\mathrm{in}}}d_{\rho_{\mathrm{out}}},
 \label{eq:count-bc}\\
N_{\mathrm c\leftarrow\mathrm e}
 &=r_0d_{\varepsilon_{\mathrm{in}}}d_{\delta_{\mathrm{out}}},
 \label{eq:count-ce}\\
N_{\mathrm e\leftarrow\mathrm c}
 &=r_0d_{\delta_{\mathrm{in}}}d_{\varepsilon_{\mathrm{out}}},
 \label{eq:count-ec}\\
N_{\mathrm e\leftarrow\mathrm e}^{\mathrm{corner}}
 &=r_0^2d_{\varepsilon_{\mathrm{in}}}d_{\varepsilon_{\mathrm{out}}}.
 \label{eq:count-ee-corner}
\end{align}
\end{subequations}
Their sum is the dimension of the complete equivariant layer space.
\end{theorem}

\begin{proof}
For the bulk block, the center tap has stabilizer \(D_4\); each of the
\(r_0\) nonzero axis radii and \(r_0\) diagonal radii has the corresponding
reflection stabilizer; the remaining generic \(D_4\)-orbits are represented
by \(0<b<a\leq r_0\), of which there are \(q(r_0)\), and have trivial
stabilizer.  The same-edge block has one central \(\Z_2\)-intertwiner and one
full matrix for each positive displacement.  The edge--bulk blocks have
\(r_0\) normal-axis orbits with reflection stabilizer and \(r_0^2\) generic
mirror-pair orbits.  The corner--bulk blocks have \(r_0\) bisector orbits and
\(q(r_0)\) generic off-bisector orbits.  The remaining corner--edge and
across-corner edge families have trivial stabilizer and the indicated numbers
of representatives.  Proposition~\ref{prop:finite-completeness} completes the
count.
\end{proof}

\begin{example}[A parameter-count check]
\label{ex:225-parameter-layer}
Take
\begin{align*}
\rho_{\mathrm{in}}&=A_1\oplus E,
&\rho_{\mathrm{out}}&=A_1\oplus B_1\oplus E,\\
\varepsilon_{\mathrm{in}}&=\mathbf1\oplus\sgn,
&\varepsilon_{\mathrm{out}}&=2\mathbf1\oplus\sgn,\\
\delta_{\mathrm{in}}&=\mathbf1\oplus\sgn,
&\delta_{\mathrm{out}}&=\mathbf1\oplus2\sgn,
\end{align*}
and \(r_0=2\).  The ten dimensions are
\begin{equation}
  (40,15,3,46,40,17,16,12,12,24),
  \label{eq:worked-counts}
\end{equation}
for a total of \(225\) parameters.  The accompanying finite constraint solver\footnote{GitHub url: "https://github.com/mariajimvaz-afk/groupoid-equivariant-cnn".} \verb+groupoid_cnn.py+
returns nullspaces of exactly these dimensions.
\end{example}

\begin{remark}[Exact finite constraints versus continuum discretization]
\label{rem:finite-exactness}
Once the finite groupoid, gauges, feature types, and support set are fixed, the
transport equations are exact algebraic identities.  Numerical nullspace
computation may introduce floating-point error, but there is no approximation
in the definition of equivariance.  This should be distinguished from the
approximation incurred by replacing \(\OO(2)\) with \(D_4\), replacing the
continuous domain with a grid, and sampling continuous pair-orbit variables.
\end{remark}

Part~\ref{part:euclidean-groupoids} has classified individual finite-propagation
layers.  The next part studies what remains of the same partial equivariance
when such layers are composed.  Global bisections compose exactly, whereas
properly local bisections require an ``erosion'' of their admissible domains.

\part{Filtered equivariance under composition}
\label{part:filtered-equivariance}

The kernel theorem of Part~\ref{part:symmetry-kernels} classifies a single
layer exactly.  A new issue appears when such layers are composed.  For a
global bisection, the usual intertwining relation is stable under composition.
For a proper local bisection, however, an input may influence an output through
intermediate points that do not belong to the bisection domain.  Exact
partial equivariance is therefore retained only after removing a boundary
layer from that domain.

The purpose of this part is to make this statement precise.  We first work
with finite-propagation integral channels and prove a two-layer composition
theorem.  We then establish an $n$-layer version in which the required erosion
is determined by the geometry of paths through the receptive field.  Finally,
we formulate the resulting filtered notion of equivariance and show how
pointwise nonlinearities, biases, residual branches, and global symmetries fit
into it.

\section{Finite propagation, erosion, and composition}
\label{sec:propagation-erosion}

\subsection{Metric symmetry data and finite propagation}
\label{ssec:metric-propagation}

Throughout this part, the object space $\Omega$ is equipped with a metric $d$
compatible with its topology and measurable structure.  We assume that every
$b\in\mathcal B$ acts by a local isometry\footnote{This is not necessarily so when dealing with non-flat Riemannian manifolds.  In such case the theory should be refined and the relevant groupoid is the so called Riemann (or Poincar\'e) groupoid - see, for instance, \cite{ibort-marmo-mas-dorca-schiavone}.},
\begin{equation}
  d\bigl(\tau_b(x),\tau_b(y)\bigr)=d(x,y),
  \qquad x,y\in U_b.
  \label{eq:local-isometry-bisection}
\end{equation}
This is automatic for the rigid Euclidean bisections used in
Part~\ref{part:euclidean-groupoids}.  On the pixel grid we use the
$\ell^\infty$ metric, so that a kernel of radius $r$ is supported on a
$(2r+1)\times(2r+1)$ patch.

For measurable subsets $A,C\subseteq\Omega$, write
\[
  d(A,C):=\inf\{d(a,c):a\in A,\ c\in C\},
\]
with the convention $d(A,\varnothing)=+\infty$.

\begin{definition}[Finite propagation]
\label{def:finite-propagation}
Let
\[
  \Phi\colon L^2(\Omega,E^{\mathrm{in}};\nu)
  \longrightarrow L^2(\Omega,E^{\mathrm{out}};\nu)
\]
be a bounded linear operator.  We say that $\Phi$ has propagation at most $r\geq0$ if
\begin{equation}
  P_A^{E^{\mathrm{out}}}\,\Phi\,P_C^{E^{\mathrm{in}}}=0
  \qquad\text{whenever}\qquad d(A,C)>r.
  \label{eq:finite-propagation-projections}
\end{equation}
The infimum of such $r$ is denoted by $\operatorname{prop}(\Phi)$.
\end{definition}

For an integral channel $\Phi_K$, condition
\eqref{eq:finite-propagation-projections} is equivalent, up to null sets, to
\begin{equation}
  K(y,x)=0
  \qquad\text{whenever}\qquad d(y,x)>r.
  \label{eq:finite-propagation-kernel}
\end{equation}
Pointwise channels have propagation zero.  The following standard estimate
will be useful later.

\begin{proposition}[Propagation under composition]
\label{prop:propagation-composition}
If $\Phi_1$ and $\Phi_2$ have propagation at most $r_1$ and $r_2$,
respectively, then
\begin{equation}
  \operatorname{prop}(\Phi_2\Phi_1)\leq r_1+r_2.
  \label{eq:propagation-composition}
\end{equation}
More generally, an $n$-fold composite
$\Phi_n\cdots\Phi_1$ has propagation at most $\sum_{j=1}^n r_j$.
\end{proposition}

\begin{proof}
For integral kernels, a nonzero contribution to the composite kernel
\[
  K_{21}(z,x)=\int_\Omega K_2(z,y)K_1(y,x)\,d\nu(y)
\]
requires $d(z,y)\leq r_2$ and $d(y,x)\leq r_1$.  Hence
$d(z,x)\leq r_1+r_2$.  The projection formulation follows by the same
neighbourhood argument, and the $n$-layer statement follows inductively.
\end{proof}

Propagation controls how far information may travel.  It does not itself
imply equivariance; throughout the composition results below, every factor is
also assumed to satisfy Definition~\ref{def:B-equivariant} on the maximal
domains of the bisections in $\mathcal B$.

\subsection{Eroded bisection domains}
\label{ssec:eroded-domains}

Let $b\in\mathcal B$ have source domain $U_b$ and image domain
$U_b'=\tau_b(U_b)$.  For $x\in\Omega$ and a subset $F\subseteq\Omega$, let
$d(x,F)=\inf_{z\in F}d(x,z)$, with $d(x,\varnothing)=+\infty$.

\begin{definition}[Symmetric erosion of a bisection domain]
\label{def:bisection-erosion}
For $r\geq0$, define
\begin{equation}
  U_b^{\ominus r}
  :=\left\{
  x\in U_b:
  \begin{array}{l}
  d\bigl(x,\Omega\setminus U_b\bigr)>r,\\[-1pt]
  d\bigl(\tau_b(x),\Omega\setminus U_b'\bigr)>r
  \end{array}
  \right\}.
  \label{eq:eroded-domain}
\end{equation}
The corresponding restricted bisection is
\begin{equation}
  b^{\ominus r}:=b|_{U_b^{\ominus r}},
  \qquad
  (U_b')^{\ominus_b r}:=\tau_b(U_b^{\ominus r}).
  \label{eq:eroded-bisection}
\end{equation}
\end{definition}

The second condition in \eqref{eq:eroded-domain} is essential: the source and
target sides of the partial symmetry must both contain the neighbourhoods
through which information propagates.  Because $\mathcal B$ is closed under
restriction, $b^{\ominus r}$ is again an admissible local bisection whenever
its domain is nonempty.

\begin{lemma}[Elementary properties of erosion]
\label{lem:erosion-properties}
For every $b\in\mathcal B$ and $0\leq r\leq s$:
\begin{enumerate}[label=\textup{(\roman*)}]
\item $U_b^{\ominus0}=U_b$ and
      $U_b^{\ominus s}\subseteq U_b^{\ominus r}$;
\item if $x\in U_b^{\ominus r}$ and $d(x,y)\leq r$, then $y\in U_b$;
\item if $x\in U_b^{\ominus r}$, $y\in U_b$, and $d(x,y)\leq r$, then
      $\tau_b(y)\in U_b'$ and
      $d(\tau_b(x),\tau_b(y))=d(x,y)$;
\item if $b$ is global, then $U_b^{\ominus r}=\Omega$ for every $r\geq0$.
\end{enumerate}
\end{lemma}

\begin{proof}
The first statement follows directly from the definition and openness of the
bisection domains.  If $y\notin U_b$, then
$d(x,\Omega\setminus U_b)\leq d(x,y)\leq r$, contradicting
$x\in U_b^{\ominus r}$; this proves (ii).  Statement (iii) follows from (ii),
the definition of a bisection, and the local-isometry hypothesis
\eqref{eq:local-isometry-bisection}.  If $b$ is global, both complements in
\eqref{eq:eroded-domain} are empty, proving (iv).
\end{proof}

In the finite grid, Definition~\ref{def:bisection-erosion} is equivalent to
requiring the complete closed $r$-neighbourhood of $x$ to lie in $U_b$ and the
complete closed $r$-neighbourhood of $\tau_b(x)$ to lie in $U_b'$.  This is
the erosion used by the discrete certificates of Part~\ref{part:experiments}.

\subsection{The two-layer composition theorem}
\label{ssec:two-layer-composition}

We now prove the basic structural result.  The erosion radius is the smaller
of the two propagation radii.  Indeed, every intermediate point contributing
to the composite lies simultaneously within distance $r_1$ of the source and
within distance $r_2$ of the target; it is therefore enough that either one of
these two neighbourhoods remain inside the bisection domain.

\begin{theorem}[Two-layer composition and erosion]
\label{thm:two-layer-erosion}
\label{prop:composition-erosion} 
Let
\[
  \Phi_1\colon \mathcal H_0\longrightarrow\mathcal H_1,
  \qquad
  \Phi_2\colon \mathcal H_1\longrightarrow\mathcal H_2
\]
be $\mathcal B$-equivariant integral channels with propagation at most
$r_1$ and $r_2$.  Assume that the composite has the iterated kernel obtained
by Fubini's theorem.  Set
\begin{equation}
  r_{12}:=\min\{r_1,r_2\}.
  \label{eq:two-layer-erosion-radius}
\end{equation}
Then, for every $b\in\mathcal B$, the composite $\Phi_2\Phi_1$ is equivariant
under the restricted bisection $b^{\ominus r_{12}}$.  Equivalently, if
$V=U_b^{\ominus r_{12}}$ and $V'=\tau_b(V)$, then
\begin{equation}
  P_{V'}^{E_2}\,\Phi_2\Phi_1\,
  \Lambda_0(b^{\ominus r_{12}})
  =
  \Lambda_2(b^{\ominus r_{12}})
  P_V^{E_2}\,\Phi_2\Phi_1\,P_V^{E_0}.
  \label{eq:two-layer-filtered-equivariance}
\end{equation}
\end{theorem}

\begin{proof}
Let $K_1$ and $K_2$ be the kernels of the two channels.  The composite kernel
is
\begin{equation}
  K_{21}(z,x)
  =\int_\Omega K_2(z,y)K_1(y,x)\,d\nu(y).
  \label{eq:two-layer-composite-kernel}
\end{equation}
Fix $x,z\in V$.  A nonzero integrand in
\eqref{eq:two-layer-composite-kernel} satisfies
\[
  d(y,x)\leq r_1,
  \qquad
  d(z,y)\leq r_2.
\]
If $r_1\leq r_2$, the first inequality and
$x\in U_b^{\ominus r_1}$ imply $y\in U_b$.  If $r_2\leq r_1$, the second
inequality and $z\in U_b^{\ominus r_2}$ imply the same conclusion.  Thus, in
either case, the integral may be restricted to $U_b$.

The corresponding integral defining
$K_{21}(\tau_b(z),\tau_b(x))$ may likewise be restricted to $U_b'$.
Changing variables $y'=\tau_b(y)$ and using preservation of $\nu$ gives
\begin{align*}
 K_{21}\bigl(\tau_b(z),\tau_b(x)\bigr)
 & =\int_{U_b}
 K_2\bigl(\tau_b(z),\tau_b(y)\bigr)
 K_1\bigl(\tau_b(y),\tau_b(x)\bigr)\,d\nu(y).
\end{align*}
Apply the kernel constraint \eqref{eq:kernel_constraint} to both factors.  The
representation matrices on the intermediate fibre cancel:
\begin{align*}
 &K_2\bigl(\tau_b(z),\tau_b(y)\bigr)
 K_1\bigl(\tau_b(y),\tau_b(x)\bigr)\\
 &\quad=
 R_2(b(z))K_2(z,y)
 \underbrace{R_1(b(y))^{-1}R_1(b(y))}_{I}
 K_1(y,x)R_0(b(x))^{-1}.
\end{align*}
Consequently,
\[
 K_{21}\bigl(\tau_b(z),\tau_b(x)\bigr)
 =R_2(b(z))K_{21}(z,x)R_0(b(x))^{-1}
\]

a.e. on $V\times V$.  The bisection-equivariant kernel theorem applied to
$b^{\ominus r_{12}}$ yields
\eqref{eq:two-layer-filtered-equivariance}.
\end{proof}

\begin{remark}[A safe but non-sharp radius]
The larger choice $\max\{r_1,r_2\}$ is also sufficient, but it erodes more of
the bisection domain than the proof requires.  The radius
\eqref{eq:two-layer-erosion-radius} is the natural balanced bound: every
intermediate point is controlled from both endpoints, and the shorter of the
two legs is enough to keep it inside the admissible region.
\end{remark}

\begin{corollary}[Global bisections]
\label{cor:global-composition}
If $b$ is global, then $\Phi_2\Phi_1$ is exactly $b$-equivariant on all of
$\Omega$, with no erosion.  In particular, the class of operators equivariant
under a global symmetry group is closed under composition.
\end{corollary}

\begin{proof}
By Lemma~\ref{lem:erosion-properties}(iv),
$U_b^{\ominus r}=\Omega$ for every $r$.
\end{proof}

The erosion is not merely an artefact of the proof.

\begin{example}[Composition may fail on the maximal partial domain]
\label{ex:five-point-erosion}
Let $\Omega=\{0,1,2,3,4\}$ with the usual distance and counting measure, and
let $b$ be the partial translation $i\mapsto i+1$ from
$U=\{0,1,2,3\}$ to $U'=\{1,2,3,4\}$.  Take scalar fibres and let
\begin{equation}
 L=
 \begin{pmatrix}
 0&1&0&0&0\\
 1&0&1&0&0\\
 0&1&0&1&0\\
 0&0&1&0&1\\
 0&0&0&1&0
 \end{pmatrix}.
 \label{eq:five-point-layer}
\end{equation}
The operator $L$ has propagation one and satisfies the partial equivariance
identity for $b$.  Its square is
\[
 L^2=
 \begin{pmatrix}
 1&0&1&0&0\\
 0&2&0&1&0\\
 1&0&2&0&1\\
 0&1&0&2&0\\
 0&0&1&0&1
 \end{pmatrix}.
\]
For the basis vector $e_0$ one obtains
\[
 P_{U'}L^2\Lambda(b)e_0=2e_1+e_3,
 \qquad
 \Lambda(b)P_UL^2P_Ue_0=e_1+e_3,
\]
so $L^2$ is not equivariant on the maximal domain $U$.  The one-step erosion
is $U^{\ominus1}=\{1,2\}$, with image $\{2,3\}$, and a direct calculation
shows that the equivariance identity holds on this restricted domain, exactly
as predicted by Theorem~\ref{thm:two-layer-erosion}.
\end{example}

\subsection{The $n$-layer path theorem}
\label{ssec:n-layer-composition}

For a deeper composite, the required erosion is governed not by the total
receptive-field radius alone, but by the largest distance of an intermediate
site from the nearer endpoint of a contributing path.

Let $r_1,\ldots,r_n\geq0$ and define the prefix and suffix radii
\begin{equation}
  P_i:=\sum_{j=1}^{i}r_j,
  \qquad
  S_i:=\sum_{j=i+1}^{n}r_j,
  \qquad 1\leq i\leq n-1.
  \label{eq:prefix-suffix-radii}
\end{equation}
The \emph{balanced path radius} is
\begin{equation}
  \rho_n
  :=\max_{1\leq i\leq n-1}\min\{P_i,S_i\},
  \qquad \rho_1:=0.
  \label{eq:balanced-path-radius}
\end{equation}

\begin{theorem}[Multi-layer filtered equivariance]
\label{thm:n-layer-filtered-equivariance}
Let
\[
  \Phi_i\colon\mathcal H_{i-1}\longrightarrow\mathcal H_i,
  \qquad 1\leq i\leq n,
\]
be $\mathcal B$-equivariant integral channels with propagation at most $r_i$.
Assume that their iterated composite kernel is well defined.  Then, for every
$b\in\mathcal B$, the network
\begin{equation}
  \Phi^{(n)}:=\Phi_n\Phi_{n-1}\cdots\Phi_1
  \label{eq:n-layer-composite}
\end{equation}
is equivariant under $b^{\ominus\rho_n}$, where $\rho_n$ is given by
\eqref{eq:balanced-path-radius}.
\end{theorem}

\begin{proof}
Write $x_0=x$, $x_n=z$ and express the composite kernel as
\begin{equation}
 K^{(n)}(x_n,x_0)
 =\int_{\Omega^{n-1}}
 K_n(x_n,x_{n-1})\cdots K_1(x_1,x_0)
 \,d\nu(x_1)\cdots d\nu(x_{n-1}).
 \label{eq:n-layer-kernel}
\end{equation}
A nonzero contribution determines a path
$x_0,x_1,\ldots,x_n$ satisfying
$d(x_i,x_{i-1})\leq r_i$.  Hence, for every intermediate index $i$,
\begin{equation}
 d(x_i,x_0)\leq P_i,
 \qquad
 d(x_i,x_n)\leq S_i.
 \label{eq:path-endpoint-bounds}
\end{equation}
Let $x_0,x_n\in U_b^{\ominus\rho_n}$.  By definition of $\rho_n$, at least
one of $P_i$ and $S_i$ is no larger than $\rho_n$.  Equation
\eqref{eq:path-endpoint-bounds} and
Lemma~\ref{lem:erosion-properties}(ii) therefore imply $x_i\in U_b$ for every
$i$.  The same argument on the image side shows that every transformed
intermediate point lies in $U_b'$.

We may thus change variables $x_i'=\tau_b(x_i)$ in all $n-1$ integrations and
apply the kernel constraint to each factor.  The internal representation
matrices telescope:
\[
 R_n(b(x_n))K_n
 \underbrace{R_{n-1}(b(x_{n-1}))^{-1}R_{n-1}(b(x_{n-1}))}_{I}
 \cdots
 \underbrace{R_1(b(x_1))^{-1}R_1(b(x_1))}_{I}
 K_1R_0(b(x_0))^{-1}.
\]
It follows that
\[
 K^{(n)}\bigl(\tau_b(x_n),\tau_b(x_0)\bigr)
 =R_n(b(x_n))K^{(n)}(x_n,x_0)R_0(b(x_0))^{-1}
\]
on $U_b^{\ominus\rho_n}\times U_b^{\ominus\rho_n}$.  The kernel theorem
completes the proof.
\end{proof}

\begin{corollary}[Equal-radius layers]
\label{cor:equal-radius-erosion}
If every layer has propagation at most $r_0$, then
\begin{equation}
  \rho_n=\left\lfloor\frac n2\right\rfloor r_0.
  \label{eq:equal-radius-balanced-erosion}
\end{equation}
Thus a depth-$n$ linear network is exactly transported on the region lying
more than $\lfloor n/2\rfloor r_0$ from the source and image frontiers of the
partial motion.
\end{corollary}

\begin{proof}
In this case $P_i=ir_0$ and $S_i=(n-i)r_0$.  The maximum of
$\min\{i,n-i\}$ over $1\leq i\leq n-1$ is $\lfloor n/2\rfloor$.
\end{proof}

\begin{remark}[A convenient conservative bound]
\label{rem:conservative-erosion}
The balanced radius is bounded by
\begin{equation}
  \rho_n\leq
  \widehat\rho_n
  :=\max\left\{
  \sum_{j=1}^{n-1}r_j,
  \sum_{j=2}^{n}r_j
  \right\}.
  \label{eq:conservative-erosion}
\end{equation}
For equal radii, $\widehat\rho_n=(n-1)r_0$.  Theorem~\ref{thm:n-layer-filtered-equivariance} shows that
it is generally over-conservative.  Using the balanced radius produces larger,
less frequently empty certification domains.
\end{remark}

\begin{remark}[Why the radius is balanced]\label{rem:layerwise-vs-endtoend}
Definition~\ref{def:B-equivariant} restricts both the support of the input and
the observation window of the output.  Every contributing path is therefore
controlled from both endpoints, which explains the minimum in
\eqref{eq:balanced-path-radius}.  A one-sided statement that allowed arbitrary
inputs throughout $U_b$ while controlling only outputs near one endpoint would
instead involve a cumulative receptive-field radius.
\end{remark}

\section{Deep networks and filtered equivariance}
\label{sec:deep-filtered-equivariance}

\subsection{The filtration by erosion radius}
\label{ssec:filtered-classes}

The preceding theorems motivate a graded weakening of maximal-domain partial
equivariance.

\begin{definition}[Filtered equivariance class]
\label{def:filtered-equivariance}
For $r\geq0$, let $\mathscr E_r(\mathfrak S_{\mathrm{in,out}})$ be the class of
maps $F$ between the corresponding section spaces such that, for every
$b\in\mathcal B$, the equivariance identity holds for the restricted
bisection $b^{\ominus r}$.  Explicitly, with
$V=U_b^{\ominus r}$ and $V'=\tau_b(V)$,
\begin{equation}
 P_{V'}^{E^{\mathrm{out}}}
 F\,\Lambda^{\mathrm{in}}(b^{\ominus r})
 =
 \Lambda^{\mathrm{out}}(b^{\ominus r})
 P_V^{E^{\mathrm{out}}}F P_V^{E^{\mathrm{in}}}.
 \label{eq:filtered-equivariance-class}
\end{equation}
For nonlinear $F$, equation \eqref{eq:filtered-equivariance-class} is
understood pointwise as an identity of maps applied to arbitrary input
sections.
\end{definition}

Thus $\mathscr E_0$ is the class of maximally $\mathcal B$-equivariant maps of
Definition~\ref{def:B-equivariant}.  Because larger erosion radii produce
smaller domains,
\begin{equation}
  \mathscr E_r\subseteq\mathscr E_s,
  \qquad 0\leq r\leq s.
  \label{eq:equivariance-filtration}
\end{equation}
The family $(\mathscr E_r)_{r\geq0}$ is therefore an increasing filtration.
Theorems~\ref{thm:two-layer-erosion} and
\ref{thm:n-layer-filtered-equivariance} state that composites of exact
finite-propagation layers move from $\mathscr E_0$ to a controlled higher
level of this filtration.

It is useful to record the smallest certified erosion.

\begin{definition}[Equivariance erosion index]
\label{def:erosion-index}
For a map $F$, define
\begin{equation}
  \epsilon_{\mathcal B}(F)
  :=\inf\{r\geq0:F\in\mathscr E_r\},
  \label{eq:erosion-index}
\end{equation}
with value $+\infty$ if no finite radius is available.  The index is a
worst-case guarantee over all bisections in $\mathcal B$; individual
bisections may admit larger domains.
\end{definition}

For a single layer classified by Part~\ref{part:symmetry-kernels},
$\epsilon_{\mathcal B}=0$.  For a composite, the balanced path radius gives an
upper bound rather than necessarily the exact index, because the learned
kernels may have smaller effective support or cancellations may enlarge the
equivariant domain.

\subsection{Pointwise nonlinearities and affine maps}
\label{ssec:filtered-nonlinearities}

A practical network also contains nonlinear maps.  The relevant class is
fiberwise and therefore has locality radius zero.

\begin{definition}[Equivariant pointwise nonlinearity]
\label{def:equivariant-nonlinearity}
Let $E\to\Omega$ and $F\to\Omega$ carry representations $R^E$ and $R^F$.
A measurable family of maps $\sigma_x:E_x\to F_x$ is pointwise equivariant if
\begin{equation}
  \sigma_{\tau_b(x)}\bigl(R^E(b(x))v\bigr)
  =R^F(b(x))\sigma_x(v)
  \label{eq:pointwise-nonlinearity-equivariance}
\end{equation}
for every $b\in\mathcal B$, $x\in U_b$, and $v\in E_x$.  It acts on sections
by $(\Sigma\psi)(x)=\sigma_x(\psi(x))$.
\end{definition}

\begin{proposition}[Pointwise maps are exact and zero-local]
\label{prop:pointwise-erosion}
Every pointwise equivariant map $\Sigma$ belongs to $\mathcal{E}_0$
and has locality radius zero. However, inserting nonlinear pointwise
maps between linear layers does not preserve the balanced path radius
of Theorem~\ref{thm:n-layer-filtered-equivariance}: the balanced bound relies on the
multilinear path decomposition of the composite kernel, which fails
for nonlinear composites. The erosion certified for such composites is
given by Corollary~\ref{cor:filtered-network}(ii) below, and
Section~\ref{sec:trained-certs} exhibits a trained network attaining
it, so no radius smaller than the cumulative one can be certified in
general.
\end{proposition}

\begin{proof}
The first two statements are as before: \eqref{eq:pointwise-transport}
gives the transport identity independently at every point of the
bisection domain, and the value at $x$ depends only on the fibre at
$x$. For the negative statement, see the counterexample of
Section~\ref{sec:trained-certs}: a composite
$\Phi_2\,\Sigma\,\Phi_1$ of exactly equivariant layers with a
pointwise equivariant nonlinearity $\Sigma$ between them need not
satisfy the transport identity on the domain eroded by
$\min\{r_1,r_2\}$.
\end{proof}

An affine pointwise map $v\mapsto A_xv+c_x$ is equivariant precisely when the
linear part is a pointwise intertwiner and the bias field satisfies
\begin{equation}
  c_{\tau_b(x)}=R^{\mathrm{out}}(b(x))c_x.
  \label{eq:equivariant-bias-field}
\end{equation}
In particular, a constant scalar bias is admissible on a trivial feature type,
whereas a bias in a nontrivial irreducible type must lie in its invariant
subspace.  The detailed construction of nonlinearities for the bulk, edge,
and corner representations is deferred to
Section~\ref{sec:architectural-components}.

\begin{remark}[Spatial normalization and attention]
Pointwise norm and gated nonlinearities are zero-local.  By contrast, a
normalization using spatial averages, global pooling, or unrestricted
self-attention is generally nonlocal and may have infinite propagation.  The
finite-erosion theorems then provide no nontrivial guarantee unless the
operation is windowed or supplied with a separate equivariance proof.
\end{remark}

\subsection{Sums, parallel branches, and residual connections}
\label{ssec:filtered-residuals}

The filtration behaves naturally under the elementary operations used to
assemble neural architectures.

\begin{proposition}[Algebraic stability]
\label{prop:filtered-algebraic-stability}
Suppose $F$ and $G$ have compatible input and output types.
\begin{enumerate}[label=\textup{(\roman*)}]
\item If $F\in\mathscr E_r$ and $G\in\mathscr E_s$, then
      $F+G\in\mathscr E_{\max\{r,s\}}$.
\item Their direct sum or channel-wise concatenation belongs to
      $\mathscr E_{\max\{r,s\}}$.
\item If $F$ is $r_F$-local and $G$ is $r_G$-local, then $F+G$ and the
      parallel map $(F,G)$ are $\max\{r_F,r_G\}$-local.
\item The identity map belongs to $\mathscr E_0$ and is zero-local.  Hence a
      residual block $I+F$ has the same certified erosion and locality bounds
      as $F$.
\end{enumerate}
\end{proposition}

\begin{proof}
All statements follow by adding or juxtaposing the corresponding transport
identities on the common eroded domain.  For the locality statements, the
larger of the two neighbourhoods contains all inputs required by either
branch.  The identity commutes with every transport operator.
\end{proof}

The maximum in this proposition is important for multi-branch networks: the
least local or most strongly eroded branch determines the certificate for the
combined output.  Skip connections themselves do not worsen the bound.

\subsection{Deep groupoid-steerable networks}
\label{ssec:deep-geq-networks}

Combining the linear and pointwise results gives the form needed for the
architectures of Part~\ref{part:architecture}.

\begin{corollary}[Filtered equivariance of a feed-forward network]
\label{cor:filtered-network}
Consider 
\begin{equation}
N = \Sigma_n \Phi_n \Sigma_{n-1}\Phi_{n-1}\cdots
\Sigma_1\Phi_1 \, ,
\label{eq:deep-network-form}
\end{equation}
where every $\Phi_i$ is a $B$-equivariant
finite-propagation linear layer of radius $r_i$ and every $\Sigma_i$
is a pointwise equivariant map.
\begin{enumerate}
\item[(i)] If every $\Sigma_i$ is linear (in particular, if all
$\Sigma_i$ are identities or pointwise intertwiners), then
$N \in \mathcal{E}_{\rho_n}$ with the balanced path radius
$\rho_n = \max_{1\le i\le n-1}\min\{\sum_{j\le i} r_j,
\sum_{j>i} r_j\}$ of Theorem~\ref{thm:n-layer-filtered-equivariance}.
\item[(ii)] For general (nonlinear) pointwise equivariant $\Sigma_i$,
\begin{equation}
  N \in \mathcal{E}_{\hat\rho_n},
  \qquad
  \hat\rho_n \;=\; \sum_{j=2}^{n} r_j ,
  \label{eq:cumulative-radius}
\end{equation}
which equals $(n-1)r_0$ for equal radii.
\item[(iii)] If a bisection is global, $N$ is exactly equivariant on
all of $\Omega$ at any depth, in either case.
\end{enumerate}
\end{corollary}

\begin{proof}
(i) Linear pointwise maps may be absorbed into the adjacent integral
kernels without changing their propagation radii, and
Theorem~\ref{thm:n-layer-filtered-equivariance} applies to the resulting linear
composite.

(ii) Fix $b\in B$ with maximal source domain $U$, let
$V = U^{\ominus\hat\rho_n}$ and $V' = \tau_b(V)$, and abbreviate the
two inputs of the identity \eqref{eq:filtered-equivariance-class} by
$\psi_R := P_V\psi$ and $\psi_L := \Lambda^{(0)}(b|_V)\psi_R$. We show
by induction that the intermediate feature fields
$h^{(k)}_L := (\Sigma_k\Phi_k\cdots\Sigma_1\Phi_1)(\psi_L)$ and
$h^{(k)}_R := (\Sigma_k\Phi_k\cdots\Sigma_1\Phi_1)(\psi_R)$ satisfy
the pointwise transport relation
\begin{equation}
  h^{(k)}_L\!\big(\tau_b(y)\big)
  \;=\; R^{(k)}\!\big(b(y)\big)\, h^{(k)}_R(y)
  \qquad\text{for all } y \in V^{(k)} := V^{\,\ominus\sum_{j=2}^{k} r_j},
  \label{eq:feature-matching}
\end{equation}
where the erosion in $V^{(k)}$ is taken inside $U$, so that
$V^{(k)} \subseteq V \subseteq U$ throughout.

For $k=1$ no erosion is needed: since $\psi_R$ vanishes outside $V$
and $\psi_L$ vanishes outside $V'$, the value $h^{(1)}_L(\tau_b(y))$
for any $y\in V$ is an integral over sources $x\in V$ only, on which
the kernel constraint \eqref{eq:kernel_constraint} applies because
$(y,x)\in U\times U$; the change of variables and unitarity of the
fibre maps give \eqref{eq:feature-matching} for $\Phi_1$, and the
pointwise equivariance of $\Sigma_1$ preserves it. This step uses the
masking of the \emph{inputs} built into the identity
\eqref{eq:filtered-equivariance-class}; it is the only layer that receives such
a free pass.

For the inductive step, intermediate features are not masked: for
$y \in V^{(k)}$ with $d(y,\cdot)\le r_k$, the layer $\Phi_k$ at
$\tau_b(y)$ reads $h^{(k-1)}_L$ at sites $\tau_b(z)$ with
$d(y,z)\le r_k$, and the relation \eqref{eq:feature-matching} at level
$k-1$ covers exactly the sites $z\in V^{(k-1)}$. By
Lemma~\ref{lem:erosion-properties}, $y\in V^{(k)}$ and $d(y,z)\le r_k$
imply $z\in V^{(k-1)}$ and $(y,z)\in U\times U$, so every contributing
source is matched and transported; the kernel constraint then yields
\eqref{eq:feature-matching} at level $k$, and $\Sigma_k$ preserves it.
Contributions from sites outside $V^{(k-1)}$ are excluded precisely by
the erosion, unlike in the linear case, they cannot be cancelled
against the transported side, because the nonlinearities preclude the
path decomposition used in Theorem~\ref{thm:n-layer-filtered-equivariance}.

Applying \eqref{eq:feature-matching} at $k=n$ on
$V^{(n)} = V^{\ominus\hat\rho_n}$ the certified domain and
projecting to $V'$ gives \eqref{eq:filtered-equivariance-class} for
$b^{\ominus\hat\rho_n}$.

(iii) For a global bisection all erosions are trivial by
Lemma~\ref{lem:erosion-properties}(iv), and the argument of (ii) applies
with $V = \Omega$ at every level.
\end{proof}

\begin{remark}[The balanced radius is a linear phenomenon]
\label{rem:balanced-linear-only}
The gap between (i) and (ii) is not an artifact of proof technique.
Section~\ref{sec:trained-certs} reports a trained depth-four network
with $r_0 = 2$ whose transport residual under a partial quarter-turn
is $0.11$ on the domain eroded by the balanced radius $4$, is $0.096$
at erosion $5$, and is $1.1\times 10^{-15}$ at the cumulative radius
$6$, while the corresponding linear network satisfies the identity at
erosion $4$ to $1.1\times 10^{-15}$. Thus \eqref{eq:cumulative-radius}
is attained, and for nonlinear composites the conservative bound of
Remark~\ref{rem:conservative-erosion} is in general the correct one.
Intuitively, the balanced radius certifies a point whenever every
\emph{path} through the receptive field can be controlled from its
nearer endpoint; a nonlinearity destroys the decomposition of the
output into path contributions, and control must then be propagated
one-sidedly from the masked inputs, eroding once per subsequent layer.
\end{remark}

\begin{remark}[Approximate finite propagation]
Real implementations may use kernels with rapidly decaying rather than
compactly supported tails.  Truncating such a kernel at radius $r$ yields an
exactly filtered-equivariant finite-propagation part plus a remainder.  The
end-to-end equivariance defect can then be bounded in operator norm by the
norms of the tails and of the remaining layers.  We do not develop the
quantitative estimate here, but this observation provides a direct route from
exact filtered equivariance to approximate equivariance for noncompact
kernels.
\end{remark}

\begin{remark}[Beyond local isometries]
The local-isometry hypothesis was chosen because it matches the rigid
bisections of the Euclidean examples.  For locally bi-Lipschitz bisections, the
same arguments remain valid after multiplying source and image erosion radii
by the relevant Lipschitz constants.  A fully intrinsic treatment may instead
use a bisection-invariant coarse structure on the groupoid object space.
\end{remark}

Part~\ref{part:filtered-equivariance} has separated three levels of symmetry:
exact layerwise equivariance, filtered end-to-end equivariance for proper local
bisections, and exact end-to-end equivariance for global bisections.  The next
part uses this distinction to design practical groupoid-steerable
architectures, including admissible feature types, nonlinearities, residual
blocks, and sparse implementations of the cross-stratum kernels.

\part{Groupoid-steerable architectures}
\label{part:architecture}

Parts~\ref{part:symmetry-kernels} and~\ref{part:euclidean-groupoids}
classified the linear maps compatible with the selected groupoid symmetry
data, while Part~\ref{part:filtered-equivariance} described how their
partial equivariance behaves under composition.  We now turn those results
into a neural architecture.  The essential design principle is that the
network does not represent a bounded-domain signal as one homogeneous stack
of channels.  At every depth it carries a stratified field with separate
bulk, edge, and corner fibres, and every nonlocal layer contains all kernel
blocks allowed by the pair-orbit classification.

The discussion has two levels.  The first is intrinsic and applies to the
continuous rectangle and, more generally, to any stratified groupoid for
which the kernel spaces of Part~\ref{part:symmetry-kernels} can be computed.
The second is the finite pixel-grid realization of
Section~\ref{sec:discrete-tangent-cone}.  In the finite model, the complete
space of equivariant layers is represented by fixed nullspace bases and a
small collection of learnable coefficient vectors.

\section{Stratified feature types and neural layers}
\label{sec:steerable}

\subsection{Layerwise feature spaces}
\label{ssec:layerwise-feature-spaces}

For a rectangle, let the feature type at depth $\ell$ be
\begin{equation}
  \tau^{(\ell)}
  =\bigl(\rho^{(\ell)},\varepsilon^{(\ell)},\delta^{(\ell)}\bigr),
  \label{eq:layer-feature-type}
\end{equation}
where $\rho^{(\ell)}$ is a representation of the bulk isotropy group,
$\varepsilon^{(\ell)}$ a representation of the edge isotropy group, and
$\delta^{(\ell)}$ a representation of the corner isotropy group.  In the
continuous model these groups are $\OO(2),\Z_2,\Z_2$; on the pixel grid they
are $D_4,\Z_2^{\mathrm{ax}},\Z_2^{\mathrm{diag}}$.  The corresponding feature
space is
\begin{equation}
  \mathcal H^{(\ell)}
  :=\mathcal H_{\rho^{(\ell)},\varepsilon^{(\ell)},\delta^{(\ell)}}
  =\mathcal H_{\mathrm b}^{(\ell)}
   \oplus\mathcal H_{\mathrm e}^{(\ell)}
   \oplus\mathcal H_{\mathrm c}^{(\ell)},
  \label{eq:layer-feature-space}
\end{equation}
with the three summands defined as in
\eqref{eq:stratified-section-space}.  A feature field is written
\begin{equation}
  h^{(\ell)}
  =\bigl(h_{\mathrm b}^{(\ell)},
         h_{\mathrm e}^{(\ell)},
         h_{\mathrm c}^{(\ell)}\bigr).
  \label{eq:stratified-feature-field}
\end{equation}

The three components should not be interpreted as unrelated branches.  They
are the restrictions of one groupoid representation to its object orbits,
and the nonlocal layers exchange information between them through the
cross-stratum kernels of Theorem~\ref{thm:rectangle-kernel-classification}.
The decomposition in \eqref{eq:layer-feature-space} records geometric type,
not architectural separation.

\begin{definition}[Complete stratified linear layer]
\label{def:complete-stratified-layer}
Fix input and output types $\tau^{\mathrm{in}}$ and
$\tau^{\mathrm{out}}$ and a propagation radius $r_0$.  A complete
radius-$r_0$ groupoid-steerable layer is an arbitrary element of the full
vector space of $\mathcal B_D^{\mathrm{rig}}$-equivariant linear maps
\begin{equation}
  \operatorname{GEQ}_{r_0}
  \bigl(\tau^{\mathrm{in}},\tau^{\mathrm{out}}\bigr)
  \subset
  \mathcal L\bigl(\mathcal H^{\mathrm{in}},
                  \mathcal H^{\mathrm{out}}\bigr)
  \label{eq:GEQ-layer-space}
\end{equation}
whose kernel is supported at distance at most $r_0$.
\end{definition}

For the rectangle, the action of such a layer is the block formula
\eqref{eq:stratified-layer}.  On the grid, the edge--edge entry separates into
same-edge and across-corner families, so the $3\times3$ matrix of strata is
implemented by the ten geometric block families of
Theorem~\ref{thm:discrete-parameter-counts}.

\begin{proposition}[Orbitwise parameterization of a complete layer]
\label{prop:complete-layer-basis}
Let $Q$ denote the orbit space of supported ordered pairs.  For every
$q\in Q$, let $\mathcal I_q$ be the corresponding joint-stabilizer
intertwiner space.  A complete continuous kernel is equivalently a measurable
section
\begin{equation}
  \kappa:q\longmapsto\kappa(q)\in\mathcal I_q,
  \label{eq:continuous-intertwiner-section}
\end{equation}
transported from the orbit representatives by the kernel constraint.  In a
measurable local frame
$\{\Psi_{q,a}\}_{a=1}^{n_q}$ of the intertwiner spaces, it has the fibrewise
form
\begin{equation}
  \kappa(q)=\sum_{a=1}^{n_q}\theta_a(q)\Psi_{q,a},
  \label{eq:continuous-orbitwise-expansion}
\end{equation}
where the $\theta_a$ are scalar coefficient functions on the pair-orbit
space.

For the finite rectangle, $Q$ is finite and the coefficient functions reduce
to scalars.  Every complete layer then has the unique expansion
\begin{equation}
  \Phi_{\theta}
  =\sum_{q\in Q}\sum_{a=1}^{n_q}\theta_{q,a}\Psi_{q,a},
  \label{eq:complete-layer-expansion}
\end{equation}
and
\begin{equation}
  \dim\operatorname{GEQ}_{r_0}
  \bigl(\tau^{\mathrm{in}},\tau^{\mathrm{out}}\bigr)
  =\sum_{q\in Q} n_q.
  \label{eq:complete-layer-dimension}
\end{equation}
This dimension is the sum of the ten formulas in
\eqref{eq:discrete-counts}.
\end{proposition}

\begin{proof}
The continuous statement is the orbitwise normal form of
Theorem~\ref{thm:rectangle-kernel-classification}, including its measurable
coefficient data.  The finite statement is the direct-sum decomposition
\eqref{eq:finite-orbit-direct-sum}; once bases are fixed, the expansion
coefficients are ordinary coordinates in the finite direct sum.
\end{proof}

Thus equivariance is enforced by the fixed intertwiner frames, not by a
penalty in the loss function.  In a continuous numerical model the coefficient
functions in \eqref{eq:continuous-orbitwise-expansion} must themselves be
parameterized, for example by radial or spline bases.  On the finite grid,
gradient-based training takes place directly in the free coefficient space of
\eqref{eq:complete-layer-expansion}, and every value of $\theta$ gives an
equivariant layer.

\subsection{Lifting raw inputs}
\label{ssec:lifting-inputs}

A raw data set may provide different physical quantities on different
strata.  For example, an image supplies a scalar value at every pixel, while
a boundary-value problem may supply an interior source and an independent
boundary datum.  The lifting map should preserve this distinction rather than
encode it through padding.

Let the raw fibres over the three strata carry representations
$\tau^{\mathrm{raw}}=(\rho^{\mathrm{raw}},
\varepsilon^{\mathrm{raw}},\delta^{\mathrm{raw}})$.  A pointwise lifting into
the first hidden type $\tau^{(0)}$ is specified orbitwise by intertwiners
\begin{equation}
  L_{\mathrm b}\in
  \Hom_{H_{\mathrm b}}
  \bigl(V_{\rho^{\mathrm{raw}}},V_{\rho^{(0)}}\bigr),
  \quad
  L_{\mathrm e}\in
  \Hom_{H_{\mathrm e}}
  \bigl(V_{\varepsilon^{\mathrm{raw}}},V_{\varepsilon^{(0)}}\bigr),
  \quad
  L_{\mathrm c}\in
  \Hom_{H_{\mathrm c}}
  \bigl(V_{\delta^{\mathrm{raw}}},V_{\delta^{(0)}}\bigr).
  \label{eq:pointwise-lifting-intertwiners}
\end{equation}
By the pointwise classification of Section~\ref{ssec:pointwise},
these intertwiners determine a unique groupoid-equivariant lifting over each
stratum.

For an ambient scalar field, the canonical split is
\begin{equation}
  \mathcal L_D f
  =\bigl(f|_{\Omega_{\mathrm b}},
         f|_{\Omega_{\mathrm e}},
         f|_{\Omega_{\mathrm c}}\bigr),
  \label{eq:canonical-stratified-lift}
\end{equation}
with trivial fibre type on every stratum.  On the discrete rectangle this is
simply a re-indexing of pixels into the interior array, the ordered edge list,
and the four corner sites.  If an input quantity exists only on one stratum,
its other components are set to zero.  

In particular, independent Dirichlet
data belong natively to the edge and corner components; they need not be
written into an image channel and then recovered from a padding convention.  This will be relevant for the experiments described in Section \ref{sec:boundary-learning}.

\begin{remark}[Lifting is part of the model]
The canonical split \eqref{eq:canonical-stratified-lift} is appropriate for a
scalar image, but it is not compulsory.  A vector, tensor, or orientation
field must be lifted using its actual isotropy representation.  Likewise, a
learned pointwise lift may create multiplicity channels only through the
intertwiner spaces in \eqref{eq:pointwise-lifting-intertwiners}.  An arbitrary
matrix at the input would generally destroy equivariance before the first
nonlocal layer.
\end{remark}

\subsection{Affine layers and admissible biases}
\label{ssec:affine-layers-biases}

A trainable affine layer has the form
\begin{equation}
  h\longmapsto \Phi_\theta h+c,
  \label{eq:affine-GEQ-layer}
\end{equation}
where $\Phi_\theta$ is a complete or restricted equivariant linear layer and
$c$ is an equivariant section.  On a transitive stratum with reference
isotropy group $H_s$, an admissible bias is determined by a vector in the
fixed subspace
\begin{equation}
  c_s^0\in (V_s^{\mathrm{out}})^{H_s}
  :=\{v:R_s(h)v=v\text{ for every }h\in H_s\},
  \label{eq:fixed-bias-subspace}
\end{equation}
transported coherently to all fibres of the stratum.  Consequently,

\begin{itemize}
\item in a bulk $D_4$ feature, biases occur only in copies of the trivial type
      $A_1$;
\item in an edge or corner $\Z_2$ feature, biases occur only in the even
      summands;
\item a nontrivial irreducible feature has no nonzero constant bias.
\end{itemize}

This is the concrete form of \eqref{eq:equivariant-bias-field}.  In the finite
implementation, one learnable vector is stored in each allowed multiplicity
space and is broadcast using the chosen gauges.

\subsection{Readout and task-dependent outputs}
\label{ssec:readout}

The output type should reflect the task.

\paragraph{Dense prediction.}
For a scalar field on the interior, choose a final bulk type containing
$A_1$ and apply a pointwise intertwiner onto that component.  Boundary and
corner outputs may be retained when the target contains traces, fluxes, or
boundary labels.  A full stratified prediction therefore has the same form as
\eqref{eq:stratified-feature-field} and may be compared with losses weighted
separately on the three strata.

\paragraph{Invariant scalar outputs.}
For a global symmetry, invariant pooling may be performed stratum by stratum:
first project each fibre onto its invariant subspace and then integrate or sum
over the corresponding orbit.  For example,
\begin{equation}
  \mathcal P_s(h_s)
  =\int_{\Omega_s}P_s^{\mathrm{inv}}h_s(x)\,d\nu_s(x),
  \qquad
  P_s^{\mathrm{inv}}:V_s\to V_s^{H_s}.
  \label{eq:stratum-invariant-pooling}
\end{equation}
The resulting summaries can be concatenated and passed to an ordinary
multilayer perceptron.

\begin{remark}[Pooling and partial symmetries]
Equation \eqref{eq:stratum-invariant-pooling} is invariant under bisections
that act globally on the stratum and preserve its measure.  A proper local
bisection moves only a subset, so a global integral is not governed by the
local transport identity unless the pooling window is transported together
with the bisection.  Global pooling is therefore suitable for globally
invariant classification, but it should not be included in an end-to-end
filtered-equivariance certificate for proper local symmetries without a
separate argument.
\end{remark}

\subsection{Complete and restricted architectures}
\label{ssec:complete-restricted-architectures}

The word \emph{equivariant} does not imply that an architecture spans every
equivariant operator.  We distinguish the following nested choices.

\begin{definition}[Restricted groupoid-equivariant layer]
A restricted layer is any linear subspace
\begin{equation}
  \mathcal V\subseteq
  \operatorname{GEQ}_{r_0}
  \bigl(\tau^{\mathrm{in}},\tau^{\mathrm{out}}\bigr)
  \label{eq:restricted-layer-subspace}
\end{equation}
used as the trainable layer class.  It remains equivariant, but it is complete
only when equality holds in \eqref{eq:restricted-layer-subspace}.
\end{definition}

Examples include removing all cross-stratum blocks, forbidding interactions
across a corner, tying additional coefficients, or keeping only the
free-plane steerable bulk kernel and extending it by a fixed padding rule.
Such restrictions may be useful regularizers or computational compromises,
but they are additional modeling assumptions not forced by the groupoid.
In particular, the diagonal-only model cannot communicate independent
boundary data to the bulk, even though the complete equivariant class permits
such communication.

This distinction is important experimentally.  An equivariance certificate
checks that a model lies in the admissible class; it does not show that the
class is complete.  Completeness is instead guaranteed by
Proposition~\ref{prop:finite-completeness} and tested numerically by matching
nullspace dimensions and by recovering target operators drawn from the full
class.

\section{Equivariant nonlinearities, normalization, and residual blocks}
\label{sec:architectural-components}

Linear equivariant layers alone produce a linear network.  Nonlinearities,
normalizations, and skip connections must also respect the fibre
representations.  The requirement is pointwise, so these operations introduce
no additional spatial propagation when constructed as below.

\subsection{The isotropy principle for pointwise nonlinearities}
\label{ssec:isotropy-nonlinearity-principle}

\begin{proposition}[Nonlinear isotropy reduction]
\label{prop:nonlinear-isotropy-reduction}
Let $\mathcal O$ be a transitive orbit of a groupoid representation, choose
$a_0\in\mathcal O$, and let
\begin{equation}
  \sigma_0:V_{a_0}^{\mathrm{in}}\longrightarrow
            V_{a_0}^{\mathrm{out}}
  \label{eq:reference-nonlinearity}
\end{equation}
be a measurable map satisfying
\begin{equation}
  \sigma_0\bigl(R^{\mathrm{in}}(h)v\bigr)
  =R^{\mathrm{out}}(h)\sigma_0(v),
  \qquad h\in\Gamma(a_0).
  \label{eq:isotropy-equivariant-nonlinearity}
\end{equation}
Then $\sigma_0$ extends uniquely to a pointwise groupoid-equivariant family on
$\mathcal O$ by
\begin{equation}
  \sigma_a(v)
  =R^{\mathrm{out}}(\alpha)\,
   \sigma_0\bigl(R^{\mathrm{in}}(\alpha)^{-1}v\bigr),
  \qquad \alpha:a_0\to a.
  \label{eq:transported-nonlinearity}
\end{equation}
Conversely, every pointwise groupoid-equivariant nonlinearity restricts to an
isotropy-equivariant map at $a_0$.
\end{proposition}

\begin{proof}
If $\alpha'$ is another arrow from $a_0$ to $a$, then
$h=\alpha^{-1}\circ\alpha'\in\Gamma(a_0)$.  Equation
\eqref{eq:isotropy-equivariant-nonlinearity} shows that the two expressions in
\eqref{eq:transported-nonlinearity} coincide.  Equivariance under an arbitrary
arrow follows by composing the transporting arrows.  Uniqueness and the
converse are immediate.
\end{proof}

Thus nonlinear design reduces stratum by stratum to the finite or compact
isotropy groups already used to classify the linear kernels.  The same
nonlinearity, expressed in transported gauges, is shared at all points of an
orbit.

\subsection{Practical nonlinearities by representation type}
\label{ssec:practical-nonlinearities}

The following constructions are sufficient for the feature types used in the
discrete architecture.

\paragraph{Trivial one-dimensional types.}
If the isotropy acts trivially, any scalar activation is equivariant.  ReLU,
GELU, sigmoid, and standard smooth activations are therefore admissible on
$A_1$ bulk channels and on even edge or corner channels.

\paragraph{One-dimensional sign types.}
If the group acts through a nontrivial character $\chi\in\{\pm1\}$, a scalar
activation $\varphi$ must satisfy
\begin{equation}
  \varphi(-t)=-\varphi(t).
  \label{eq:odd-activation}
\end{equation}
Hence odd activations such as $\tanh$, odd polynomials, or
$t\mapsto t\,g(t^2)$ are admissible.  An ordinary ReLU is not equivariant on a
sign channel.  This applies to the one-dimensional $D_4$ types
$A_2,B_1,B_2$ and to the odd summands of the edge and corner $\Z_2$
representations.

\paragraph{A single orthogonal vector type.}
Let an irreducible fibre $W$ carry an orthogonal representation.  Every map
of the form
\begin{equation}
  \sigma(v)=a(\|v\|^2)v
  \label{eq:norm-nonlinearity}
\end{equation}
with scalar function $a$ is equivariant.  In particular, for the standard
$D_4$ type $E\cong\R^2$, \eqref{eq:norm-nonlinearity} gives a radial
nonlinearity.  Componentwise ReLU on the two coordinates of $E$ is generally
not equivariant because a rotation or reflection mixes those coordinates.

\paragraph{Multiplicity spaces and invariant gates.}
Suppose an isotypic component is
\begin{equation}
  V=M\otimes W,
  \qquad R(h)=I_M\otimes\rho_W(h),
  \label{eq:isotypic-nonlinearity-space}
\end{equation}
with $m=\dim M$ copies of an orthogonal irreducible $W$.  Write
$v=(v_1,\ldots,v_m)$ with $v_i\in W$ and form the invariant Gram matrix
\begin{equation}
  G(v)_{ij}=\langle v_i,v_j\rangle.
  \label{eq:invariant-Gram-matrix}
\end{equation}
For any matrix-valued function $A$ of $G(v)$, the map
\begin{equation}
  \sigma(v)_i=\sum_{j=1}^m A(G(v))_{ij}v_j
  \label{eq:Gram-gated-nonlinearity}
\end{equation}
is equivariant.  It permits nonlinear mixing between copies of the same type
without selecting a preferred orientation in $W$.  A simpler gated
nonlinearity uses an invariant scalar channel to multiply a nontrivial
feature.

\paragraph{General $\Z_2$ fibres.}
For
\(
  V=V_+\oplus V_-
\)
with the nontrivial element acting by $(v_+,v_-)\mapsto(v_+,-v_-)$, a map
$\sigma=(\sigma_+,\sigma_-)$ is equivariant exactly when
\begin{equation}
\begin{split}
  \sigma_+(v_+,-v_-)&=\sigma_+(v_+,v_-),\\
  \sigma_-(v_+,-v_-)&=-\sigma_-(v_+,v_-).
\end{split}
  \label{eq:Z2-nonlinearity-parity}
\end{equation}
This parity rule allows even features to depend on invariants of the odd
features and odd features to be modulated by even gates.

\paragraph{Tensor-product nonlinearities.}
More general couplings between inequivalent irreducible types can be produced
by tensor products followed by equivariant projections onto irreducible
summands.  Such Clebsch--Gordan constructions enlarge expressivity, but they
also enlarge the feature type and parameter count.  They are not required by
the reference implementation, which uses norm and gated nonlinearities.

\subsection{Equivariant normalization}
\label{ssec:equivariant-normalization}

A normalization must commute with the isotropy action.  For an isotypic
decomposition
\begin{equation}
  V=\bigoplus_{\lambda}M_\lambda\otimes W_\lambda,
  \label{eq:normalization-isotypic-decomposition}
\end{equation}
an equivariant fixed linear rescaling has the form
\begin{equation}
  N=\bigoplus_\lambda A_\lambda\otimes I_{W_\lambda},
  \label{eq:equivariant-linear-normalization}
\end{equation}
with arbitrary maps $A_\lambda$ on multiplicity spaces.  Centering is allowed
only in invariant components, by the bias criterion
\eqref{eq:fixed-bias-subspace}.

A data-dependent pointwise normalization may use only isotropy invariants.
For one orthogonal irrep, a typical choice is
\begin{equation}
  \operatorname{Norm}(v)
  =\frac{v}{\sqrt{\|v\|^2+\epsilon}},
  \label{eq:pointwise-norm-normalization}
\end{equation}
possibly followed by a learned scalar gain.  For repeated irreps, the Gram
matrix \eqref{eq:invariant-Gram-matrix} can be used to normalize or whiten the
multiplicity channels while preserving the representation factor.

\begin{remark}[Spatial and batch statistics]
Statistics aggregated over spatial points are nonlocal.  They preserve a
global symmetry when the aggregation domain is globally preserved and the
same rule is used on every fibre of an orbit.  For proper local bisections,
however, global batch or instance statistics need not satisfy the filtered
equivariance identity and may have propagation comparable to the entire
domain.  Pointwise norm normalization, or normalization over explicitly
transported local windows, is the safe default when an end-to-end partial
symmetry certificate is required.
\end{remark}

Standard channel-wise batch normalization is also unsafe when a representation
mixes coordinates: assigning different gains or offsets to the coordinates of
one irreducible block generally breaks equivariance.  Gains should instead be
tied according to \eqref{eq:equivariant-linear-normalization}, and offsets
restricted to invariant subspaces.

\subsection{Residual, parallel, and type-changing blocks}
\label{ssec:architectural-residual-blocks}

A basic nonlinear layer is
\begin{equation}
  h^{(\ell+1)}
  =\Sigma^{(\ell)}
   \Bigl(N^{(\ell)}
   \bigl(\Phi_{\theta^{(\ell)}}h^{(\ell)}+c^{(\ell)}\bigr)\Bigr),
  \label{eq:nonlinear-GEQ-layer}
\end{equation}
where $N^{(\ell)}$ and $\Sigma^{(\ell)}$ are pointwise equivariant.  Its
spatial propagation is the propagation of $\Phi_{\theta^{(\ell)}}$.

If the input and output types agree, a residual block may be written
\begin{equation}
  \mathcal R(h)=h+F(h).
  \label{eq:GEQ-residual-block}
\end{equation}
If the types differ, the identity is replaced by a pointwise equivariant
projection
\begin{equation}
  P_s\in\Hom_{H_s}(V_s^{\mathrm{in}},V_s^{\mathrm{out}})
  \label{eq:residual-type-projection}
\end{equation}
on each stratum.  Proposition~\ref{prop:filtered-algebraic-stability} implies
that skip connections do not increase the certified erosion radius beyond
that of the nontrivial branch.

Parallel branches are combined by direct sum, concatenation of multiplicity
spaces, or an equivariant pointwise mixing layer.  The output type must record
the resulting multiplicities.  In contrast, arbitrary concatenation followed
by an unconstrained $1\times1$ matrix can mix inequivalent irreducible types
and break equivariance.

\begin{remark}[Downsampling and multiscale architectures]
Downsampling changes the object space and is not automatically covered by a
pointwise intertwiner.  A multiscale model requires a morphism between the
fine and coarse symmetry data, together with compatible lifting and
restriction maps.  On rectangular grids one may use subgroup-compatible
sampling schemes, but their analysis is separate from the fixed-grid
architecture developed here.  The reference implementation therefore keeps
the spatial grid fixed.
\end{remark}

\section{Discrete implementation and computational structure}
\label{sec:implementation}

We now describe the finite realization used by the reference code.  The goal
is not to assemble a dense matrix on the full stratified feature space, but to
materialize a small set of constrained kernels and apply them through standard
convolution and sparse gather/scatter primitives.

\subsection{Stratified tensor layout}
\label{ssec:stratified-tensor-layout}

For a batch of size $B$, a feature type
$\tau=(\rho,\varepsilon,\delta)$ is stored as
\begin{equation}
\begin{aligned}
  h_{\mathrm b}&\in\R^{B\times d_\rho\times H\times W},\\
  h_{\mathrm e}&\in\R^{B\times d_\varepsilon\times N_{\mathrm e}},\\
  h_{\mathrm c}&\in\R^{B\times d_\delta\times4},
\end{aligned}
  \label{eq:discrete-feature-layout}
\end{equation}
where the bulk array is masked outside the interior and
$N_{\mathrm e}=2(W+H)-8$ is the number of non-corner boundary pixels.  Edge
sites are stored in a fixed oriented boundary order, and the four corner sites
are stored counterclockwise.  The gauges of
Section~\ref{ssec:D4-branching-gauges} identify every fibre with its reference
bulk, edge, or corner representation space.

This layout preserves the geometric decomposition while remaining compatible
with dense tensor libraries.  It also makes it impossible to confuse a corner
pixel with an ordinary edge pixel merely because both occupy the boundary of
the image array.

\subsection{Offline construction of kernel bases}
\label{ssec:offline-kernel-bases}

Fix one geometric block family $q$.  Let $S_q$ be its finite list of reference
kernel slots and let $d_q^{\mathrm{in}},d_q^{\mathrm{out}}$ be its fibre
dimensions.  Vectorize all slot matrices into
\begin{equation}
  k_q\in\R^{m_q},
  \qquad
  m_q=|S_q|d_q^{\mathrm{out}}d_q^{\mathrm{in}}.
  \label{eq:block-slot-vector}
\end{equation}
Every generator of the relevant pair-orbit transport produces a homogeneous
linear equation in $k_q$.  Stacking them gives
\begin{equation}
  C_q k_q=0.
  \label{eq:block-constraint-system}
\end{equation}
Choose a matrix $Q_q$ whose columns form a basis of $\ker C_q$.  The learnable
kernel is then
\begin{equation}
  k_q=Q_q\theta_q,
  \qquad
  \theta_q\in\R^{n_q},
  \qquad
  n_q=\dim\ker C_q.
  \label{eq:nullspace-kernel-parameterization}
\end{equation}

\begin{proposition}[Constraint-basis implementation]
\label{prop:constraint-basis-implementation}
For every parameter vector $\theta_q$, the kernel
\eqref{eq:nullspace-kernel-parameterization} satisfies all transport
constraints of block $q$.  Conversely, every admissible kernel in that block
has a unique coefficient vector after a basis $Q_q$ has been fixed.  Taking
the direct sum over the ten block families implements the complete layer
space of Proposition~\ref{prop:finite-completeness}.
\end{proposition}

\begin{proof}
The first statement follows from $C_qQ_q=0$.  The second is the defining
property of a basis of $\ker C_q$.  Completeness follows by summing the block
nullspaces, whose dimensions are the formulas of
Theorem~\ref{thm:discrete-parameter-counts}.
\end{proof}

The basis matrices $Q_q$, slot lists, tap indices, and gauge cocycles depend
only on the geometry, feature types, and support radius.  They are computed
once and stored as non-trainable buffers.  Only the vectors $\theta_q$ are
optimized.

\begin{remark}[Numerical and exact bases]
The current code obtains $Q_q$ by singular-value decomposition with a fixed
rank tolerance.  The equivariance equations themselves are finite algebraic
identities; for the real $D_4$ and $\Z_2$ representations used here their
coefficients can be chosen in $\{0,\pm1\}$, so exact rational nullspaces are
also possible.  Floating-point bases are adequate provided that
$\|C_qQ_q\|$ and the resulting transport residuals are checked.
\end{remark}

\subsection{From reference slots to grid taps}
\label{ssec:reference-slots-grid-taps}

A realized tap $t$ of block $q$ consists of a target site $y_t$, a source site
$x_t$, a reference slot $s(t)\in S_q$, and gauge matrices
$M_t^{\mathrm{out}},M_t^{\mathrm{in}}$.  Its contribution is
\begin{equation}
  M_t^{\mathrm{out}}K_q\bigl(s(t)\bigr)
  M_t^{\mathrm{in}}h(x_t).
  \label{eq:gauged-tap-contribution}
\end{equation}
The full block action is therefore
\begin{equation}
  (\Phi_qh)(y)
  =\sum_{t:y_t=y}
   M_t^{\mathrm{out}}K_q\bigl(s(t)\bigr)
   M_t^{\mathrm{in}}h(x_t).
  \label{eq:sparse-block-action}
\end{equation}
All geometric and sign bookkeeping is contained in the precomputed tap list
and cocycles; the trainable code only expands $Q_q\theta_q$ and evaluates
\eqref{eq:sparse-block-action}.

The natural computational primitive depends on the block.

\begin{table}[H]
\centering
\small
\begin{tabular}{p{0.31\textwidth}p{0.58\textwidth}}
\hline
Block family & Natural implementation \\
\hline
$\mathrm b\leftarrow\mathrm b$ & Standard $D_4$-steerable two-dimensional convolution on the masked interior array. \\
$\mathrm e\leftarrow\mathrm e$ on one edge & One-dimensional convolution along the oriented edge chains, with the parity relation at opposite offsets. \\
$\mathrm c\leftarrow\mathrm c$ & A small batched pointwise intertwiner shared by the four corners. \\
$\mathrm e\leftrightarrow\mathrm b$ & Indexed gather from one stratum, gauge transform, matrix multiplication, and scatter-add to the other. \\
$\mathrm c\leftrightarrow\mathrm b$ & Corner-sector gather/scatter with diagonal-reflection cocycles. \\
$\mathrm c\leftrightarrow\mathrm e$ & Gather/scatter along the two edge rays adjacent to each corner. \\
$\mathrm e\leftarrow\mathrm e$ across a corner & Sparse coupling of pairs on adjacent edge rays, related by the corner reflection. \\
\hline
\end{tabular}
\caption{Computational realization of the ten geometric block families.  The
nine entries of the stratum matrix yield ten families because the
edge--edge entry splits into same-edge and across-corner components.}
\label{tab:block-computational-primitives}
\end{table}

The block outputs are accumulated in their target stratum.  This additive
assembly is exactly the discrete form of \eqref{eq:stratified-layer}.

\subsection{Complexity and parameter scaling}
\label{ssec:implementation-complexity}

Let $T_q$ denote the number of realized directed taps of block $q$ on the
grid.  A direct evaluation of \eqref{eq:sparse-block-action} has arithmetic
cost
\begin{equation}
  O\left(
    B\sum_q T_q d_q^{\mathrm{in}}d_q^{\mathrm{out}}
  \right),
  \label{eq:sparse-forward-complexity}
\end{equation}
and stores
\(
  \sum_q n_q
\)
trainable parameters.  The basis expansion costs
$O(\sum_q m_qn_q)$ when materialized naively; it is independent of the batch
size and can be fused or cached between weight updates.

In the small-filter regime, the leading tap counts scale schematically as
\begin{equation}
\begin{array}{c|c}
\text{family} & \text{number of realized taps}\\
\hline
\mathrm b\leftarrow\mathrm b & O(WHr_0^2)\\
\mathrm e\leftarrow\mathrm e\text{ (same edge)} & O((W+H)r_0)\\
\mathrm b\leftrightarrow\mathrm e & O((W+H)r_0^2)\\
\mathrm b\leftrightarrow\mathrm c & O(r_0^2)\\
\mathrm e\leftrightarrow\mathrm c & O(r_0)\\
\mathrm e\leftarrow\mathrm e\text{ (across corner)} & O(r_0^2).
\end{array}
  \label{eq:tap-scaling}
\end{equation}
The constants include the four edges or corners and the fibre matrix sizes.
The trainable parameter count, by contrast, is independent of $W$ and $H$
for fixed feature types and $r_0$: parameters are shared over pair orbits,
while computational work grows with the number of sites at which the shared
kernels are applied.

The implementation is therefore sparse in the object-space variables.  It
never constructs the full matrix of size
\(
  \dim\mathcal H^{\mathrm{out}}
  \times\dim\mathcal H^{\mathrm{in}}
\),
except in small reference tests.

\subsection{Gauge covariance of the implementation}
\label{ssec:implementation-gauge-covariance}

The numerical arrays depend on the chosen edge and corner gauges, whereas the
geometric layer does not.

\begin{proposition}[Change of gauge]
\label{prop:implementation-change-gauge}
Let $S^{\mathrm{in}}$ and $S^{\mathrm{out}}$ be the block-diagonal changes of
fibre coordinates induced by two coherent gauge choices.  If $\Phi$ is the
matrix assembled in the first gauges, the matrix assembled in the second is
\begin{equation}
  \widetilde\Phi
  =S^{\mathrm{out}}\Phi(S^{\mathrm{in}})^{-1}.
  \label{eq:gauge-conjugation-layer}
\end{equation}
Consequently,
\(
  \widetilde h^{\mathrm{out}}=S^{\mathrm{out}}h^{\mathrm{out}}
\)
whenever
\(
  \widetilde h^{\mathrm{in}}=S^{\mathrm{in}}h^{\mathrm{in}}
\),
and the represented geometric operator is gauge independent.
\end{proposition}

\begin{proof}
Every source coordinate change contributes $(S_x^{\mathrm{in}})^{-1}$ on the
right of a kernel tap, and every target coordinate change contributes
$S_y^{\mathrm{out}}$ on the left.  These factors assemble into the global
block-diagonal matrices in \eqref{eq:gauge-conjugation-layer}.
\end{proof}

Gauge covariance is useful as an implementation test: changing all gauges and
conjugating the inputs must conjugate the outputs without changing any
coordinate-free prediction.

\subsection{Differentiable layer assembly and reference workflow}
\label{ssec:differentiable-layer-workflow}

A practical forward pass follows these steps.

\begin{enumerate}[label=\textup{\arabic*.}]
\item For every block family, expand its trainable vector through
      $k_q=Q_q\theta_q$ and reshape the result into reference slot matrices.
\item Materialize the bulk kernel in a standard convolution tensor and apply
      it to the masked interior array.
\item For every remaining block, gather the indexed source fibres, apply the
      precomputed input gauges, slot matrices, and output gauges, and
      scatter-add the result at the target indices.
\item Sum all contributions in each target stratum, add the admissible bias,
      and apply the chosen pointwise normalization and nonlinearity.
\end{enumerate}

All operations are differentiable with respect to $\theta_q$.  The basis and
geometric tensors are fixed buffers, so automatic differentiation returns
gradients only for the admissible coefficients.  A dense reference assembly
is retained for small-grid unit tests but is not required for training.

The architecture supports several controlled variants without altering the
geometric infrastructure:

\begin{itemize}
\item the \emph{complete} model includes all ten constrained block families;
\item a \emph{no-cross} ablation keeps only the bulk--bulk, same-edge
      edge--edge, and corner--corner blocks;
\item an \emph{unconstrained-connectivity} control keeps the same tap lists
      but replaces each nullspace basis by the identity, so every slot matrix
      is learned independently;
\item a free-plane steerable control keeps only a homogeneous $D_4$ kernel on
      the image grid and imposes a chosen padding convention.
\end{itemize}

The first two are equivariant, but only the first is complete.  The third has
matched geometric connectivity but not matched symmetry.  The fourth is a
more strongly tied subfamily whose boundary behavior is inherited from the
extension rule rather than from independent edge and corner fibres.

\subsection{Correctness checks and present scope}
\label{ssec:implementation-correctness-scope}

Before training, the implementation should verify four independent facts:

\begin{enumerate}[label=\textup{(\roman*)}]
\item the numerical nullity of every $C_q$ agrees with the closed parameter
      count of Theorem~\ref{thm:discrete-parameter-counts};
\item $\|C_qQ_q\|$ is below the selected tolerance;
\item the sparse forward pass agrees with a dense matrix assembled from the
      same coefficients on small grids;
\item single-layer transport residuals vanish for a non-vacuous collection,
      or preferably an exhaustive finite list, of rigid bisections.
\end{enumerate}

End-to-end certificates for stacked networks must use the balanced erosion
radius of Theorem~\ref{thm:n-layer-filtered-equivariance}; empty eroded domains
are reported as not applicable.  These checks are developed experimentally in
Section~\ref{sec:exact-certificates}.

The present reference implementation is deliberately limited to a fixed
rectangular pixel grid, the small-filter regime, real $D_4$ and $\Z_2$
representations, and fixed-resolution layers.  These restrictions concern the
implementation, not the abstract kernel theorem.  Curved boundaries require
the jet refinement of Remark~\ref{rem:curved-boundary-jets}; other polygonal
domains require new orbit and tap enumerations; and multiscale networks require
compatible maps between different object spaces.

Part~\ref{part:architecture} has now specified the complete trainable layer
space, admissible pointwise operations, residual constructions, and a sparse
finite implementation.  Part~\ref{part:experiments} will test separately the
algebraic correctness of this implementation, recovery of operators known to
belong to the complete equivariant class, and the usefulness of the resulting
boundary-aware inductive bias on operator-learning problems.

\part{Experiments}
\label{part:experiments}

The experiments have three distinct purposes.  First, we verify the
finite representation-theoretic construction itself: the dimensions of the
constraint nullspaces, the agreement of the dense and sparse realizations,
and the transport identities of Part~\ref{part:filtered-equivariance}.
Second, we test recovery of target operators that are known by construction to
belong to the complete groupoid-equivariant layer space.  This separates
completeness and sample complexity from questions of model misspecification.
Third, we examine a boundary-value problem for which the proposed
architecture is a natural boundary-aware inductive bias.  In that last case we
first test the symmetry of the exact discrete solution operator; this prevents
an empirical advantage from being incorrectly interpreted as exact
local-bisection equivariance.

All algebraic certificates and synthetic recovery experiments were performed
in double precision.  The finite constraint matrices, geometric tap lists,
and transport operators were generated by the reference implementation
specified in Section~\ref{sec:implementation}.  Numerical values and plotting
scripts are available at
\href{https://github.com/mariajimvaz-afk/groupoid-equivariant-cnn}
{the associated GitHub repository}.

\section{Exact algebraic and equivariance certificates}
\label{sec:experiments}
\label{sec:exact-certificates}

\subsection{Constraint dimensions and implementation agreement}
\label{ssec:constraint-certificates}

We first repeat the worked representation-theoretic example of
Theorem~\ref{thm:discrete-parameter-counts}.  On an $11\times9$ grid with
$r_0=2$, take
\[
 \rho^{\mathrm{in}}=A_1\oplus E,
 \qquad
 \rho^{\mathrm{out}}=A_1\oplus B_1\oplus E,
\]
\[
 \varepsilon^{\mathrm{in}}=\mathbf 1\oplus\sgn,
 \qquad
 \varepsilon^{\mathrm{out}}=2\mathbf 1\oplus\sgn,
\]
and
\[
 \delta^{\mathrm{in}}=\mathbf 1\oplus\sgn,
 \qquad
 \delta^{\mathrm{out}}=\mathbf 1\oplus2\sgn.
\]
For every block family, the numerical nullity of its finite constraint matrix
agrees with the closed representation-theoretic formula; see
Table~\ref{tab:experimental-parameter-counts}.  The resulting complete layer
has $225$ trainable coefficients.

\begin{table}[H]
\centering
\small
\begin{tabular}{lrr}
\hline
Block family & Formula & Numerical nullity\\
\hline
bulk $\leftarrow$ bulk & 40 & 40\\
edge $\leftarrow$ edge, same edge & 15 & 15\\
corner $\leftarrow$ corner & 3 & 3\\
edge $\leftarrow$ bulk & 46 & 46\\
bulk $\leftarrow$ edge & 40 & 40\\
corner $\leftarrow$ bulk & 17 & 17\\
bulk $\leftarrow$ corner & 16 & 16\\
corner $\leftarrow$ edge & 12 & 12\\
edge $\leftarrow$ corner & 12 & 12\\
edge $\leftarrow$ edge, across a corner & 24 & 24\\
\hline
Total & 225 & 225\\
\hline
\end{tabular}
\caption{Closed parameter counts and numerically computed nullities for the
worked $D_4$ example.}
\label{tab:experimental-parameter-counts}
\end{table}

To test the realization of these bases, we sampled random admissible
coefficients and random input fields.  The maximal absolute difference between
the dense matrix implementation and the sparse gather--transform--scatter
implementation was
\begin{equation}
  5.33\times10^{-15}.
  \label{eq:sparse-dense-certificate}
\end{equation}
Thus the sparse layer used for training realizes, to floating-point accuracy,
the same linear operator as the finite constraint construction.

\subsection{Exhaustive rigid-motion tests on a finite grid}
\label{ssec:exhaustive-transport-tests}

We next test the transport identities independently of any training
procedure.  To make exhaustive enumeration inexpensive, we use scalar feature
types on all three strata of a $7\times6$ grid and radius $r_0=1$.  We
enumerate every pair $(A,t)$ with $A\in D_4$ and integer translation $t$ for
which the maximal admissible rigid bisection is nonempty.  This gives $310$
nonempty partial rigid motions.

For a random complete layer $\Phi_1$, the maximal absolute residual of the
single-layer identity
\[
 P_{U'}\Phi_1\Lambda(b)
 -\Lambda(b)P_U\Phi_1P_U
\]
over all $310$ motions is $2.22\times10^{-15}$.  We then compose two
independent random layers of radius one.  Among the enumerated motions, $82$
have a nonempty one-pixel eroded domain.  On those domains the worst residual
of the two-layer composite is $3.11\times10^{-15}$.  On the maximal,
non-eroded domains, by contrast, the worst residual is $1.92$, and $276$ of
the $310$ motions violate the identity above tolerance $10^{-8}$.  Finally,
on a square grid the eight global $D_4$ motions remain exact for the two-layer
composite, with worst residual $1.78\times10^{-15}$.

\begin{table}[H]
\centering
\small
\begin{tabular}{lr}
\hline
Certificate & Worst absolute residual\\
\hline
Single layer, all $310$ nonempty rigid motions & $2.22\times10^{-15}$\\
Two layers, $82$ nonempty eroded domains & $3.11\times10^{-15}$\\
Two layers, maximal domains & $1.92$\\
Two layers, global $D_4$ on a square & $1.78\times10^{-15}$\\
\hline
\end{tabular}
\caption{Exact transport certificates.  The order-one residual on maximal
partial domains confirms the necessity of erosion under composition; global
symmetries do not erode.}
\label{tab:transport-certificates}
\end{table}

These tests verify separately the three claims proved earlier: individual
layers satisfy the bisection-equivariant kernel theorem; composition is exact
on the eroded domains predicted by Theorem~\ref{thm:n-layer-filtered-equivariance};
and global equivariance is preserved at arbitrary depth.  They also show why
an empty eroded domain must be recorded as not applicable rather than as a
zero-residual certificate.

\section{Recovery of exactly equivariant target operators}
\label{sec:synthetic-recovery}

The preceding certificates show that the implementation satisfies the desired
identities, but they do not test whether the parameterization is complete or
whether the symmetry restriction reduces the amount of data needed to identify
an operator.  We therefore construct targets that lie in the complete
finite-dimensional layer space by design.


We use a $9\times8$ grid, radius $r_0=2$, and scalar feature types on the bulk,
edge, and corner strata.  In this case the four model spaces have dimensions
\[
\begin{array}{c|cccc}
\text{model}
 & \mathrm{GEQ}
 & \mathrm{GEQ\text{-}noX}
 & \mathrm{GEQ\text{-}noCorner}
 & \mathrm{UNC}
\\ \hline
\text{parameters} & 36 & 10 & 21 & 75.
\end{array}
\]
Here GEQ is the complete equivariant layer; GEQ-noX retains only the three
diagonal stratum blocks; GEQ-noCorner removes every block involving a corner;
and UNC retains the same geometric tap connectivity but learns every reference
slot independently.

For each of $40$ independent trials, we draw a random complete GEQ target
$\Phi_\star$ and normalize its Frobenius norm.  A scalar training observation
is a random bilinear probe
\begin{equation}
 z_i
 =\ell_i^{\mathsf T}\Phi_\star\psi_i,
 \label{eq:random-bilinear-probe}
\end{equation}
where $\psi_i$ is an isotropic Gaussian input field and $\ell_i$ is an
independent isotropic Gaussian output functional.  Each model is fitted by
minimum-norm least squares from $n$ such scalar observations.  We evaluate the
relative squared operator error
\begin{equation}
 \mathcal E(\widehat\Phi)
 =
 \frac{\|\widehat\Phi-\Phi_\star\|_{\mathrm F}^2}
      {\|\Phi_\star\|_{\mathrm F}^2}.
 \label{eq:relative-operator-error}
\end{equation}
For isotropic input and output probes, this quantity is also the relative
expected mean-squared error of \eqref{eq:random-bilinear-probe}.


Figure~\ref{fig:synthetic-geq-recovery} reports the median error and the
interquartile range over the $40$ trials.  Representative median values are
listed in Table~\ref{tab:synthetic-recovery}.

\begin{figure}[H]
\centering

\includegraphics[width=0.72\textwidth]{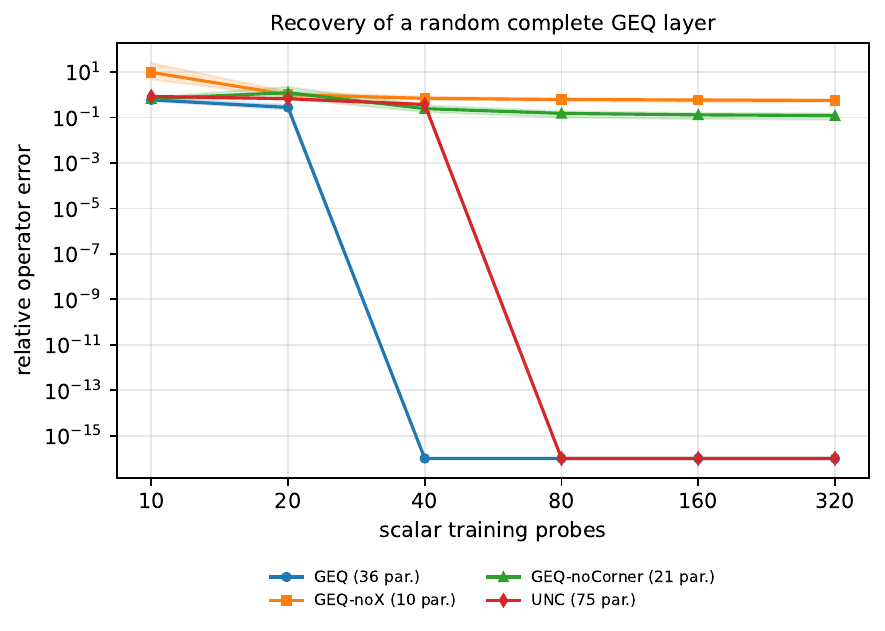}
\caption{Recovery of a random complete groupoid-equivariant layer from scalar
bilinear probes.  Curves show medians over $40$ independent targets.}
\label{fig:synthetic-geq-recovery}
\end{figure}

\begin{table}[H]
\centering
\small
\begin{tabular}{lrrrr}
\hline
Model & $n=20$ & $n=40$ & $n=80$ & $n=320$\\
\hline
GEQ & $2.76\times10^{-1}$ & $5.11\times10^{-30}$
    & $3.02\times10^{-30}$ & $2.54\times10^{-30}$\\
GEQ-noX & $1.02$ & $6.93\times10^{-1}$
    & $6.17\times10^{-1}$ & $5.61\times10^{-1}$\\
GEQ-noCorner & $1.23$ & $2.44\times10^{-1}$
    & $1.51\times10^{-1}$ & $1.21\times10^{-1}$\\
UNC & $6.70\times10^{-1}$ & $3.71\times10^{-1}$
    & $1.58\times10^{-29}$ & $4.14\times10^{-30}$\\
\hline
\end{tabular}
\caption{Median relative operator error in the exact-recovery experiment.}
\label{tab:synthetic-recovery}
\end{table}

The complete GEQ model reaches numerical precision once the number of generic
scalar probes exceeds its $36$-dimensional parameter space.  The unconstrained
model also contains the target, but requires approximately twice as many
probes, consistently with its $75$ free slot parameters.  The two restricted
equivariant models do not contain a generic complete target: their errors
approach nonzero approximation floors.  In particular, removing the corner
sector and its incident blocks leaves a median relative error of about $0.12$
even at $320$ probes.  Thus the corner and cross-stratum components are not
merely decorative additions to an otherwise homogeneous convolution; they
span directions of the complete equivariant operator space that cannot be
recovered by the corresponding ablations.

This experiment is deliberately linear and optimization-free.  Its purpose is
not to model a natural-data task, but to isolate two structural consequences
of the theory: completeness of the ten-block parameterization and the reduction
of identification complexity obtained by solving the equivariance constraints
before learning.


\section{Boundary-aware operator learning}
\label{sec:boundary-learning}

The certificates of Section \ref{sec:experiments} and the synthetic study of Section \ref{sec:synthetic-recovery}
concern targets inside the exact equivariant class. This section tests
the architecture as an inductive bias \emph{outside} that class, on
discrete Poisson--Dirichlet solution operators. We first verify that
the target operators are not exactly equivariant under proper local
bisections, so that any advantage must be interpreted as approximation
and regularization rather than as membership of the hypothesis class.
We then report a controlled benchmark executed under a pre-specified
protocol: independent training, validation, and test sets;
normalization constants estimated on training data only; per-model
learning rates selected on validation data; early stopping on
validation with the test set evaluated exactly once per run, at the
best-validation checkpoint; and ten independent seeds, each with fresh
data and initialization. Means, standard deviations, medians and
interquartile ranges are reported throughout. Finally, we measure the
transport residuals of the \emph{trained} operators, which yields two
results of independent interest: unconstrained models do not acquire
even approximate partial equivariance from training on this family of
tasks, and a trained nonlinear network furnishes a numerical
counterexample showing that the balanced erosion radius of
Theorem~\ref{thm:n-layer-filtered-equivariance} certifies linear composites only (see the
corrected Corollary~\ref{cor:filtered-network}).

\subsection{Symmetry diagnostic for the Poisson--Dirichlet inverse}
\label{sec:poisson-diagnostic}

Consider the discrete Dirichlet problem
\begin{equation}
  -\Delta_h u = f \ \text{ on the interior pixels},
  \qquad u = g \ \text{ on the boundary},
  \label{eq:dirichlet}
\end{equation}
where $\Delta_h$ is either the five-point Laplace stencil or the
nine-point Mehrstellen stencil
$20u_p - 4\sum_{\mathrm{cross}} u - \sum_{\mathrm{diag}} u = 6 f_p$.
Both stencils are compatible with rigid local motions, but their
inverses on a fixed rectangle depend on the Green operator of the
entire domain. Assembling the exact solution operator
$S_D\colon (f,g)\mapsto u$ on an $8\times 6$ grid and evaluating the
relative transport residual:
\begin{equation}
 \frac{\|P_{U'}S_D\Lambda(b)-\Lambda(b)P_US_DP_U\|_{\mathrm F}}
      {\|\Lambda(b)P_US_DP_U\|_{\mathrm F}}
 \label{eq:poisson-relative-transport-residual}
\end{equation}
confirms the
expectation: the global dihedral symmetries $D_2$ of the rectangle are
preserved to $2.8\times 10^{-16}$, whereas proper partial motions
(partial translations and a partial quarter-turn) produce residuals
between $0.14$ and $0.26$. Consequently the experiments below test the
value of a boundary-aware inductive bias; they are not exact-recovery
experiments, and empirical advantages must not be read as target
equivariance.

The two stencils were chosen as a matched pair with different
information flow from the corner stratum. In the five-point problem no
interior equation reads a corner value, so the target is exactly
corner-blind. In the nine-point problem the interior equation adjacent
to each corner reads the corner Dirichlet value through its diagonal
neighbour, with weight $1/20$; the corner data therefore influence the
target, but weakly. Section~\ref{sec:cornerx} adds a third task in
which the corner influence is total.

\subsection{Benchmark design and executed protocol}
\label{sec:protocol}

\paragraph{Tasks and data.} All experiments use a $16\times 12$ grid.
Inputs are a smooth random interior source $f$ and independent smooth
random Dirichlet data $g$; the target is the exact finite-difference
solution of \eqref{eq:dirichlet} for the corresponding stencil. For
every task and seed we draw a fresh training pool of $400$ fields, a
validation set of $200$, and a test set of $400$. A run trained on $n$
fields normalizes all targets by the standard deviation of its own $n$
training targets; validation and test data are never used for any
estimation other than the final evaluation.

\paragraph{Models.}
All trained networks have depth four and propagation radius $r_0 = 2$
per layer. The stratified models use the hidden type
$\rho^{(\ell)} = A_1\oplus A_2\oplus B_1\oplus B_2\oplus E$,
$\varepsilon^{(\ell)} = 2\cdot\mathbf{1}\oplus\mathrm{sgn}$,
$\delta^{(\ell)} = \mathbf{1}\oplus\mathrm{sgn}$, with scalar types on
input and output.

The main benchmark reported in Table~7 and Figure~3 comprises the
complete model GEQ (1076~parameters); GEQ-noX (all cross-stratum
blocks removed, 332); the connectivity-matched unconstrained control
UNC (3716); the zero-padded free-plane $D_4$-steerable network STEER
(283) and its width-matched variant STEER-w (1018); and zero-padded
ordinary CNNs of widths $4$, $6$, $12$ (1100, 2250, 8100 parameters),
where CNN-4 is parameter-matched to GEQ.

The finer block ablations reported in Table~8 are GEQ-noEB
(bulk$\leftrightarrow$edge blocks removed, 624), GEQ-noC (all four
corner cross-blocks $c\!\leftrightarrow\!b$, $c\!\leftrightarrow\!e$
removed, 880), and GEQ-noXC (across-corner edge block removed, 980);
GEQ-noX is included there again as the fully cross-stratum-ablated
reference.

Finally, the nonlinear round on the nine-point task reported in
Section~15.5 and Figure~4 adds GEQ-nl (1088): the same linear layers
interleaved with the pointwise equivariant nonlinearities of
Section~\ref{ssec:practical-nonlinearities} (GELU on invariant
channels, odd $\tanh$ on sign channels, norm gating on $E$) and
admissible biases per \eqref{eq:block-constraint-system}; and
CNN-6-nl (2250) with GELU between convolutions.

\paragraph{Optimization.} Adam with cosine annealing, batch size
$\min(64, n)$; stratified models train $300$ steps, image-format
models $600$ (their per-step cost is far lower and their deep-linear
optimization is slower). Validation
is evaluated every $50$ steps with patience-based early stopping; the
parameters achieving the best validation error are restored and the
test error computed once. Learning rates are selected per model and
task from $\{3\times 10^{-2}, 10^{-2}, 3\times 10^{-3}\}$ on
validation data of a held-out selection seed; the three GEQ block
ablations reuse GEQ's rate. The selections differ across models
(GEQ-noX selects $10^{-2}$ on both tasks, while CNN-12 selects
$10^{-2}$ on the five-point task and $3\times10^{-3}$ on the
nine-point task), confirming that a shared budget distorts comparisons. Training is in single precision; all transport
certificates of Section~\ref{sec:trained-certs} are evaluated in
double precision.

\subsection{Sample efficiency and completeness ablations}
\label{sec:definitive-benchmark}

Table~\ref{tab:definitive} and
Figure~\ref{fig:definitive-curves} report the benchmark.

\begin{table}[t]
\centering\small
\begin{tabular}{lrcccccc}
\hline
 & & \multicolumn{3}{c}{five-point} & \multicolumn{3}{c}{nine-point}\\
Model & Par. & $n{=}25$ & $n{=}100$ & $n{=}400$
             & $n{=}25$ & $n{=}100$ & $n{=}400$\\
\hline
GEQ     & 1076 & $.0027{\pm}.0005$ & $.0024{\pm}.0004$ & $.0025{\pm}.0006$
               & $.0026{\pm}.0005$ & $.0024{\pm}.0004$ & $.0024{\pm}.0004$\\
GEQ-noX &  332 & $.0086{\pm}.0026$ & $.0084{\pm}.0023$ & $.0092{\pm}.0050$
               & $.0090{\pm}.0028$ & $.0084{\pm}.0029$ & $.0089{\pm}.0042$\\
UNC     & 3716 & $.0161{\pm}.0049$ & $.0147{\pm}.0067$ & $.0141{\pm}.0060$
               & $.0225{\pm}.0206$ & $.0163{\pm}.0100$ & $.0152{\pm}.0090$\\
STEER   &  283 & $.0287{\pm}.0411$ & $.0298{\pm}.0455$ & $.0305{\pm}.0492$
               & $.0320{\pm}.0483$ & $.0284{\pm}.0463$ & $.0312{\pm}.0499$\\
STEER-w & 1018 & $.1532{\pm}.1520$ & $.2455{\pm}.4034$ & $.2316{\pm}.3395$
               & $.4232{\pm}.9438$ & $.3260{\pm}.6727$ & $.2779{\pm}.5183$\\
CNN-4   & 1100 & $.0117{\pm}.0033$ & $.0104{\pm}.0034$ & $.0099{\pm}.0038$
               & $.0130{\pm}.0055$ & $.0100{\pm}.0031$ & $.0098{\pm}.0033$\\
CNN-6   & 2250 & $.0213{\pm}.0189$ & $.0161{\pm}.0086$ & $.0156{\pm}.0076$
               & $.0196{\pm}.0077$ & $.0161{\pm}.0082$ & $.0154{\pm}.0085$\\
CNN-12  & 8100 & $.1833{\pm}.1065$ & $.1559{\pm}.0754$ & $.1419{\pm}.0567$
               & $.1957{\pm}.0811$ & $.1503{\pm}.0446$ & $.1460{\pm}.0419$\\
\hline
\end{tabular}
\caption{Relative test MSE (mean $\pm$ standard deviation over ten
seeds) for the two Poisson--Dirichlet solution operators under the
executed protocol. Depth four, $r_0=2$ throughout.}
\label{tab:definitive}
\end{table}
Four conclusions are stable across both tasks. First, the complete
groupoid model attains $0.0024$--$0.0027$ and is essentially flat in
$n$ from 25 to 400 training fields, consistent with a 1076-parameter
hypothesis class that tightly brackets the target; its margin over the
strongest baseline is a factor of about four, and over the remaining
baselines a factor of six or more. Second, the strongest baseline is
the \emph{parameter-matched} ordinary CNN (CNN-4, $0.0098$--$0.0099$): a
small zero-padded convolutional network exploiting absolute position
through padding \cite{KayhanGemert2020,IslamJiaBruce2020} is a genuinely competitive boundary
heuristic. The groupoid model outperforms it fourfold at comparable
parameter count while additionally carrying exact certificates
(Section~\ref{sec:trained-certs}). Third, the unconstrained control UNC, with identical
geometric connectivity and 3.5$\times$ the parameters of GEQ, is a
further $50\%$ worse than CNN-4 and never approaches GEQ; symmetry, not
connectivity, drives the gap. Fourth, the zero-padded steerable model
STEER stagnates at $0.028$--$0.032$ at every training size. Its
width-matched variant STEER-w fails to train under any validated rate
(median error $0.11$--$0.13$); a 2500-step diagnostic run still
plateaus near $0.042$, above the plain STEER floor, so enlarging the
steerable family's types makes its deep-linear optimization harder
rather than its hypothesis class effectively richer. The GEQ--STEER
gap is therefore not a parameter-count artifact.

\begin{figure}[H]
\centering
\includegraphics[width=1\textwidth]{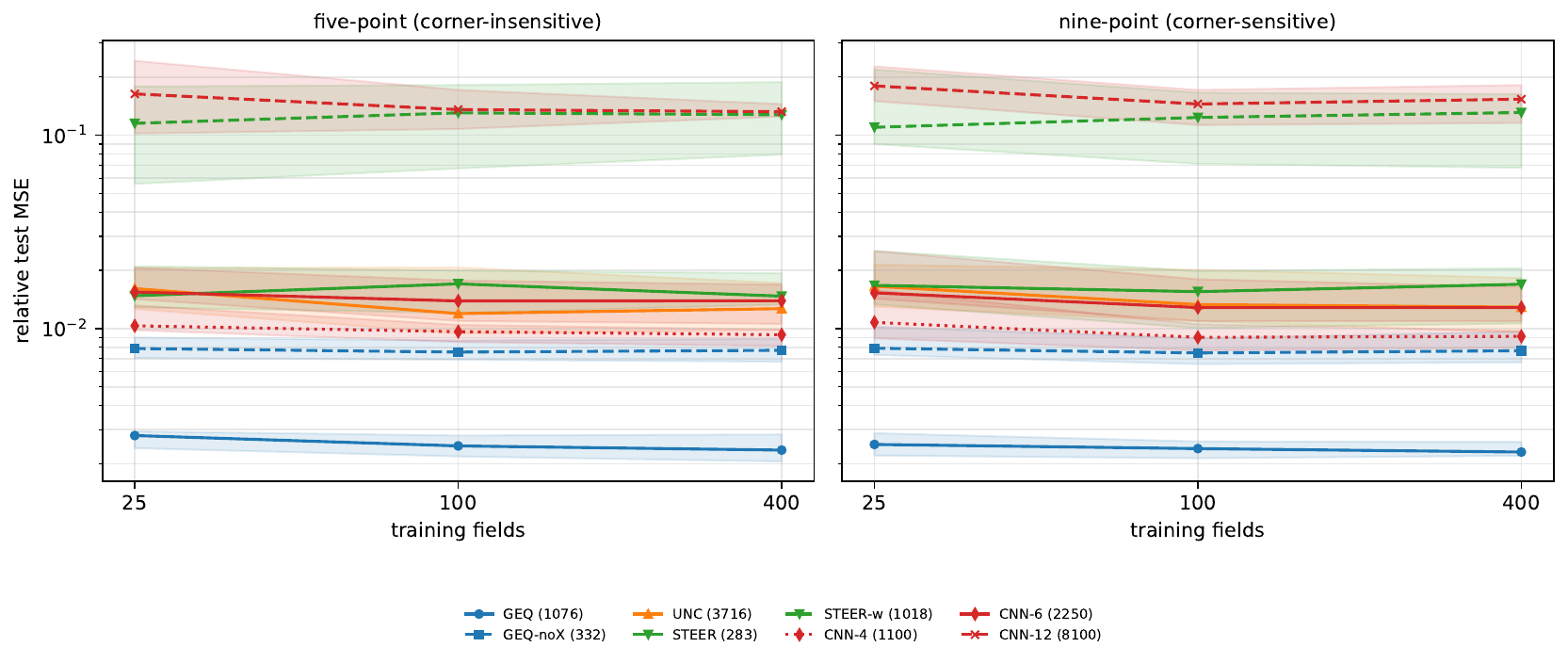}
\caption{Sample-efficiency curves for the Poisson--Dirichlet benchmark.}
\label{fig:definitive-curves}
\end{figure}

\begin{table}[t]
\centering\small
\begin{tabular}{lrcc}
\hline
Model & Par. & five-point, $n{=}400$ & nine-point, $n{=}400$\\
\hline
GEQ       & 1076 & $.0025\pm.0006$ & $.0024\pm.0004$\\
GEQ-noXC  &  980 & $.0023\pm.0004$ & $.0021\pm.0003$\\
GEQ-noC   &  880 & $.0023\pm.0004$ & $.0024\pm.0004$\\
GEQ-noEB  &  624 & $.0052\pm.0016$ & $.0057\pm.0016$\\
GEQ-noX   &  332 & $.0092\pm.0050$ & $.0089\pm.0042$\\
\hline
\end{tabular}
\caption{Per-block ablation at $n=400$ (ten seeds). Removing the
bulk$\leftrightarrow$edge blocks degrades both tasks by a factor
$2.1$--$2.4$; removing the corner-incident or across-corner blocks has
no measurable effect on either PDE task.}
\label{tab:ablation}
\end{table}

The per-block ablation (Table~\ref{tab:ablation}) resolves \emph{which}
cross-stratum channels carry the advantage. On both stencils the
entire effect is carried by the bulk$\leftrightarrow$edge blocks: their
removal alone accounts for most of the gap to the fully diagonal model,
while the corner-incident and across-corner ablations are statistically
indistinguishable from the complete model. For the five-point stencil
this null result is a structural prediction confirmed, the target is
provably corner-blind, and for the nine-point stencil it shows that a
corner coupling of relative weight $1/20$ in a single interior equation
per corner is buried beneath the locality-approximation floor of
$\approx 0.0024$. The empirical case for the corner machinery therefore
cannot be made on these two tasks, and Section~\ref{sec:cornerx}
supplies the task on which it can.

\subsection{A corner-sensitive diagnostic: harmonic extension of
corner data}
\label{sec:cornerx}

We pose the nine-point problem with $f\equiv 0$ and Dirichlet data
supported \emph{only} at the four corners, so that the target is the
discrete (Mehrstellen) harmonic extension of four scalars and every
path from input to output passes through the corner stratum. The
complete model can express this map only through its classified
corner$\to$bulk couplings; GEQ-noC provably has no such pathway and
its best attainable prediction is zero, with relative error one.
Table~\ref{tab:cornerx} confirms the prediction exactly.

\begin{table}[t]
\centering\small
\begin{tabular}{lcc}
\hline
Model & Par. & relative test MSE ($n{=}100$, ten seeds)\\
\hline
GEQ      & 1076 & $0.0024\pm 0.0004$\\
GEQ-noC  &  880 & $1.0000\pm 0.0000$\\
STEER    &  283 & $0.0015\pm 0.0001$\\
CNN-6    & 2250 & $0.0016\pm 0.0001$\\
\hline
\end{tabular}
\caption{Corner-extension diagnostic. The corner-ablated model fails
identically, as predicted; the image-format baselines succeed through
the corner pixels visible to their padded convolutions.}
\label{tab:cornerx}
\end{table}

Two remarks keep the interpretation honest. The image-format baselines
also solve the task, because a padded convolution reads the corner
pixels of its input array directly; the diagnostic separates
\emph{corner information pathways}, not architectures. And taken
together with Table~\ref{tab:ablation}, the three tasks show the
ablation pattern tracking the information flow of the target operator
in every case: bulk$\leftrightarrow$edge blocks necessary exactly when
edge data drives the solution, corner blocks necessary exactly when
corner data does. This is the operational content of the completeness
theorem: the classified blocks are not interchangeable capacity but
distinct, geometrically indexed communication channels.

\subsection{Nonlinear round}
\label{sec:nonlinear-round}

On the nine-point task we additionally train GEQ-nl and CNN-6-nl;
the results are shown in Figure~\ref{fig:poisson-nonlinear-round}. GEQ-nl reaches $0.0565\pm0.0188$, $0.0341\pm0.0160$, and
$0.0332\pm0.0162$ at $n = 25, 100, 400$; CNN-6-nl reaches
$0.0576\pm0.0251$, $0.0338\pm0.0187$, $0.0298\pm0.0177$. Both are
worse than their linear counterparts, as they must be: the target
operator is linear, and a nonlinear network can only approach it by
learning to operate in the quasi-linear regime of its activations. The
round is included not as a benchmark claim but as an end-to-end
validation of Section~\ref{sec:architectural-components}: the gated equivariant
nonlinearities, parity rules, and admissible biases train stably under
gradient descent, preserve exact global equivariance at any depth, and
satisfy the (corrected) filtered transport identities measured next.
\begin{figure}[t]
\centering
\includegraphics[width=0.62\textwidth]{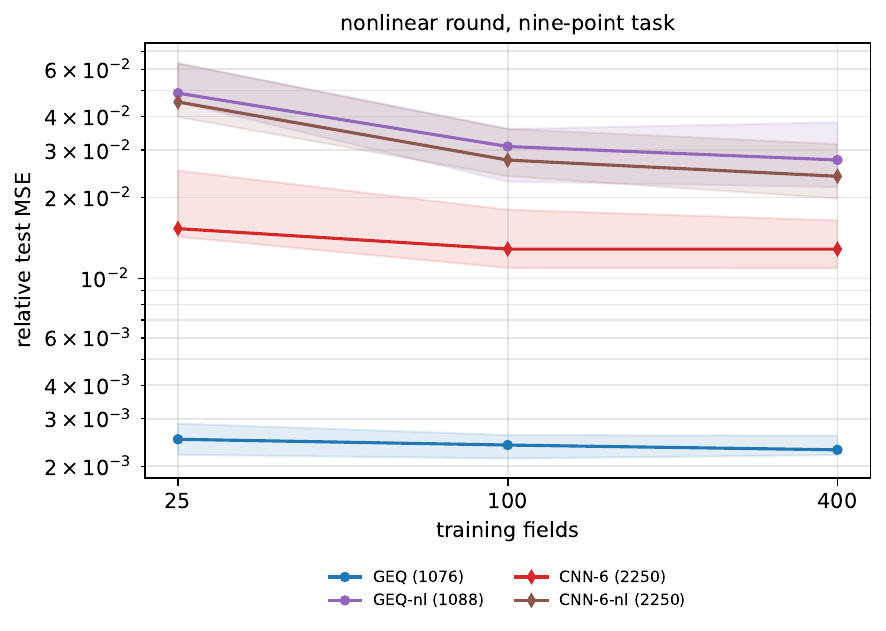}
\caption{Nonlinear-round results on the nine-point Poisson--Dirichlet task.
Median relative test MSE is shown as a function of the number of training
fields, with shaded interquartile ranges. GEQ-nl and CNN-6-nl denote the
nonlinear counterparts of GEQ and CNN-6, respectively.}
\label{fig:poisson-nonlinear-round}
\end{figure}

\subsection{Transport certificates of trained operators}
\label{sec:trained-certs}

For every trained model (seed~$0$, $n=400$) we evaluate, in double
precision, the relative transport residual on: the partial translations
$(1,0)$ and $(2,1)$ and the partial quarter-turn about a corner, each
on the domain eroded by the balanced radius
$\rho_4 = \lfloor 4/2\rfloor r_0 = 4$ of Theorem~\ref{thm:n-layer-filtered-equivariance}
($60$, $4$, and $49$ transported pairs respectively), Corollary~\ref{cor:equal-radius-erosion}; the partial
translation $(3,2)$, whose radius-$4$ eroded domain is empty; and the three nontrivial
global $D_2$ symmetries of the $16\times 12$ rectangle, which require
no erosion at any depth.

\begin{table}[t]
\centering\small
\begin{tabular}{lccc}
\hline
Model & partial motions, $\rho_4$-eroded & global $D_2$\\
\hline
GEQ and all GEQ ablations & $\le 3.0\times 10^{-15}$ & $\le 3.0\times 10^{-15}$\\
STEER, STEER-w            & $\le 4.7\times 10^{-15}$ & $\le 6.3\times 10^{-15}$\\
UNC                       & $1.5$--$1.7$             & $1.6$--$1.8$\\
CNN-4, CNN-6, CNN-12      & $1.2$--$2.0$             & $1.3$--$1.4$\\
\hline
GEQ-nl (at $\rho_4=4$)    & $0.11$ (quarter-turn)    & $1.8\times 10^{-15}$\\
GEQ-nl (at $(n{-}1)r_0=6$)& $1.1\times 10^{-15}$     & ---\\
CNN-6-nl                  & $1.4$                    & $1.5$\\
\hline
\end{tabular}
\caption{Worst relative transport residuals of \emph{trained}
operators. Translations are exact for all models (zero padding
coincides with the input masking of the transport identity);
the reflection and rotation certificates separate the classes.}
\label{tab:trained-certs}
\end{table}

Three conclusions. First, the trained GEQ family satisfies every
certificate to machine precision, as guaranteed by construction: the
certificates are parameter-independent. STEER also passes, as the
classification theorem predicts for a position-independent sub-family
of the $B$-equivariant class; the distinction between STEER and GEQ is
completeness of parameterization, never certificate satisfaction.
Second, the unconstrained models, UNC and the ordinary CNNs,
including the nonlinear one, exhibit order-one residuals on the
rotation and global reflection certificates after training. On this
family of tasks, training does \emph{not} induce even approximate
partial equivariance; whatever boundary competence CNN-4 acquires, it
is not of transport type.

Third, and most consequentially, the trained nonlinear network GEQ-nl
\emph{fails} the quarter-turn certificate on the balanced-radius
domain ($0.11$ at erosion $4$; $0.096$ at erosion $5$) and passes at
the cumulative radius $\smash{\sum_{j\ge 2} r_j} = (n-1)r_0 = 6$ with
residual $1.1\times 10^{-15}$, while the linear GEQ passes at radius
$4$ with residual $1.1\times 10^{-15}$ and both preserve the global
$D_2$ symmetries exactly. The balanced radius
$\lfloor n/2\rfloor r_0$ is a strictly multilinear phenomenon, whereas
nonlinear composites are certified on the one-sided cumulative
erosion, which the experiment shows to be attained here. We regard
this as the strongest argument for reporting trained-operator
certificates routinely: they audit not only implementations but
theorems.

Taken together, the experiments support the following conclusions at
stated levels of strength. The algebraic certificates verify the
implementation, the composition theorem, and through the nonlinear
counterexample the corrected erosion radii, to floating-point
precision. The synthetic study of Section \ref{sec:synthetic-recovery} verifies completeness and
identification complexity inside the exact class. The boundary
benchmark shows a four- to six-fold accuracy advantage for the
complete stratified parameterization over connectivity-, parameter-,
and width-matched controls on operators outside the exact class, with
the per-block ablations tracking the information flow of each target;
and the trained-certificate audit shows that this advantage coexists
with exact partial-symmetry guarantees that unconstrained training
does not discover on its own.


\section{Conclusions and outlook}
\label{sec:outlook}

This paper develops a theory of equivariant neural networks for situations in
which symmetry is not represented by a group acting globally on the signal
space.  The basic object is the measured groupoid symmetry datum
\[
 \bigl(
   \Gamma\rightrightarrows\Omega,
   \mathcal B,
   \nu;
   R^{\mathrm{in}},R^{\mathrm{out}}
 \bigr),
\]
which records admissible arrows, coherent partial transformations, the measure
on the object space, and the transformation laws of the feature fibres.  This
separation is essential: the same groupoid may support different notions of
equivariance depending on the chosen bisections and representations.

The principal abstract result is the bisection-equivariant kernel theorem.  It
reduces equivariance of an object-space integral channel to the transport law
\[
 K\bigl(\tau_b(y),\tau_b(x)\bigr)
 =R^{\mathrm{out}}\bigl(b(y)\bigr)
  K(y,x)
  R^{\mathrm{in}}\bigl(b(x)\bigr)^{-1}.
\]
The solution space is controlled by the diagonal action of the bisection
pseudogroup on pairs of objects: one chooses an intertwiner of the joint
stabilizer at each pair-orbit representative and transports it along the
orbit.  This both recovers the classical steerable-kernel constraint on
homogeneous spaces and explains what replaces translational weight sharing
when homogeneity is lost.

For bounded planar domains, the tangent-cone groupoid captures the boundary part
of the symmetry model.  Its object orbits distinguish bulk, edge, and corner
geometry; its pair orbits classify all admissible communication between these
strata.  The resulting layer is not a direct sum of three independent
networks.  The bulk--edge, bulk--corner, and edge--corner kernels are allowed
and constrained by equivariance, and hence provide intrinsic channels through
which boundary information enters the interior.  On the pixel grid, the
nonconjugate edge and corner reflection subgroups of $D_4$ produce distinct
branching rules.  The complete finite layer is obtained by solving explicit
homogeneous constraint systems, and its trainable dimension agrees with the
closed representation-theoretic formulas.

Partial symmetry also changes the algebra of depth.  Individual finite-range
layers satisfy the full local transport identity, but their composition can
route information through points outside the domain of a proper local
bisection.  The erosion theorems quantify the exact region on which the
composite remains equivariant.  This leads to filtered equivariance rather
than a binary distinction: the depth and propagation radii determine a
controlled boundary layer, while global symmetries remain exact at every
depth.  Pointwise equivariant nonlinearities, admissible biases, normalization,
parallel branches, and residual connections fit naturally into this filtered
structure.

The numerical results separate three different claims that are often
conflated.  First, the nullspace dimensions, sparse implementation, and
transport identities verify the mathematical construction to floating-point
precision.  Second, the synthetic identification experiment shows that the
complete parameterization spans the exact equivariant class and requires fewer
generic observations than an unconstrained model with the same geometric
connectivity.  Third, the Poisson--Dirichlet experiment illustrates the use of
the architecture as a boundary-aware inductive bias outside the exact class.
The direct symmetry diagnostic is important: the inverse Dirichlet operator on
a fixed rectangle preserves the global symmetries of that rectangle, but not
general proper rigid local bisections.  The empirical advantage observed in
this experiment must therefore be interpreted as approximation and regularization,
not as exact target equivariance.

\medskip
Several questions remain open.

\medskip
\noindent\textit{Curved boundaries and higher-order geometry.}
Tangent cones capture first-order boundary type and are sufficient for
polygonal domains.  For curved boundaries, matching tangent cones does not
ensure that a rigid motion maps one boundary germ to another; curvature and
higher jets provide additional obstructions.  A natural extension is to
replace the tangent-cone groupoid by a hierarchy of jet-refined groupoids.
This would separate boundary points by curvature or other local invariants and
produce correspondingly refined isotropy and kernel spaces.  It would also
clarify the transition between exact local isometries and approximate
geometric matching.

\medskip
\noindent\textit{Manifolds and the choice of geometric groupoid.}
On a general Riemannian manifold there need not be a rich pseudogroup of local
isometries.  Extending the construction therefore requires an explicit choice
of geometry: the groupoid of germs of local isometries, a frame or gauge
groupoid, a holonomy groupoid, a jet groupoid, or a groupoid supplied by the
application.  These choices encode different notions of symmetry and should
not be identified.  Similar questions arise for Lorentzian, gauge-theoretic,
and constrained systems.

\medskip
\noindent\textit{Morita invariance and presentation independence.}
A groupoid may admit several equivalent presentations of the same underlying
stack or orbit geometry.  The present construction is expressed using a
specific object space, measure, family of bisections, and bundle
trivialization.  It is therefore important to determine which parts of the
architecture are invariant under Morita equivalence and how symmetry data,
kernel spaces, and learned operators should be transported between equivalent
groupoids. 

\medskip
\noindent\textit{Relation with arrow-space convolution and categorical models.}
The kernel theorem classifies operators on fields over the object space,
whereas the standard convolution algebra of a Lie groupoid is defined on the
arrow space using a Haar system.  Establishing precise transforms between
these constructions is an important next step.  More general categories may
also be needed when the relevant relations are noninvertible.  Trace maps,
restriction from bulk to boundary, incidence maps between cells, and coarse-
to-fine maps are naturally categorical rather than groupoidal.  Combining the
present invertible local symmetries with such noninvertible geometric maps may
lead to a unified architecture for stratified and cellular domains.

\medskip
\noindent\textit{Approximation theory and statistical complexity.}
The finite layer space is completely characterized, but universal
approximation for nonlinear groupoid-steerable networks remains to be
established in the present measured and partially equivariant setting.  One
would like criteria ensuring density in spaces of continuous or measurable
maps satisfying the relevant filtered transport laws, together with estimates
of sample complexity in terms of pair-orbit and stabilizer data.  The
synthetic experiment suggests that the dimension of the equivariant operator
space governs generic identification complexity; a statistical theory should
make this relation precise.

\medskip
\noindent\textit{Multiscale architectures.}
The reference model works at fixed resolution.  Pooling, subsampling, mesh
coarsening, and multigrid operations change the object space and hence the
symmetry datum.  Their equivariance should be formulated through morphisms
between groupoids or through compatible functors between stratified
resolutions.  Such a theory would enable U-net, multiresolution, and neural-
operator architectures while retaining explicit control of local symmetry and
erosion across scales.

\medskip
\noindent\textit{Boundary-value and operator-learning problems.}
The most informative applied tests should distinguish local differential
operators from their domain-dependent inverses.  Promising directions include
learning an ambient equivariant core together with a stratified boundary
correction, using families of transformed domains as inputs, and studying
corner-sensitive finite-element or mixed-boundary problems.  These experiments
should report both prediction error and transport residuals, use independent
validation and test sets, and include matched ablations of individual
cross-stratum blocks.  Domains with holes, nonconvex polygons, and varying
boundary types would test whether the learned representation genuinely
exploits the groupoid stratification.

\medskip
\noindent\textit{Broader representation classes.}
The present implementation uses finite-dimensional orthogonal or unitary
feature types and measure-preserving bisections.  Quasi-invariant measures
would introduce Radon--Nikodym factors in the transport operators.  Projective,
nonunitary, graded, and operator-algebraic representations could encode
additional physical structure.  In particular, imposing complete positivity
and normalization on the pointwise or integral channels leads toward
symmetry-constrained quantum channels and groupoid-based quantum neural
architectures.

The main conclusion is that boundaries and other failures of global
homogeneity need not be treated as defects of an otherwise homogeneous CNN.
They can be incorporated into the symmetry datum itself.  Groupoids then do
more than rephrase group equivariance: their object orbits identify geometric
strata, their pair orbits determine the admissible kernels between those
strata, and their local bisections provide an exact, computable notion of
partial symmetry whose behavior under depth can be quantified.


\section*{Acknowledgements}
A.I. acknowledges financial support from the Spanish Ministry of Economy and Competitiveness, through the Severo Ochoa Program for Centers of Excellence in RD (SEV-2015/0554). The authors acknowledge the MINECO research project PID2024-160539NB-I00, and the Comunidad de Madrid project TEC-2024/COM-84 QUITEMAD-CM. M.J.V. acknowledges the partial support of CaLIGOLA HERP 101086123.

\section*{Tool and computational resource disclosure}
The authors of this article adhere to the Leiden Declaration on Artificial Intelligence and Mathematics ( https://leidendeclaration.ai ).   In accordance with that we acknowledge that the code as well as part of the text used in the certificates and experiments described in Parts \ref{part:architecture}-\ref{part:experiments} was done with the help of Claude Fable 5 (Anthropic) and ChatGPT 5.6 (OpenAI) and it was independently drafted and verified by the authors.

\end{document}